\pdfoutput=1

\documentclass[11pt]{article}

\usepackage{acl}
\usepackage{times}
\usepackage{latexsym}

\usepackage[T1]{fontenc}

\usepackage[utf8]{inputenc}

\usepackage{microtype}

\usepackage{inconsolata}

\usepackage[hang,flushmargin]{footmisc}

\title{
\textit{When Calibration Rankings Reverse:} Accuracy-Controlled Evaluation for Fair Comparison of LLMs
}

\author{
Zhichao Yang\textsuperscript{1}\thanks{Equal contribution. \dag Corresponding author.}, 
Caiqi Zhang\textsuperscript{2}\footnotemark[1], 
Ruihan Yang\textsuperscript{1}\footnotemark[1], 
Chengzu Li\textsuperscript{2},\\
\textbf{Nigel Collier}\textsuperscript{2},
\textbf{Deqing Yang}\textsuperscript{1\dag}
\\
\textsuperscript{1}Fudan University 
\quad 
\textsuperscript{2}University of Cambridge 
\\
\texttt{\{zcyang25, rhyang21\}@m.fudan.edu.cn, \{cz391\}@cam.ac.uk}
}

\usepackage{microtype}
\usepackage{graphicx}
\usepackage{arydshln}
\usepackage{inconsolata}
\usepackage{todonotes}
\usepackage{tabularx}
\usepackage{booktabs} 
\usepackage{multirow}
\usepackage{amsmath}
\usepackage{amssymb}
\usepackage{graphicx}
\usepackage{subcaption}
\usepackage{fge}

\usepackage[framemethod=TikZ]{mdframed} 
\usepackage{listings}
\usepackage{longtable}
\usepackage{tcolorbox}
\usepackage{xspace}
\usepackage{wrapfig}
\usepackage{pgfplots}
\usepackage{color, colortbl}
\usepackage{array}
\usepackage{paralist}
\usepackage{enumitem}
\usepackage{float}  
\usepackage{makecell}
\usepackage{pifont} 
\usepackage{xcolor}  
\usepackage{colortbl}

\usepackage{geometry}

\usepackage{hyperref}

\definecolor{Gray}{gray}{0.92}
\definecolor{racing-green}{rgb}{0.0, 0.8, 0.6}
\definecolor{awesome-red}{rgb}{1.0, 0.13, 0.32}
\definecolor{LightCyan}{rgb}{0.88,1,1}
\definecolor{darkgreen}{RGB}{0,150,0}
\definecolor{Ground}{RGB}{255,184,55}
\definecolor{Dirt}{RGB}{191,169,115}
\definecolor{Pink}{RGB}{226,184,176}
\definecolor{Violet}{RGB}{163,148,170}
\definecolor{darkred}{RGB}{150,0,0} 
\definecolor{Red}{RGB}{171, 61, 56}
\definecolor{Green}{RGB}{62, 139, 117}
\definecolor{Blue}{RGB}{48, 110, 184}

\definecolor{CC}{RGB}{198, 226, 212} 
\definecolor{UU}{RGB}{198, 228, 253} 
\definecolor{CU}{RGB}{247, 202, 193} 
\definecolor{UC}{RGB}{242, 224, 253}

\newcommand{\ie}{\textit{i}.\textit{e}.,\ }
\newcommand{\eg}{\textit{e}.\textit{g}.,\ }

\newtheorem{assumption}{Assumption}

\newtheorem{remark}{Remark}

\newtheorem{proposition}{Proposition} 

\newenvironment{proof}[1][Proof]{\par\noindent\textit{#1.}\space\ignorespaces}{\hfill$\square$\par\medskip}

\definecolor{level4}{RGB}{110,136,203}
\definecolor{level3}{RGB}{173,190,226}
\definecolor{level2}{RGB}{205,208,243}
\definecolor{level1}{RGB}{236,236,252}

\newcommand{\cellcolorval}[1]{
   \ifdim#1pt>100pt \cellcolor{level4!95}#1\relax
   \else
   \ifdim#1pt>90pt \cellcolor{level4!85}#1\relax
   \else\ifdim#1pt>80pt \cellcolor{level3!75}#1\relax
   \else\ifdim#1pt>70pt \cellcolor{level3!60}#1\relax
   \else\ifdim#1pt>60pt \cellcolor{level2!70}#1\relax
   \else\ifdim#1pt>50pt \cellcolor{level2!45}#1\relax
   \else\ifdim#1pt>40pt \cellcolor{level2!30}#1\relax
   \else\ifdim#1pt>30pt \cellcolor{level2!10}#1\relax
   \else\ifdim#1pt>20pt \cellcolor{level1!10}#1\relax
   \else \cellcolor{level1!0}#1\relax
   \fi\fi\fi\fi\fi\fi\fi\fi
}

\newcolumntype{g}{>{\columncolor{Ground!5.2}}c}
\newcolumntype{d}{>{\columncolor{cyan!6}}c}
\newcolumntype{f}{>{\columncolor{lime!6}}c}
\newcolumntype{v}{>{\columncolor{purple!6}}c}
\newcolumntype{u}{>{\cellcolorval}c}

\definecolor{lightblue}{RGB}{130,169,217}
\definecolor{green}{RGB}{29,177,0}

\definecolor{Gray}{gray}{0.92}
\definecolor{racing-green}{rgb}{0.0, 0.8, 0.6}
\definecolor{awesome-red}{rgb}{1.0, 0.13, 0.32}
\definecolor{LightCyan}{rgb}{0.88,1,1}
\definecolor{darkgreen}{RGB}{0,150,0}
\definecolor{Ground}{RGB}{255,184,55}
\definecolor{Dirt}{RGB}{191,169,115}
\definecolor{Pink}{RGB}{226,184,176}
\definecolor{Violet}{RGB}{163,148,170}
\definecolor{darkred}{RGB}{150,0,0} %
\definecolor{bluelight}{RGB}{240,240,255}
\definecolor{greenight}{RGB}{240,255,240}
\definecolor{amb50}{HTML}{FAEEDA}
\definecolor{amb100}{HTML}{FAC775}
\definecolor{amb200}{HTML}{EF9F27}
\definecolor{amb400}{HTML}{BA7517}
\definecolor{ambtxt50}{HTML}{633806}
\definecolor{ambtxt100}{HTML}{412402}
\definecolor{ambtxtlt}{HTML}{FAEEDA}
\definecolor{bluov}{HTML}{E6F1FB}
\definecolor{bluovtxt}{HTML}{0C447C}

\begin{document}
\maketitle
\begin{abstract}
Calibration evaluates whether a model’s confidence aligns with its empirical accuracy. Existing studies often compare the calibration of different large language models using global calibration metrics such as Expected Calibration Error and Brier Score.
We begin by showing, both theoretically and empirically, that such comparisons are confounded by differences in model accuracy.
For fairer cross-model comparison, we then propose ACE, an accuracy-controlled evaluation framework with three complementary views: Instance-Aligned, Distribution-Aligned, and Candidate-Aligned calibration. Across multiple benchmarks, model families, and confidence elicitation methods, we use ACE to study two practically important comparison axes, small versus large models and thinking versus non-thinking models. We find that many previously reported calibration advantages under raw global metrics weaken substantially after accuracy control. We also find that ranking reversal is frequent: models favored by raw metrics often cease to be favored once accuracy is controlled. Our results show that raw global calibration metrics are not robust for cross-model comparison, and that fair calibration comparison requires accuracy-aware evaluation.\footnote{Project page: \url{https://github.com/chaochaoxx/ACE}.}

\end{abstract}

\section{Introduction}

\begin{figure*}[t]
    \centering
    \includegraphics[width=1.0\linewidth]{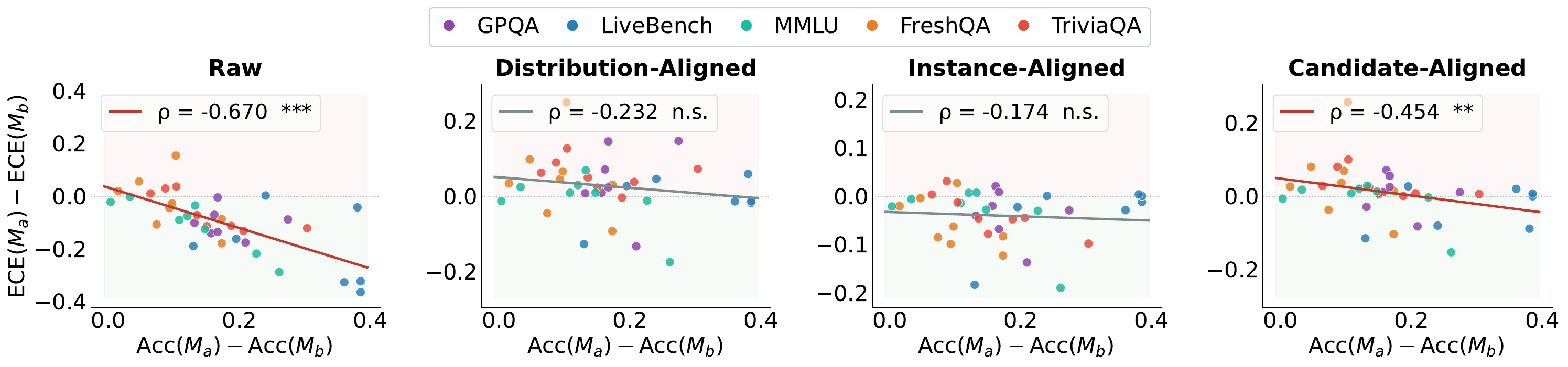}
   \caption{
\textbf{Accuracy differences strongly confound raw ECE comparison.}
Each point is one dataset-level average for a model pair ($\rho \!=\!-0.670$, $n\!=\!40$, 95\% CI $[-0.81,-0.45]$, $p<10^{-5}$; Appendix~\ref{app:empirical_analysis}). Larger accuracy gaps correspond to larger raw calibration gaps;
our accuracy-controlled views reduce this confound, most substantially under the Instance-Aligned (IA) and Distribution-Aligned (DA).}
    \label{fig:head}
\end{figure*}

Large Language Models (LLMs) have achieved remarkable capabilities, making their reliable deployment in high-stakes domains a critical priority~\citep{wang2023decodingtrust, huang2024trustllm, zhang-etal-2024-luq, zhang-etal-2025-atomic, yang-etal-2025-logu, yang-etal-2025-uncle}. In such settings, it is not enough for a model to achieve high average accuracy; it must also indicate how likely its prediction is to be correct \citep{zhang-etal-2026-confidence-estimation, zhang-etal-2026-lovec}. This has made calibration, the alignment between a model’s confidence and its empirical correctness, a central topic in recent work on trustworthy language models~\citep{xiong2023can, geng2024survey, zhou2025finallayerintermediaterepresentations, zhang-etal-2025-roads,  doi:10.36227/techrxiv.177130685.57616643/v1,li-etal-2025-large-language-models}. Global metrics such as Expected Calibration Error (ECE) and Brier Score (BS) are now widely used to quantify the calibration level of an LLM.

A common but underexamined assumption behind this practice is that \emph{global calibration scores are directly comparable across models}. Under this view, if one model achieves lower ECE/BS than another, it is natural to conclude that the former is better calibrated \citep{zeng2025thinking,mei2025reasoning,yoon2025reasoning,lacombe2025don}. However, we argue that this interpretation becomes problematic when the models being compared differ substantially in task accuracy. A model with higher accuracy can achieve a better global calibration score for two different reasons: it may be truly better calibrated, or it may simply make fewer mistakes. For example, if two models both assign confidence 1.0 to every answer, the more accurate one will automatically score better on ECE/BS, despite identical confidence behavior. In Section \ref{sec:background}, we first show that global calibration metrics are confounded by model accuracy, and therefore do not reflect calibration quality alone. This raises a basic but important question: \textbf{when two models operate at different accuracy levels, are their calibration scores really being compared on equal footing?}

In this work, we argue that naïve cross-model comparisons based on raw global calibration metrics can be misleading. 
When models differ greatly in accuracy, raw global metrics can systematically favor the stronger model even when the underlying confidence behavior is not correspondingly better. More importantly, this confound can produce what we term \textit{\textbf{ranking reversal}}: a model that appears better calibrated under raw global metrics can become \emph{less} calibrated once accuracy is controlled. We argue that ranking reversal is a particularly revealing diagnostic for calibration evaluation, because it makes the failure of naïve cross-model comparison directly observable. Rather than merely shifting the magnitude of a calibration gap, controlling for accuracy can overturn the conclusion about which model is better calibrated.

We therefore propose \textbf{ACE} (\textbf{A}ccuracy-\textbf{C}ontrolled \textbf{E}valuation), a framework for fairer cross-model calibration analysis. ACE compares models under matched or partially matched correctness conditions before drawing calibration conclusions. It includes three complementary views. First, \emph{Instance-Aligned} (IA) calibration compares models only on examples where they share the same outcome, providing strict instance-level control. Second, \emph{Distribution-Aligned} (DA) calibration matches the aggregate outcome composition of the compared models, enabling a broader distribution-level comparison. Third, \emph{Candidate-Aligned} (CA) calibration extends the same principle to evaluative settings with a shared answer space, where models judge a common pool of candidate answers. Together, these three views disentangle calibration quality from task accuracy, which raw global metrics conflate.

Using ACE across multiple benchmarks, model families, calibration metrics, and confidence elicitation methods, we find that many raw calibration conclusions are substantially less stable than they appear. Apparent calibration advantages often shrink once accuracy is controlled, and ranking reversal occurs frequently in both scale-based and thinking-versus-non-thinking model comparisons. These findings do not imply that calibration metrics such as ECE or BS are uninformative. Rather, they show that \emph{raw} global calibration scores are not always apples-to-apples across models with different accuracies. When the goal is fair cross-model comparison, accuracy control is not a minor refinement; it can change the conclusion itself.
\begin{itemize}[leftmargin=*, itemsep=0pt, topsep=0pt, parsep=0pt, partopsep=0pt]
    \item We show that raw comparisons of calibration across models can be misleading when models differ in accuracy, and identify a key \textbf{ranking reversal} phenomenon.
    \item We propose \textbf{ACE}, an accuracy-controlled evaluation framework with three complementary views for fairer comparison.
    \item Across multiple datasets, model families, and confidence elicitation methods, we show that many apparent calibration advantages weaken substantially under ACE, and that ranking reversals are common enough to materially change model-level conclusions.
\end{itemize}
\section{Related Work}


\paragraph{Calibration Metrics and Calibration Methods.}
Existing work studies calibration using both visualization-based and scalar metrics. Common examples include the \textit{reliability diagram} and \textit{Expected Calibration Error} (ECE), which summarize the gap between confidence and empirical accuracy across bins \citep{niculescu2005predicting,naeini2015obtaining}, as well as probabilistic scoring rules such as the \textit{Brier Score} and \textit{negative log-likelihood} (NLL) \citep{brier1950verification,lakshminarayanan2017simple}. Prior work has also examined the sensitivity of ECE-style analysis to evaluation design, including equal-mass, adaptive, and class-conditional variants \citep{nguyen2015posterior,nixon2019measuring}. To improve calibration, widely used post-hoc methods include \textit{temperature scaling}, \textit{Platt scaling}, and \textit{isotonic regression} \citep{platt1999probabilistic,zadrozny2002transforming,guo2017calibration}, while training-time approaches such as \textit{label smoothing} and \textit{deep ensembles} have also been explored \citep{pereyra2017regularizing,lakshminarayanan2017simple}. These methods are complementary to our work: rather than proposing a new way to recalibrate model confidences, we study whether calibration comparisons themselves are fair when models differ in task accuracy.

\paragraph{Fairness in Calibration Comparison.}
Recent work has shown that calibration evaluation itself can be sensitive to how it is defined and aggregated \citep{nixon2019measuring,vaicenavicius2019evaluating}. In particular, prior studies have shown that calibration conclusions may change with the choice of metric, binning scheme, and aggregation rule, and that standard ECE-style evaluation may reward confidence assignments that reduce aggregate error without improving the separation between correct and incorrect predictions \citep{si2022re}. Other work further suggests that calibration scores should not be interpreted independently of model accuracy, since more accurate models can obtain systematically more favorable scores under standard global metrics \citep{minderer2021revisiting,chidambaram2024reassessing}. Recent proposals such as \textit{CKCE} highlight that relative calibration comparison can be distorted by differences in predictive marginals \citep{moskvichev2025all}. However, these studies overlook a critical question: when models differ in accuracy, how can we fairly compare their calibration to ensure results reflect genuine confidence quality rather than mere task success? Our work addresses this gap by making accuracy the explicit control variable in cross-model calibration comparison and by showing that such control can change not only calibration gaps but also model rankings.
\section{Background and Motivation}
\label{sec:background}
Reliable confidence estimation is central to the safe deployment of LLMs. A standard lens for studying this property is \emph{calibration}, which asks whether confidence estimates match empirical correctness probabilities. In practice, calibration is often summarized by global metrics such as Expected Calibration Error (ECE) and Brier Score (BS). However, when models differ substantially in task accuracy, such comparisons become problematic. This section introduces the notation and shows, analytically and empirically, why accuracy confounds cross-model calibration comparisons, motivating our accuracy-controlled evaluation framework.

\paragraph{Notation and Standard Metrics.}
Let $D=\{(x_i,y_i)\}_{i=1}^N$ be an evaluation set, where $x_i$ is an input and $y_i$ is its gold label. Given a model $M$, let $\hat{y}_i$ denote its prediction and $c_i\in[0,1]$ the confidence assigned to that prediction. 
We define the correctness indicator $z_i = \mathbf{1}[\hat{y}_i = y_i]$,
and the empirical accuracy 
\[
\mathrm{Acc}(M) = \frac{1}{N}\sum_{i=1}^N z_i.
\]
To measure calibration, we consider two standard global metrics. 
First, ECE partitions the confidence range into bins $\{I_b\}_{b=1}^B$ and computes
\[
\mathrm{ECE}(M) = \sum_{b=1}^B \frac{|S_b|}{N} \left|\mathrm{acc}(S_b)-\mathrm{conf}(S_b)\right|,
\]
where $S_b \!=\!\{i \!:\! c_i\in I_b\}$, $\mathrm{acc}(S_b)=\frac{1}{|S_b|}\sum_{i\in S_b} z_i$, and $\mathrm{conf}(S_b)=\frac{1}{|S_b|}\sum_{i\in S_b} c_i$. 
Second, the Brier Score is
\[
\mathrm{BS}(M) = \frac{1}{N}\sum_{i=1}^N (c_i-z_i)^2.
\]
Both metrics are meaningful summaries of confidence quality for a fixed model on a fixed prediction set. Our concern is different: \emph{directly comparing these global scores across models with different accuracies can be severely confounded}. 

\paragraph{Why Accuracy Is a Confounder.}
The core problem is that global calibration metrics depend not only on how confidence is distributed, but also on the proportion of correct versus incorrect predictions. When a stronger model answers more questions correctly, its global ECE or BS can improve mechanically, even if its confidence behavior is not intrinsically better.

Intuitively, global metrics average over correct and incorrect predictions, so changing their mixture can change the score even when confidence behavior is fixed. This can be seen in a simple toy (with more rigorous discussion in Appendix~\ref{app:general_analysis}) setting where all predictions receive the same confidence $q\in[0,1]$. Let $a=\mathrm{Acc}(M)$. Then
\begin{align*}
\mathrm{ECE}(M) &= |a-q|, \\
\mathrm{BS}(M)  &= a(1-q)^2+(1-a)q^2 \\
&= q^2 + a(1-2q).
\end{align*}
Thus, BS depends linearly on accuracy. In the practically relevant high-confidence regime $q>0.5$, $\mathrm{BS}(M)$ decreases as $a$ increases. Likewise, ECE is piecewise linear in $a$; under the common overconfident regime $a\le q$, it simplifies to
\[
\mathrm{ECE}(M)=q-a,
\]
which also decreases as accuracy increases. Therefore, even when two models exhibit the same nominal confidence level, the more accurate one can obtain a better global calibration score simply because it has more correct predictions.

\paragraph{Empirical Evidence and Motivation.}
This controlled example motivates an empirical diagnostic. As shown in the ``Raw'' panel of Figure~\ref{fig:head}, experiments across mainstream LLM families and benchmark tasks reveal a strong negative correlation ($\rho = -0.670$) between global calibration scores and task accuracy. 
Models with higher accuracy mechanically appear better calibrated under raw ECE, even when they remain highly overconfident. This pattern validates our concern: a substantial part of the reported ``calibration advantage'' of stronger models may simply reflect outcome composition rather than genuinely superior uncertainty estimation. This empirical trend motivates the central question of our paper: \textit{how should we compare calibration fairly when models have different cognitive capabilities?}
\section{ACE: Accuracy-Controlled Evaluation}

We introduce \textsc{ACE} (\textbf{A}ccuracy-\textbf{C}ontrolled \textbf{E}valuation), a framework for fairer cross-model calibration comparison. The key idea is simple: when two models differ in task accuracy, their raw global calibration scores are not directly comparable because they are computed over prediction sets with different outcome compositions. \textsc{ACE} addresses this by constructing comparison protocols in which predictive success is explicitly controlled before calibration is measured. These views are intended to be read jointly rather than as interchangeable metrics: agreement indicates a stable conclusion, while disagreement identifies the relevant alignment choice. The framework contains two views for \emph{generative calibration}, where models score their own answers, and one complementary view for \emph{evaluative calibration}, where they judge a shared answer space.

\paragraph{Setup.}
Consider two models, $M_A$ and $M_B$, evaluated on the same dataset $D=\{(x_i,y_i)\}_{i=1}^N$.
For each example $x_i$, model $M_m$ produces a prediction $\hat{y}_i^{(m)}$, a confidence score $c_i^{(m)}\in[0,1]$, and a correctness indicator
\vspace{-2pt}
$$
z_i^{(m)}=\mathbf{1}\!\left[\hat{y}_i^{(m)}=y_i\right], \quad m\in\{A,B\}.
$$
Let $\mathrm{Acc}(M_m)=\frac{1}{N}\sum_{i=1}^N z_i^{(m)}$ denote the model's accuracy. We focus on binary correctness here; Appendix~\ref{app:graded-correctness} extends the formulation to graded correctness. A naive evaluation would directly compare global scores like $\mathrm{ECE}(M_A)$ and $\mathrm{ECE}(M_B)$ over the full dataset $D$. 
However, as established in Section~\ref{sec:background}, such comparisons are inherently confounded when $\mathrm{Acc}(M_A)\neq \mathrm{Acc}(M_B)$. To address this, \textsc{ACE} evaluates cross-model calibration strictly under controlled accuracy conditions.

\subsection{Instance-Aligned (IA) Calibration}

Our first view enforces \emph{instance-level alignment}. 
We compare two models only on examples where they have the same outcome, \ie both are correct or both are wrong. Formally, we define
\begin{align*}
D_{\mathrm{both\mbox{-}right}} &= \left\{i : z_i^{(A)}=1 \land z_i^{(B)}=1\right\}, \\
D_{\mathrm{both\mbox{-}wrong}} &= \left\{i : z_i^{(A)}=0 \land z_i^{(B)}=0\right\},
\end{align*}
and $D_{\mathrm{IA}} = D_{\mathrm{both\mbox{-}right}} \cup D_{\mathrm{both\mbox{-}wrong}}$.
By construction, the two models have identical empirical accuracy on $D_{\mathrm{IA}}$:
\[
\mathrm{Acc}_{D_{\mathrm{IA}}}(M_A) = \mathrm{Acc}_{D_{\mathrm{IA}}}(M_B) = \frac{|D_{\mathrm{both\mbox{-}right}}|}{|D_{\mathrm{IA}}|}.
\]
This makes $D_{\text{IA}}$ a particularly clean comparison set: because correctness is aligned \textbf{instance by instance}, calibration differences can be interpreted more directly as differences in confidence assignment rather than differences in task success. IA retains substantial coverage in our experiments (median 78.0\%, IQR [66.3\%, 82.2\%], min 55.2\%; Appendix~\ref{app:ia-retention}).

We then compute calibration metrics (\eg ECE and BS) on $D_{\mathrm{IA}}$. A useful feature of IA is its decomposition. $D_{\mathrm{both\mbox{-}right}}$ captures shared competence, while $D_{\mathrm{both\mbox{-}wrong}}$ captures shared blind spots. 
The latter is especially informative: if a stronger model still assigns systematically higher confidence on $D_{\text{both-wrong}}$, then its apparent global calibration advantage may partly mask a more severe form of overconfidence on questions that neither model actually understands. Unlike comparing calibration separately on each model's correct and incorrect predictions, this decomposition holds both the outcome and the evaluated instances fixed across models. This instance-aligned analysis therefore reveals \emph{where} calibration differences come from, not merely whether one global score is lower than another.

\subsection{Distribution-Aligned (DA) Calibration}

Our second view enforces \emph{distribution-level alignment}. Instead of matching models instance by instance, DA matches the overall outcome composition so that calibration is compared under a common effective accuracy level. Concretely, rather than requiring the two models to agree on the same instances, DA only requires them to be compared under the same aggregate proportion of correct and incorrect predictions.

A natural way to achieve this is to rebalance one model's predictions to match the empirical accuracy of the other. One implementation is resampling, which selectively discards observations until the target outcome composition is reached. While simple, this approach introduces sampling variance and reduces effective sample size. We therefore consider a more rigorous alternative that retains all
observations: \emph{Distribution-Aligned} (DA). DA uses importance weighting, a technique widely used in causal inference to correct for confounded comparisons~\citep{rosenbaum1983central}, to reweight each
prediction so that the weighted outcome composition matches a target accuracy (see Appendix~\ref{app:ipw-justification} for a detailed justification).

Let $a_A=\mathrm{Acc}(M_A)$, $a_B=\mathrm{Acc}(M_B)$ with $a_A>a_B$. 
To evaluate $M_A$ at the accuracy level of $M_B$, we define
\[
w_i^{(A)} =
\begin{cases}
a_B / a_A & \text{if } z_i^{(A)} = 1 \\[4pt]
(1 - a_B) / (1 - a_A) & \text{if } z_i^{(A)} = 0
\end{cases},
\]
which satisfies $\sum_i w_i^{(A)} z_i^{(A)} \big/ \sum_i w_i^{(A)} \!= \!a_B$. Thus, the weighted correctness composition of $M_A$ matches the target accuracy $a_B$. The resulting DA calibration metrics are computed as:
\begin{align*}
\mathrm{DA\mbox{-}BS}_{a_{A \to B}} &=
\frac{1}{W} \sum_i w_i^{(A)}\left(c_i^{(A)}-z_i^{(A)}\right)^2, \\
\mathrm{DA\mbox{-}ECE}_{a_{A \to B}} &=
\sum_{b=1}^B \frac{W_b}{W} \left|\mathrm{acc}_w(S_b)\!-\!\mathrm{conf}_w(S_b)\right|,
\end{align*}
where $W \!=\! \sum_i w_i^{(A)}$ is the total weight, $W_b = \sum_{i\in S_b} w_i^{(A)}$ is the bin weight, and $\mathrm{acc}_w$ and $\mathrm{conf}_w$ are weighted bin-level accuracy and confidence.

DA trades the strictness of IA for broader coverage. Rather than enforcing agreement on the same individual instances, it aligns only the overall distribution of correct and incorrect outcomes, allowing comparison on a larger effective sample. This makes DA a useful test of whether an apparent calibration advantage persists once differences in aggregate outcome composition are removed. 
 
\subsection{Candidate-Aligned (CA) Calibration}

The two views above concern \emph{generative calibration}: each model generates its own answer and reports confidence on that answer. Our third view removes this asymmetry by aligning the \emph{candidate answer space}.

For each input $x_i$, let $s_i^{(A)}$ and $s_i^{(B)}$ be the answers generated by $M_A$ and $M_B$, and define the shared candidate pool
$U_i=\{s_i^{(A)},\, s_i^{(B)}\}$.
More generally, $U_i$ can include any set of candidate answers under consideration. Each model then evaluates every candidate in the same pool. Let $e_i^{(m)}(s)\in[0,1]$
denote model $M_m$'s confidence that candidate $s\in U_i$ is correct, and let $y_i(s)\in\{0,1\}$
be the ground-truth correctness label of candidate $s$.
CA is then computed at the \emph{candidate level}: each pair $(x_i,s)$ is treated as one evaluation instance, and calibration compares $e_i^{(m)}(s)$ against $y_i(s)$ over all $i$ and all $s\in U_i$. For example,
$$
\mathrm{CA\mbox{-}BS}(M_m)
\!=\!
\frac{1}{\sum_i |U_i|}
\!\sum_i\! \sum_{s\in U_i}
\left(e_i^{(m)}(s)\!-\!y_i(s)\right)^2
$$
with CA-ECE defined analogously by binning the values $e_i^{(m)}(s)$ across all candidate-level instances.

As both models score the same candidates, CA controls for differences in generation space and isolates calibration in evaluation mode. It also exposes \emph{self-preference bias}: a model may assign higher confidence to its own answer than to a competing answer, even when the competing answer is better supported.

\section{Experiments}


\subsection{Experiment Setup}

\paragraph{Models.}
We study models along two comparison axes: \textit{small vs.\ large} and \textit{thinking vs.\ non-thinking}. For scale-based comparison, we evaluate Qwen2.5 (7B and 72B) \citep{qwen2025qwen25technicalreport}, LLaMA~3.1 (8B and 70B) \citep{grattafiori2024llama3herdmodels}, and GPT-OSS (20B and 120B) \citep{openai2025gptoss120bgptoss20bmodel}. For reasoning-mode comparison, we evaluate Qwen3 (8B and 32B) \citep{yang2025qwen3technicalreport} and GPT-OSS (20B and 120B) under different reasoning settings \citep{openai2025gptoss120bgptoss20bmodel}. These comparison axes are motivated by prior findings that larger models often produce better calibrated confidence estimates than smaller ones, although the gains can be task-dependent \citep{xiong2023can,fang2025credence}. Prior work also suggests that reasoning-enabled models often calibrate better than non-reasoning models on reasoning-intensive tasks, while the picture is more mixed on factual knowledge tasks \citep{yoon2025reasoning,zeng2025thinking}. Additional inference settings and answer-generation prompts are provided in Appendix~\ref{app:answer-prompts} and Appendix~\ref{app:model-details}.

\paragraph{Datasets.}
We evaluate on five datasets spanning two categories: \textit{knowledge-based} and \textit{reasoning-based} tasks. The knowledge-based datasets are FreshQA and TriviaQA \citep{vu2023freshllmsrefreshinglargelanguage,joshi2017triviaqa}, both in an open-ended factual QA format. The reasoning-based datasets are GPQA Diamond, LiveBench, and MMLU-Pro \citep{rein2023gpqagraduatelevelgoogleproofqa,white2025livebenchchallengingcontaminationlimitedllm,wang2024mmluprorobustchallengingmultitask}, which cover expert-level multiple-choice reasoning, general reasoning, and academic problem solving, respectively. The answer formats include both open-ended and multiple-choice settings. Additional dataset details are provided in  Appendix~\ref{app:data-details}.

\paragraph{Confidence Elicitation Methods.}
We consider three representative confidence elicitation methods: \textit{verbalized confidence} \citep{tian2023just}, \textit{P(True)} \citep{kadavath2022language}, and \textit{self-consistency} \citep{wang2022self,manakul2023selfcheckgpt}. Additional prompting templates and implementation details are provided in Appendix~\ref{app:confidence-details}.


\subsection{Main Results}
\begin{figure*}[ht!]
    \centering
    \includegraphics[width=1.0\linewidth]{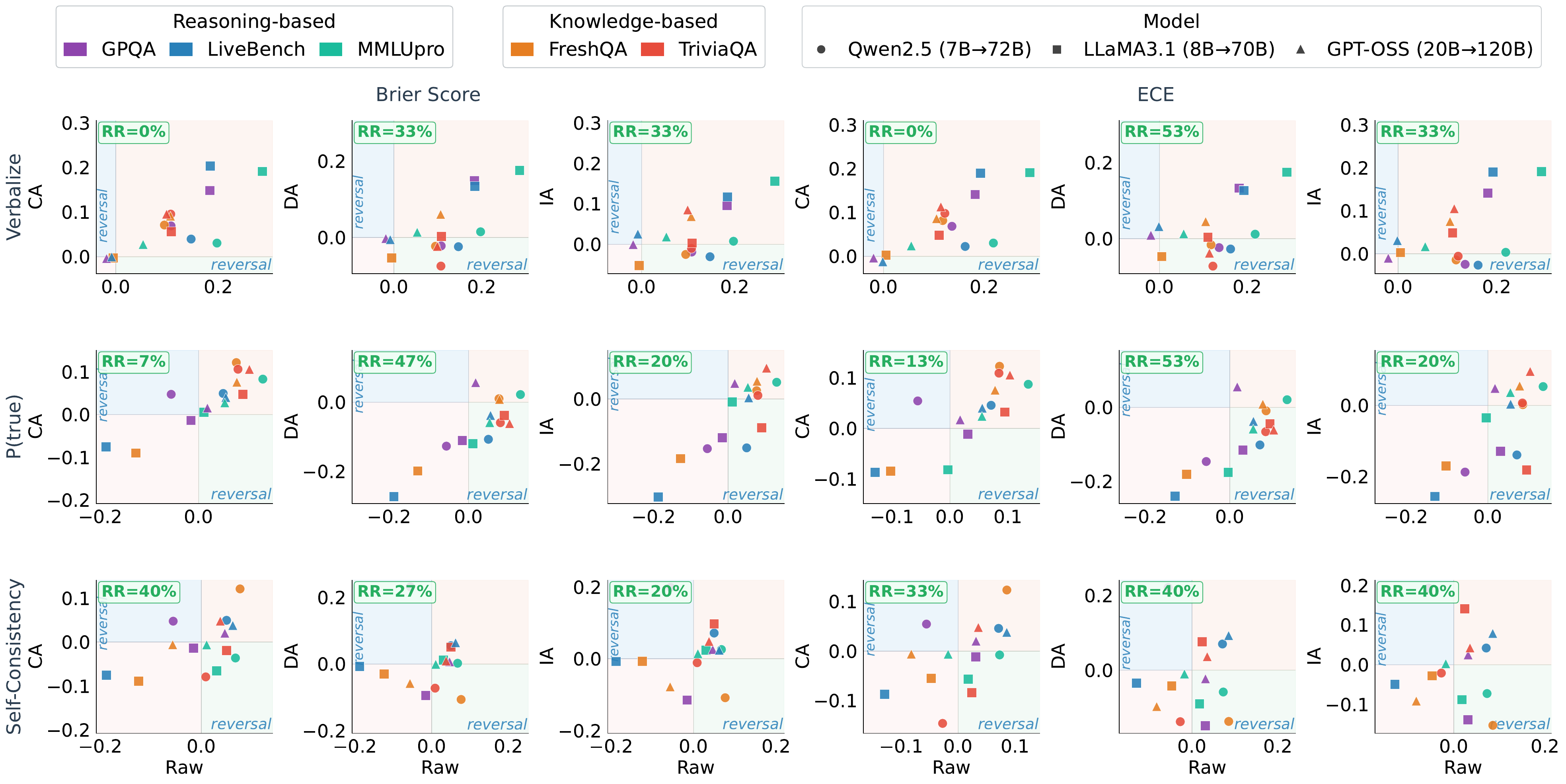}
    \caption{Calibration gap between small and large models within the same family, plotted before (\textit{x}-axis, \textit{Raw}) and after accuracy alignment (\textit{y}-axis). Each point represents one (dataset, calibration metric, confidence elicitation method) combination. Points in \textbf{Quadrants II and IV} (\textit{reversal}) indicate cases where the raw and alignment-based rankings disagree: the model favored by raw ECE/BS is no longer favored after accuracy control.}
    \label{fig:main_results_01}
\end{figure*}

\begin{figure*}[ht!]
    \centering
    \includegraphics[width=1.0\linewidth]{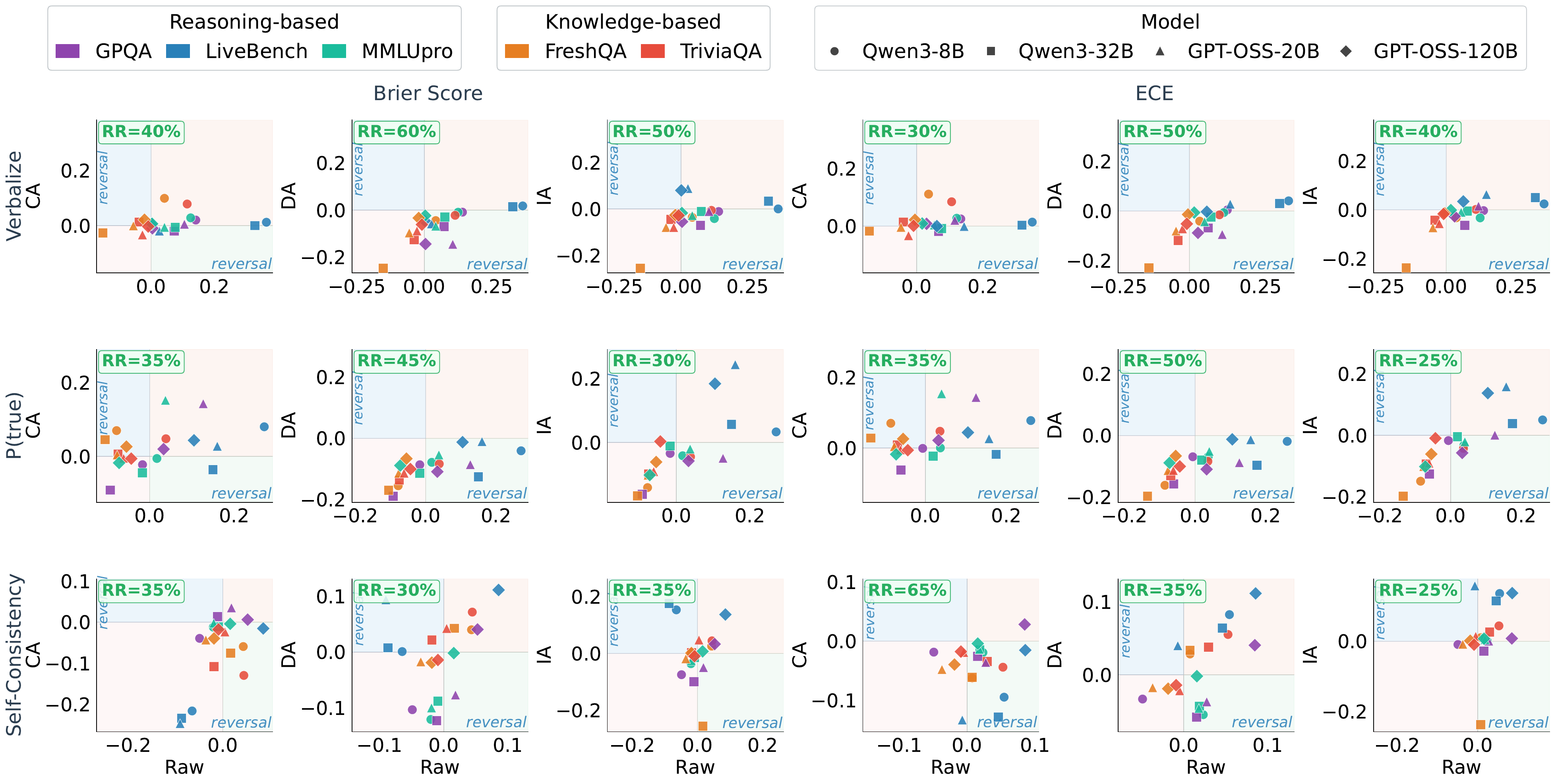}
    \caption{Calibration gap between non-thinking and thinking models, using the same setup as Figure~\ref{fig:main_results_01}.}
    \label{fig:main_results_02}
\end{figure*}

\textbf{Ranking reversal occurs frequently}, and this pattern is robust across metrics and confidence elicitation methods.
We define a \emph{ranking reversal} as a case where the model that appears better calibrated under the raw comparison becomes less calibrated after ACE. In Figures~2 and~3, such reversals correspond to points falling in the second and fourth quadrants, where the raw and ACE-based comparisons disagree in sign. We widely observe ranking reversals under both ECE and BS, and across confidence estimates. Across the 265 reversals in the original
estimates, the median probability that the Raw and aligned gaps retain opposite signs under paired bootstrap resampling is 91.9\%
(Appendix~\ref{app:bootstrap-reversals}). This means that calibration conclusions drawn from raw global metrics can be misleading when the compared models differ in accuracy.

\paragraph{For scale-based comparisons, the calibration advantage of larger models weakens markedly after ACE.}
Figure~2 shows that, when comparing Qwen2.5, LLaMA~3.1, and GPT-OSS model pairs, reversal rates are common across all three confidence elicitation methods and both metrics. Depending on the setting, the reversal rate ranges from 0\% to 53\%, with the strongest effects typically appearing under Distribution-Alignment and Instance-Alignment. For example, under verbalized confidence, the reversal rate reaches 53\% for ECE under DA, and under $P(\mathrm{true})$ it reaches 53\% in some settings. These results indicate that part of the apparent calibration advantage of larger models under raw global metrics is attributable to their higher task accuracy rather than to uniformly better confidence quality.

\paragraph{For thinking-versus-non-thinking comparisons, controlled evaluation also overturns many raw rankings.}
Figure~3 shows the same qualitative pattern when comparing thinking and non-thinking modes within the same model family. Here, reversal remains frequent after accuracy alignment, and in some settings becomes even more pronounced: for instance, under self-consistency, the reversal rate reaches 65\% when Candidate-Aligned, while several Distribution-Aligned settings fall in the 30\%--60\% range. Even under verbalized confidence, where CA can be relatively conservative, DA and IA still produce substantial reversal rates such as 60\% and 50\%. This suggests that some of the raw calibration gains attributed to thinking are mediated by improved answer accuracy, rather than reflecting a purely metacognitive improvement in confidence estimation.

\paragraph{Taken together, the results show that raw global calibration scores are often not apples-to-apples across models with different accuracies.}
Our main empirical finding is therefore not that larger or thinking models are poorly calibrated, but that raw global calibration metrics can confound confidence quality with outcome composition. Once accuracy is explicitly controlled, several widely reported calibration advantages shrink substantially, and in many cases the ranking itself changes. These findings support the use of ACE-style evaluation as a standard companion to raw calibration reporting when the goal is fair cross-model comparison.

\section{Analysis}
\begin{figure}[t]
    \centering
    \includegraphics[width=0.9\linewidth]{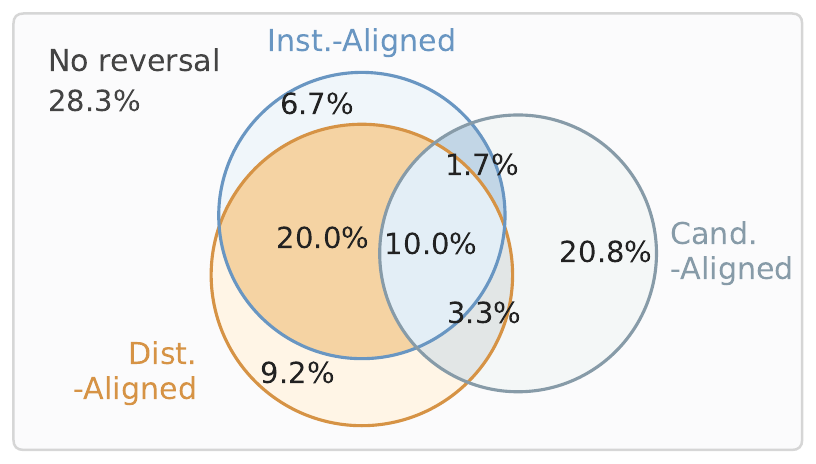}
    \caption{Overall reversal-combination distribution relative to the raw ECE winner. 
    The figure highlights that Dist.-Aligned and Inst.-Aligned agree in 79.2\% of cases, whereas Cand.-Aligned more often behaves as a separate comparison axis.}
    \label{fig:figure_03}
\end{figure}
\paragraph{Distribution-Aligned (DA) and Instance-Aligned (IA) usually overturn raw winners together, while Candidate-Aligned (CA) more often acts alone.}
As shown in Figure~\ref{fig:figure_03}, the dominant
reversal structure is not arbitrary: across all 120 cases, the three most
common patterns are no reversal (28.3\%), CA-only reversal
(20.8\%), and joint DA+IA reversal (20.0\%),
which together account for 69.2\% of all comparisons. More importantly,
DA and IA share the same reversal status in 79.2\% of cases, whereas
CA more often departs from them as a separate comparison axis. This suggests
that DA and IA often capture related corrections to the raw ranking, even
though the structure
is not fully binary, since all-three reversal still occurs in 10.0\% of cases.
The same asymmetry also appears in the knowledge-versus-reasoning split:
across method-level reversal rates, DA and IA generally move in more similar
directions, whereas CA is more likely to diverge.
This asymmetry is also
consistent with the definitions of the three views. DA and IA both still
evaluate confidence on each model's own answers and mainly differ in how they
align accuracy, whereas CA changes the candidate answer pool itself and
therefore more easily yields a separate ranking.

\paragraph{The source of reversal depends strongly on the confidence method.}
Table~\ref{tab:reversal_by_method} shows that the three confidence methods do not share a single reversal signature. Verbalized
confidence is dominated by joint DA+IA
reversal (35.0\%), suggesting that its raw ranking is especially sensitive to
distribution-aligned and instance-aligned comparison. SC follows a different
pattern: CA-only reversal is the most common outcome (35.0\%), indicating substantially greater sensitivity to candidate-pool alignment than to DA or IA comparison. P(True), in contrast, shows the most diffuse profile, with substantial mass on both DA-only (22.5\%) and CA-only (20.0\%) reversals rather than a single dominant pattern. The key point is not merely that reversal frequencies differ, but that different methods are destabilized by
different kinds of alignment.
\begin{table}[t]
\centering
\setlength{\tabcolsep}{13.2pt}
\scalebox{0.82}{%
\begin{tabular}{lcccc}
\toprule
\multirow{2}{*}{Method}
  & \multicolumn{3}{c}{Exclusive RR (\%)}
  & \multirow{2}{*}{\shortstack{NRR \\ (\%)}} \\
\cmidrule(lr){2-4}
  & IA & DA & CA & \\
\midrule
Verbalized
  & \cellcolor{amb50}\textcolor{ambtxt50}{7.5}
  & \cellcolor{amb50}\textcolor{ambtxt50}{2.5}
  & \cellcolor{amb50}\textcolor{ambtxt50}{7.5}
  & \cellcolor{bluov}\textcolor{bluovtxt}{27.5} \\
P(True)
  & \cellcolor{amb50}\textcolor{ambtxt50}{2.5}
  & \cellcolor{amb100}\textcolor{ambtxt100}{22.5}
  & \cellcolor{amb100}\textcolor{ambtxt100}{20.0}
  & \cellcolor{bluov}\textcolor{bluovtxt}{30.0} \\
SC
  & \cellcolor{amb100}\textcolor{ambtxt100}{10.0}
  & \cellcolor{amb50}\textcolor{ambtxt50}{2.5}
  & \cellcolor{amb200}\textcolor{ambtxt100}{35.0}
  & \cellcolor{bluov}\textcolor{bluovtxt}{27.5} \\
\bottomrule
\end{tabular}%
}
\caption{Method-level ranking reversal rates (RR). IA, DA, and CA columns report the
fraction of cases where \emph{exclusively} that view overturns the
Raw ECE winner. NRR (no-reversal rate) is the fraction of cases where
no view overturns the Raw ECE winner.}
\label{tab:reversal_by_method}
\end{table}

\paragraph{Dataset structure determines which kind of reversal becomes dominant.}
As illustrated in Table~\ref{tab:reversal_by_dataset}, reversal
consistency is also highly dataset-dependent. MMLU-Pro exhibits the highest
no-reversal rate (50.0\%), indicating comparatively stable raw rankings. By
contrast, FreshQA and TriviaQA are more often dominated by CA
reversals (45.8\% and 33.3\%, respectively), suggesting that candidate-pool
alignment plays a larger role in open-ended knowledge QA. GPQA
Diamond has the highest all-three-reversal rate (16.7\%), meaning that its raw
winner is more likely to be overturned jointly by all three aligned views.
LiveBench Reasoning shows the most diffuse pattern, with mass spread across
several reversal types rather than concentrated in a single dominant
signature. Taken together, these dataset-level differences indicate that the
stability of raw ECE ranking depends not only on the confidence method, but
also on task format and evaluation structure.
\begin{table}[t]
\centering
\setlength{\tabcolsep}{13.2pt}
\scalebox{0.82}{%
\begin{tabular}{lcccc}
\toprule
\multirow{2}{*}{Dataset}
  & \multicolumn{3}{c}{Exclusive RR (\%)}
  & \multirow{2}{*}{\shortstack{NRR \\ (\%)}} \\
\cmidrule(lr){2-4}
  & IA & DA & CA & \\
\midrule
FreshQA
  & \cellcolor{amb50}\textcolor{ambtxt50}{8.3}
  & \cellcolor{amb50}\textcolor{ambtxt50}{4.2}
  & \cellcolor{amb400}\textcolor{ambtxtlt}{45.8}
  & \cellcolor{bluov}\textcolor{bluovtxt}{25.0} \\
GPQA
  & \cellcolor{amb50}\textcolor{ambtxt50}{0.0}
  & \cellcolor{amb100}\textcolor{ambtxt100}{16.7}
  & \cellcolor{amb100}\textcolor{ambtxt100}{12.5}
  & \cellcolor{bluov}\textcolor{bluovtxt}{25.0} \\
LiveBench
  & \cellcolor{amb100}\textcolor{ambtxt100}{16.7}
  & \cellcolor{amb100}\textcolor{ambtxt100}{12.5}
  & \cellcolor{amb50}\textcolor{ambtxt50}{8.3}
  & \cellcolor{bluov}\textcolor{bluovtxt}{20.8} \\
TriviaQA
  & \cellcolor{amb50}\textcolor{ambtxt50}{8.3}
  & \cellcolor{amb50}\textcolor{ambtxt50}{8.3}
  & \cellcolor{amb200}\textcolor{ambtxt100}{33.3}
  & \cellcolor{bluov}\textcolor{bluovtxt}{20.8} \\
MMLU-Pro
  & \cellcolor{amb50}\textcolor{ambtxt50}{0.0}
  & \cellcolor{amb50}\textcolor{ambtxt50}{4.2}
  & \cellcolor{amb50}\textcolor{ambtxt50}{4.2}
  & \cellcolor{bluov}\textcolor{bluovtxt}{50.0} \\
\bottomrule
\end{tabular}%
}
\caption{Dataset-level ranking reversal rates (RR). IA, DA, and CA columns report the
fraction of cases where \emph{exclusively} that view overturns the
Raw ECE winner. 
}
\label{tab:reversal_by_dataset}
\end{table}

Additional analyses on reversal patterns and confidence distributions are in Appendix~\ref{app:additional-analyses}.

\paragraph{Practical Recommendations.} We recommend using DA as the default reporting metric, as it retains the full sample and agrees with IA in 79.2\% of cases (Figure~\ref{fig:figure_03}). IA should serve as a corroborating check; their agreement indicates a stable conclusion. Finally, treat CA as a diagnostic tool: if CA diverges from DA and IA, the ranking is likely sensitive to candidate-pool alignment rather than indicating a robust reversal.

\section{Conclusion}
We introduced \textsc{ACE}, an accuracy-controlled evaluation framework for fairer calibration comparison in large language models. Our analysis shows that standard global calibration metrics can be substantially confounded by differences in task accuracy, making raw cross-model comparisons misleading. To address this, \textsc{ACE} provides three complementary views that compare models under matched or partially matched correctness conditions. 
Across model families, datasets, and confidence elicitation methods, we find that many apparent calibration advantages weaken once accuracy is controlled, and that ranking reversals are common, suggesting that standard evaluation can overstate the metacognitive advantage of stronger models. 
Our results highlight the importance of accuracy-aware calibration analysis and position \textsc{ACE} as a useful companion to conventional calibration reporting for more reliable assessment of LLM confidence quality.
\section*{Limitations}
Our study also has several limitations that suggest directions for future work. First, our experiments focus on open-source models, which allow controlled and reproducible evaluation across model families, but further validation on proprietary systems would be valuable to assess how broadly the observed trends extend. Second, while \textsc{ACE} provides three complementary views for accuracy-controlled comparison, each view has its own tradeoff: IA offers strict instance-level control but depends on shared-outcome coverage, DA retains the full sample through reweighting but relies on distributional alignment, and CA exposes candidate-pool and self-preference effects but should be read as diagnostic rather than as a replacement for own-output calibration. Accordingly, we interpret the three views jointly rather than treating any single view as a complete solution. Third, \textsc{ACE} is intended as a diagnostic framework, and connecting its accuracy-controlled rankings to downstream applications such as selective prediction, abstention, or routing remains future work. Finally, our main experiments use a binary correctness variable, $z_i \in \{0,1\}$, which is a natural fit for multiple-choice and exact-match benchmarks but less directly applicable to tasks where performance is graded with partial credit or other continuous measures.
In the Appendix, we extend the framework to graded correctness and evaluate this extension on two datasets. Broader evaluation in such settings remains an important direction for future research.

\section*{Ethics Statement}
Our research adheres to strict ethical guidelines. We verified the licenses of all software and datasets used in this study to ensure full compliance with their terms. No privacy concerns have been identified. We have conducted a thorough assessment of the project and do not anticipate any further risks. We only use LLMs for grammar checking during the paper writing.

\section*{Acknowledgments}
This work was supported by the Chinese NSF General Program (No.~62572129).

\bibliography{anthology,custom}

\clearpage
\appendix

\section{General Analysis}
\subsection{Theoretical Analysis}
\label{app:general_analysis}
\paragraph{Why accuracy is a confounder}
Global calibration metrics depend not only on how confidence is distributed, but also on the proportion of correct versus incorrect predictions. 
We formalize this mechanism in a deliberately controlled limit case, which is weaker than requiring uniform confidence across all predictions but is not intended as a full model of real LLM behavior.

\begin{assumption}[Shared Conditional Confidence]
\label{assump:shared}
Two models $M_A$ and $M_B$, evaluated on the same dataset, share identical conditional confidence distributions:

\begin{itemize}[leftmargin=*]
\item \emph{Correct} ($z=1$): confidence $c \sim P_1(c)$, with mean $\mu_1 = \mathbb{E}[c \mid z=1]$ and variance $\sigma_1^2 = \mathrm{Var}(c \mid z=1)$.
\item \emph{Incorrect} ($z=0$): confidence $c \sim P_0(c)$, with mean $\mu_0 = \mathbb{E}[c \mid z=0]$ and variance $\sigma_0^2 = \mathrm{Var}(c \mid z=0)$.
\end{itemize}
The two models differ solely in their accuracy $a = \mathrm{Acc}(M)$.
\end{assumption}
Under this assumption, the two models exhibit identical confidence \emph{behavior} conditioned on correctness, differing only in how often they are right. 
This isolates accuracy as the only source of variation: any resulting calibration difference is attributable to accuracy rather than to intrinsic differences in uncertainty estimation. We use this as an illustrative sufficient case, not as a claim that real LLMs differ only in accuracy.

\begin{proposition}[Limit-case: BS Linear in Accuracy]
\label{prop:bs}
Under Assumption~\ref{assump:shared}, the Brier Score is a linear 
function of accuracy $a$. Let $L_0 = \sigma_0^2 + \mu_0^2$, $\Delta_L = \sigma_1^2 + (1-\mu_1)^2 - \sigma_0^2 - \mu_0^2$. Then 
$$
\mathrm{BS}(a) = L_0 + a \cdot \Delta_L.
$$
\end{proposition}

\begin{proof}
By the law of total expectation, decompose the Brier Score by correctness:
\begin{equation}
\begin{aligned}
\mathrm{BS}
&= \mathbb{E}\!\left[(c - z)^2\right] \notag = a \cdot \mathbb{E}\!\left[(c-1)^2 \mid z=1\right] \\
 &+ (1-a) \cdot \mathbb{E}\!\left[c^2 \mid z=0\right].
 \label{eq:bs_decomp}
\end{aligned}    
\end{equation}
For the two conditional terms, applying the bias-variance decomposition gives
\begin{equation*}
\begin{aligned}
\mathbb{E}\!\left[(c-1)^2 \mid z=1\right]
&= \mathrm{Var}(c \mid z=1) \\
&+ \bigl(\mathbb{E}[c\mid z=1]-1\bigr)^2\\
&= \sigma_1^2 + (1-\mu_1)^2, \\
\mathbb{E}\!\left[c^2 \mid z=0\right]
&= \mathrm{Var}(c \mid z=0) \\
&+ \bigl(\mathbb{E}[c\mid z=0]\bigr)^2
 = \sigma_0^2 + \mu_0^2.
\end{aligned}
\end{equation*}
Substituting into Equation \eqref{eq:bs_decomp}, we have 
\begin{equation*}
\begin{aligned}
&\mathrm{BS}(a)
= a\bigl[\sigma_1^2 + (1-\mu_1)^2\bigr] 
 + (1-a)\bigl[\sigma_0^2 + \mu_0^2\bigr] \\
&= \bigl[\sigma_0^2 + \mu_0^2\bigr] 
 + a\bigl[\sigma_1^2 + (1-\mu_1)^2 - \sigma_0^2 - \mu_0^2\bigr] \\
&= L_0 + a \cdot \Delta_L. 
\end{aligned}
\end{equation*}
\end{proof}

\begin{remark}[Slope condition is easily satisfied]
The condition $\Delta_L < 0$ holds when correct predictions 
deviate less from their ideal confidence than incorrect ones. 
Since the ideal confidence is $1$ for correct predictions and 
$0$ for incorrect ones, this reduces to
$$
\sigma_1^2 + (1-\mu_1)^2 < \sigma_0^2 + \mu_0^2,
$$
which is readily satisfied in practice: well-functioning models 
assign higher confidence to correct predictions ($\mu_1 > \mu_0$), 
making $(1-\mu_1)^2$ small, while overconfidence on incorrect 
predictions keeps $\mu_0^2$ substantial. Consequently, higher 
accuracy mechanically improves BS even when the underlying 
confidence behavior is identical.
\end{remark}

\begin{remark}[Constant-confidence case]
When $P_1$ and $P_0$ both degenerate to a point mass at $q$,
i.e., $\mu_1=\mu_0=q$ and $\sigma_1=\sigma_0=0$, Proposition~\ref{prop:bs} simplifies to
\[
    \mathrm{BS}(a) = q^2 + a(1-2q),
\]
recovering the toy example of Section~\ref{sec:background}.
\end{remark}

\paragraph{Extension to ECE.}
Unlike the Brier Score, ECE includes binning and absolute values, so its dependence on accuracy is not globally linear. Under a fixed binning and a common overconfidence regime, however, the mechanism can be stated formally.

\begin{proposition}[ECE confound under overconfidence]
\label{prop:ece_confound}
Suppose Assumption~\ref{assump:shared} holds, $a_B\geq a_A$, and a fixed binning is used. If every populated bin is overconfident for both models, then
\[
\mathrm{ECE}_A-\mathrm{ECE}_B
=(a_B-a_A)(1-\mu_1+\mu_0)\geq 0.
\]
Thus, with identical confidence behavior conditioned on correctness, higher accuracy alone cannot increase ECE under this sufficient condition.
\end{proposition}

\begin{proof}
Since every populated bin is overconfident, the absolute value in each bin can be removed in the same direction. Therefore,
\begin{align*}
\mathrm{ECE}(a)
&=\sum_b \Pr(S_b)
\bigl(\mathrm{conf}(S_b)-\mathrm{acc}(S_b)\bigr)\\
&=\mathbb{E}[c]-\mathbb{E}[z]\\
&=a\mu_1+(1-a)\mu_0-a\\
&=\mu_0-a(1-\mu_1+\mu_0).
\end{align*}
Taking the difference between models $A$ and $B$ gives
\[
\mathrm{ECE}_A-\mathrm{ECE}_B
=(a_B-a_A)(1-\mu_1+\mu_0).
\]
The result is nonnegative because $a_B\geq a_A$, $\mu_1\leq1$, and $\mu_0\geq0$.
\end{proof}
\begin{remark}[Constant-confidence ECE case]
When $P_1$ and $P_0$ both degenerate to a point mass at $q$, we have $\mu_1=\mu_0=q$. Under the overconfident regime $a\leq q$ considered in Section~\ref{sec:background}, Proposition~\ref{prop:ece_confound} reduces to
\[
\mathrm{ECE}(a)=q-a,
\]
recovering the ECE toy example in Section~\ref{sec:background}. Consequently, for two models satisfying $a_A\leq a_B\leq q$,
\[
\mathrm{ECE}_A-\mathrm{ECE}_B=a_B-a_A\geq0.
\]
\end{remark}
\paragraph{When ranking reversal occurs}
For the Brier Score, let
\[
L_{m,z}
=
\mathbb{E}\!\left[(c-z)^2\mid z\right]
\]
denote model $m$'s conditional loss for outcome $z$, and let
\[
R_m(t)
=
tL_{m,1}+(1-t)L_{m,0}
\]
be its loss under a common target accuracy $t$. Define
\[
\begin{aligned}
\Delta_{\mathrm{raw}}
&=R_A(a_A)-R_B(a_B),\\
\Delta_t
&=R_A(t)-R_B(t).
\end{aligned}
\]
Direct expansion gives
\[
\Delta_{\mathrm{raw}}
=
\Delta_t+C_t,
\]
where
\[
\begin{aligned}
C_t
={}&(a_A-t)
    (L_{A,1}-L_{A,0})\\
&-(a_B-t)
    (L_{B,1}-L_{B,0}).
\end{aligned}
\]
When $\Delta_t\neq 0$, a strict ranking reversal occurs exactly when
\[
\begin{aligned}
\operatorname{sign}(C_t)
&=-\operatorname{sign}(\Delta_t),\\
|C_t|&>|\Delta_t|.
\end{aligned}
\]
Thus, a reversal occurs when the contribution from the models'
different outcome compositions opposes and exceeds their loss gap at
the common target accuracy.

For ECE, the result above exhibits the same qualitative mechanism
under its stated condition: changing the mixture of correct and
incorrect outcomes can change ECE even when conditional confidence
behavior is held fixed, so its ranking may likewise change after
accuracy alignment.

\subsection{Empirical Analysis}
\label{app:empirical_analysis}

\paragraph{Empirical Setup.}
We verify this empirically in Figure~\ref{fig:head}. Across five 
benchmarks spanning knowledge-based (TriviaQA, FreshQA) and 
reasoning-based tasks (MMLU-Pro, GPQA, LiveBench), we compare the ECE gap against the accuracy gap between stronger models and weaker models, covering both scale-based pairs (Qwen2.5 7B vs.\ 72B, LLaMA~3.1 8B vs.\ 70B, GPT-OSS 20B vs.\ 120B) and thinking vs.\ non-thinking pairs (Qwen3-8B and GPT-OSS under 
high/low reasoning settings). The results reveal a strong negative correlation ($\rho=-0.670$, $n=40$, 95\% CI $[-0.81,-0.45]$, $p<10^{-5}$) between accuracy differences and raw ECE differences, consistent with our theoretical analysis: models with higher accuracy mechanically obtain better ECE scores even when their underlying confidence behavior is not correspondingly superior.

\section{Justification of Distribution-Aligned Calibration}
\label{app:ipw-justification}

The importance-weighting principle underlying Distribution-Aligned (DA) originates in survey sampling \citep{horvitz1952generalization} and is widely used in causal inference to correct for confounded comparisons between treatments and outcomes~\citep{rosenbaum1983central}. We briefly justify its application in our setting by making the analogy explicit.

\paragraph{Standard causal inference formulation.}
Consider a binary treatment $T \in \{0,1\}$, an outcome $Y$, and a confounder $X$ that influences both $T$ and $Y$. The na\"ive comparison $\mathbb{E}[Y \mid T=1] - \mathbb{E}[Y \mid T=0]$ is biased whenever the distribution of $X$ differs across treatment groups. DA removes this bias by reweighting each observation by the inverse of its treatment probability (propensity score), so that the weighted confounder distribution is balanced across groups.

\paragraph{Mapping to calibration comparison.}
In our setting, the role of the \emph{treatment} $T$ is played by model identity $m \in \{A, B\}$: comparing calibration across models is analogous to comparing outcomes across treatment arms. 
The \emph{confounder} $X$ corresponds to instance-level correctness $z_i^{(m)}$, which aggregates to model accuracy $a_m = \mathrm{Acc}(M_m)$; just as a covariate that differs across treatment groups biases the na\"ive comparison, differing accuracy compositions confound the raw calibration comparison. The \emph{outcome} $Y$ maps to the per-instance calibration loss, \eg $(c_i^{(m)} - z_i^{(m)})^2$ for the Brier Score. 
The \emph{propensity score} $P(T \mid X)$ corresponds to the probability of each correctness status in the source model, namely $a_A$ for correct predictions ($z=1$) and $1 - a_A$ for incorrect predictions ($z=0$). 
Finally, the \emph{target population} is the accuracy level of the reference model $a_B$, to which we wish to adjust the distribution of the results of the source model.
Under this mapping, the weights
\[
w_i^{(A)} =
\begin{cases}
a_B \,/\, a_A & \text{if } z_i^{(A)} = 1, \\[4pt]
(1 - a_B) \,/\, (1 - a_A) & \text{if } z_i^{(A)} = 0,
\end{cases}
\]
are standard importance weights that transform the correctness distribution of $M_A$ to match the target accuracy $a_B$. For each correctness class $z \in \{0,1\}$, the weight equals the target proportion divided by the source proportion, which is precisely the density ratio form used in classical importance weighting.

\paragraph{Assumptions.}
The validity of DA here requires two standard conditions:
\begin{inparaenum}[\it 1)] 
\item ~\emph{Positivity}: $0 < a_m < 1$, which holds trivially on any nontrivial benchmark; and
\item ~\emph{Conditional exchangeability}: conditioned on correctness status $z$, the confidence distribution is comparable across the accuracy shift. 
\end{inparaenum} 

\section{Instruction Prompt Examples}

\paragraph{Answer Generation Prompts}
\label{app:answer-prompts}
All models share the same answer-generation prompt family, with template selection determined by task type; reasoning behavior is controlled via model parameters rather than prompt wording. Specifically, we apply the multiple-choice template to GPQA Diamond and MMLU-Pro, the reasoning template to LiveBench, and the open-ended template to FreshQA and TriviaQA. The method-specific confidence-elicitation templates and their use are described in Appendix~\ref{app:confidence-details}.
\begin{table}[h]
\small
  \centering
    \begin{tabularx}{\linewidth}{X}
    \toprule
    \rowcolor[gray]{0.95}\multicolumn{1}{c}{\textbf{I: MULTIPLE-CHOICE TEMPLATE}} \\
    \midrule
    Task: Solve the following multiple-choice problem. \\
    The answer should be a single letter (A, B, C, D, etc.). \\
    Provide your final answer in the following format: \\
    \textbackslash boxed\{A\} \\[4pt]
    Problem: \\
    \{problem\} \\

    \midrule
    \rowcolor[gray]{0.95}\multicolumn{1}{c}{\textbf{II: REASONING TEMPLATE}} \\
    \midrule
    Task: Solve the following reasoning problem. \\
    Provide your final answer in the following XML format: \\
    \texttt{<answer>final answer here</answer>} \\[4pt]
    Problem: \\
    \{problem\} \\

    \midrule
    \rowcolor[gray]{0.95}\multicolumn{1}{c}{\textbf{III: OPEN-ENDED TEMPLATE}} \\
    \midrule
    Task: Answer the following factual question directly. \\
    This is NOT a multiple-choice question. Provide a direct factual answer. \\
    Provide your final answer in the following XML format: \\
    \texttt{<answer>final answer here</answer>} \\[4pt]
    Question: \\
    \{problem\} \\
    \bottomrule
    \end{tabularx}
  \caption{Three prompt templates used for generation across different task types.}
  \label{tab:query_gen}
\end{table}

\section{Implementation Details}
\label{app:exp-details}
\subsection{Model and Inference Details}
\label{app:model-details}
We evaluate 14 locally deployed models spanning four families: Qwen2.5, LLaMA~3.1, Qwen3, and GPT-OSS. Table~\ref{tab:freshqa} lists all models. Qwen3 and GPT-OSS each expose a reasoning-intensity control: for Qwen3, the \texttt{enable\_thinking} flag in the chat template switches between thinking (\texttt{True}) and non-thinking (\texttt{False}) modes; for GPT-OSS, the \texttt{reasoning\_effort} parameter selects between \texttt{high} and \texttt{low} settings. This yields 14 model configurations in total.
All models are served locally via vLLM~\citep{kwon2023efficientmemorymanagementlarge} in \texttt{bfloat16} precision. 
Both answer-generation and confidence-estimation stages share the same decoding configuration: temperature $= 1.0$, top-$p = 1.0$, and top-$k = -1$, with \texttt{max\_tokens} set to \texttt{max\_model\_len}. Context lengths are 32,768 tokens for Qwen2.5 and Qwen3, 40,960 for LLaMA~3.1, and 131,072 for GPT-OSS.

\subsection{Dataset Details and Evaluation Protocol}
\label{app:data-details}
We evaluate on five datasets: FreshQA, TriviaQA, GPQA Diamond, LiveBench, and MMLU-Pro.
\begin{table}[H]
\setlength{\tabcolsep}{9pt}
    \centering
    \small
    \begin{tabular}{l l r}
       \toprule
        \rowcolor[gray]{0.95}Dataset & Type & Size \\
        \midrule
        FreshQA & Knowledge & 600 \\
        TriviaQA & Knowledge & 1,500 \\
        GPQA Diamond & Reasoning & 198 \\
        LiveBench (reasoning subset) & Reasoning & 200 \\
        MMLU-Pro & Reasoning & 2,841 \\
        \bottomrule
    \end{tabular}
    \caption{Summary of the evaluation datasets used in our experiments.}
    \label{tab:dataset_summary}
\end{table}

\paragraph{FreshQA.}
We use the \texttt{2025-10-07} version of FreshQA and evaluate on 600 questions. This dataset is treated as an open-ended factual QA benchmark. Correctness is determined by LLM-based grading.

\paragraph{TriviaQA.}
We sample 1,500 examples from the \texttt{unfiltered-nocontent} validation split of TriviaQA. This dataset is treated as open-ended factual QA. We use alias-aware exact match, where a prediction is counted as correct if it matches any acceptable alias.

\paragraph{GPQA Diamond.}
We use the standard GPQA Diamond split, containing 198 examples. This dataset is treated as an expert-level multiple-choice benchmark. Correctness is determined by exact matching between the extracted option letter and the gold answer.

\paragraph{LiveBench.}
We evaluate on 200 examples from the reasoning subset of LiveBench. This benchmark is treated as a reasoning / fill-in task. Correctness is determined by rule-based matching.

\paragraph{MMLU-Pro.}
We evaluate on 2,841 examples from the computer science, chemistry, and physics domains of MMLU-Pro. This dataset is treated as a multiple-choice benchmark. Correctness is determined by exact matching between the extracted option letter and the gold answer.

\subsection{Confidence Elicitation Details}
\label{app:confidence-details}
We consider three confidence estimation methods: verbalized confidence, P(True), and self-consistency. The answer-generation prompts used to produce the responses evaluated here are listed in Appendix~\ref{app:answer-prompts}.
\paragraph{Verbalized Confidence.}
The model is presented with the original question alongside its own full response and asked to assign a confidence score from 0 to 100. We use the full \texttt{model\_response} rather than the extracted short answer as the answer source. We extract the numeric value inside \texttt{\textbackslash boxed\{\}} and normalize it to $[0,1]$ by dividing by 100. 
\begin{table}[h]
\small
  \centering
    \begin{tabularx}{\linewidth}{X}
    \toprule
    \rowcolor[gray]{0.95}\multicolumn{1}{c}{\textbf{VANILLA-VERB TEMPLATE}} \\
    \midrule
    \{Question\} \\[4pt]
    Proposed Answer: \{Answer\} \\
    How confident are you that the proposed answer is correct? \\[4pt]
    The confidence score should be a number from 0 (completely unsure) to 100 (absolutely certain). \\
    Your response MUST strictly adhere to this format: \\
    \#\#\# Confidence: \textbackslash boxed\{Your confidence score from 0-100.\} \\[4pt]
    Now it is your turn to give a response: \\
    \bottomrule
    \end{tabularx}
  \caption{Vanilla verbalized confidence estimation template.}
  \label{tab:vanilla_verb}
\end{table}

\paragraph{P(True).}
The model is presented with the original question alongside its own full response and asked to judge whether the proposed answer is true or false. Confidence is defined as the probability assigned to token \texttt{A} (``True''): we request \texttt{logprobs=20}, extract the log-probability of token \texttt{A}, and convert it to a linear probability. 
\begin{table}[h]
\small
  \centering
    \begin{tabularx}{\linewidth}{X}
    \toprule
    \rowcolor[gray]{0.95}\multicolumn{1}{c}{\textbf{BINARY VERIFICATION TEMPLATE}} \\
    \midrule
    \{Question\} \\[4pt]
    Proposed Answer: \{Answer\} \\
    Is the proposed answer: \\[4pt]
    \hspace{1em} A. True \\
    \hspace{1em} B. False \\[4pt]
    Output format: A or B only \\
    (single uppercase letter; no spaces, punctuation, or explanation): \\
    \bottomrule
    \end{tabularx}
  \caption{Binary verification template for answer correctness judgment.}
  \label{tab:binary_verif}
\end{table}

\paragraph{Self-Consistency.}
For self-consistency, we generate 10 independent responses for each question using the same answer-generation prompt and decoding setup. The confidence of an example is defined as the fraction of runs whose extracted answer matches the extracted answer from the first run:
\[
c_i=\frac{\sum_{k=1}^{10}\mathbf{1}[\hat{a}_k=\hat{a}_1]}{10}.
\]
Here, $\hat{a}_1$ denotes the extracted answer from the first run and $\hat{a}_k$ denotes the extracted answer from run $k$. If an example is marked as \texttt{is\_attempted=False} in the first run, we skip it. If the answer is missing in later runs, the denominator remains 10.

\clearpage
\section{Additional Analyses}
\label{app:additional-analyses}
\subsection{Reversal Rates on Knowledge and Reasoning Benchmarks}

We further investigate whether ranking reversals are more prevalent on reasoning than knowledge benchmarks. Table~\ref{tab:reasoning_reversal_by_method_view} shows this pattern holds consistently under the DA and IA views, but varies under CA. Specifically, verbalized confidence yields higher reasoning reversal rates across all alignment views, whereas SC and $P(\mathrm{true})$ diverge under CA.
Additionally, Table~\ref{tab:reasoning_reversal_by_pair_type} indicates that, aggregated across all methods and views, this reasoning-knowledge gap is driven primarily by mode/effort differences between evaluated models, rather than size differences.

\begin{table}[H]
\setlength{\tabcolsep}{6.2pt}
    \centering
    \small
    \begin{tabular}{lcccc}
        \toprule
        \rowcolor[gray]{0.95}Method & View & KRR & RRR & $\Delta$ (RRR$-$KRR) \\
        \midrule
        \multirow{3}{*}{SC}
                      & DA &  \phantom{0}6.2 & 41.7 & $+$35.4 \\
                      & IA &  18.8            & 41.7 & $+$22.9 \\
                      & CA &  56.2            & 41.7 & $-$14.6 \\
        \midrule
        \multirow{3}{*}{Verbalized}
                      & DA &  43.8 & 58.3 & $+$14.6 \\
                      & IA &  50.0 & 62.5 & $+$12.5 \\
                      & CA &  18.8 & 33.3 & $+$14.6 \\
        \midrule
        \multirow{3}{*}{$P(\text{true})$}
                      & DA &  37.5 & 54.2 & $+$16.7 \\
                      & IA &  25.0 & 25.0 & $+$0.0 \\
                      & CA &  50.0 & 20.8 & $-$29.2 \\
        \bottomrule
    \end{tabular}
    \caption{Knowledge reversal rate (KRR) and reasoning reversal rate (RRR) broken down by confidence method and alignment view. For $\Delta (\text{RRR} - \text{KRR})$, positive values indicate reversal is more frequent on reasoning datasets.}
    \label{tab:reasoning_reversal_by_method_view}
\end{table}

\begin{table}[H]
      \centering
      \small
      \setlength{\tabcolsep}{6pt}
      \renewcommand{\arraystretch}{1.12}
      \begin{tabular}{lcccc}
          \toprule
          \rowcolor[gray]{0.95}
          Pair type & KRR & RRR & $\Delta$ (RRR$-$KRR) & \#Pairs \\
          \midrule
          mode/effort & 33.3 & 45.2 & $+$11.9 & 5 \\
          size        & 35.2 & 37.0 & $+$1.9  & 3 \\
          \bottomrule
      \end{tabular}
      \caption{Average reversal gap by pair type after merging all confidence methods and aligned
  views. Positive $\Delta$ values indicate that reversal is more frequent on reasoning datasets.}
      \label{tab:reasoning_reversal_by_pair_type}
  \end{table}
  
\subsection{Reversal Rates Across Accuracy-Gap Bins}
\label{app:accgap-analysis}

\begin{figure*}[ht!]
    \centering
    \includegraphics[width=0.95\linewidth]{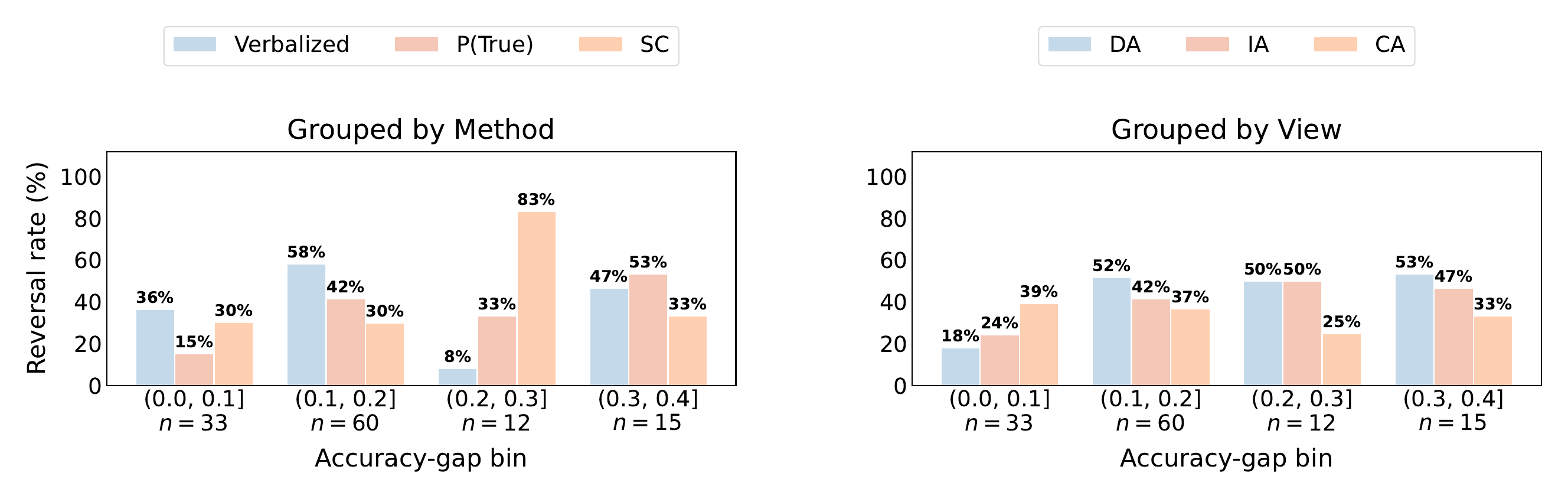}
    \caption{Reversal rate by accuracy-gap bin. The combined figure presents the same pattern from two complementary perspectives: grouped by confidence method and grouped by aligned view. In both views, reversal increases once the accuracy gap exceeds 0.1, but does not exhibit a clear monotonic increase across higher bins.}
    \label{fig:reversal_by_accgap_bin_combined}
\end{figure*}

We further examine how reversal rates change as the accuracy gap between the
two compared models increases. Figure~\ref{fig:reversal_by_accgap_bin_combined}
shows that the pattern is better described as threshold-like than monotonic.

More concretely, when the accuracy gap is below 0.1, reversal rates are
generally much lower across both the method-grouped and view-grouped panels.
Once the gap exceeds 0.1, reversal becomes noticeably more common and remains
at a substantially higher level. This indicates that reversal is much more
likely once the two models are separated by a non-trivial gap in task
accuracy.

At the same time, the figure does not support the stronger claim that
reversal keeps increasing as the accuracy gap becomes larger and larger. After
crossing the 0.1 threshold, the higher-gap bins fluctuate rather than rising
monotonically. In other words, the key empirical dividing line is whether the
accuracy gap is above 0.1, not which larger bin it falls into.

\subsection{Both-Wrong and Both-Right Confidence Distributions}
\label{app:confdist-analysis}
Figure~\ref{fig:both_wrong_vs_both_right_hist} reveals several consistent
patterns across methods and datasets. First, all three confidence methods show
substantial separation between both-wrong and both-right examples: both-right
mass is generally concentrated toward the high-confidence end, while both-wrong
mass is shifted to the left. In this sense, all three methods preserve a
meaningful distinction between clearly correct and clearly incorrect shared
outcomes.

At the same time, the three methods differ in how sharply and how reliably
they make this distinction. Verbalized confidence and $P(\mathrm{true})$ are
more polarized, with much of their mass pushed toward the two extremes. This
produces visually sharp distributions, but it also reveals substantial
overconfidence: both methods still assign considerable probability mass near
the high-confidence end even on both-wrong examples. Among the three,
verbalized confidence appears to show the weakest separation overall, with the
red and blue mean lines often relatively close to each other, indicating that
incorrect shared outcomes still receive fairly high confidence. $P(\mathrm{true})$
shows a similar but somewhat less severe pattern.

SC behaves differently. Its distributions are broader and less concentrated at
the extremes, so the distinction between both-wrong and both-right is less
``all-or-nothing'' but often more conservative. In particular, SC places less
mass on extremely high confidence for both-wrong examples, making its
overconfidence visibly weaker than that of verbalized confidence and
$P(\mathrm{true})$. This does not mean that SC is free of overconfidence---some
datasets still show substantial overlap---but the figure suggests that its
errors are less driven by extreme high-confidence mistakes.

\begin{figure*}[ht!]
    \centering
    \includegraphics[width=1.0\linewidth]{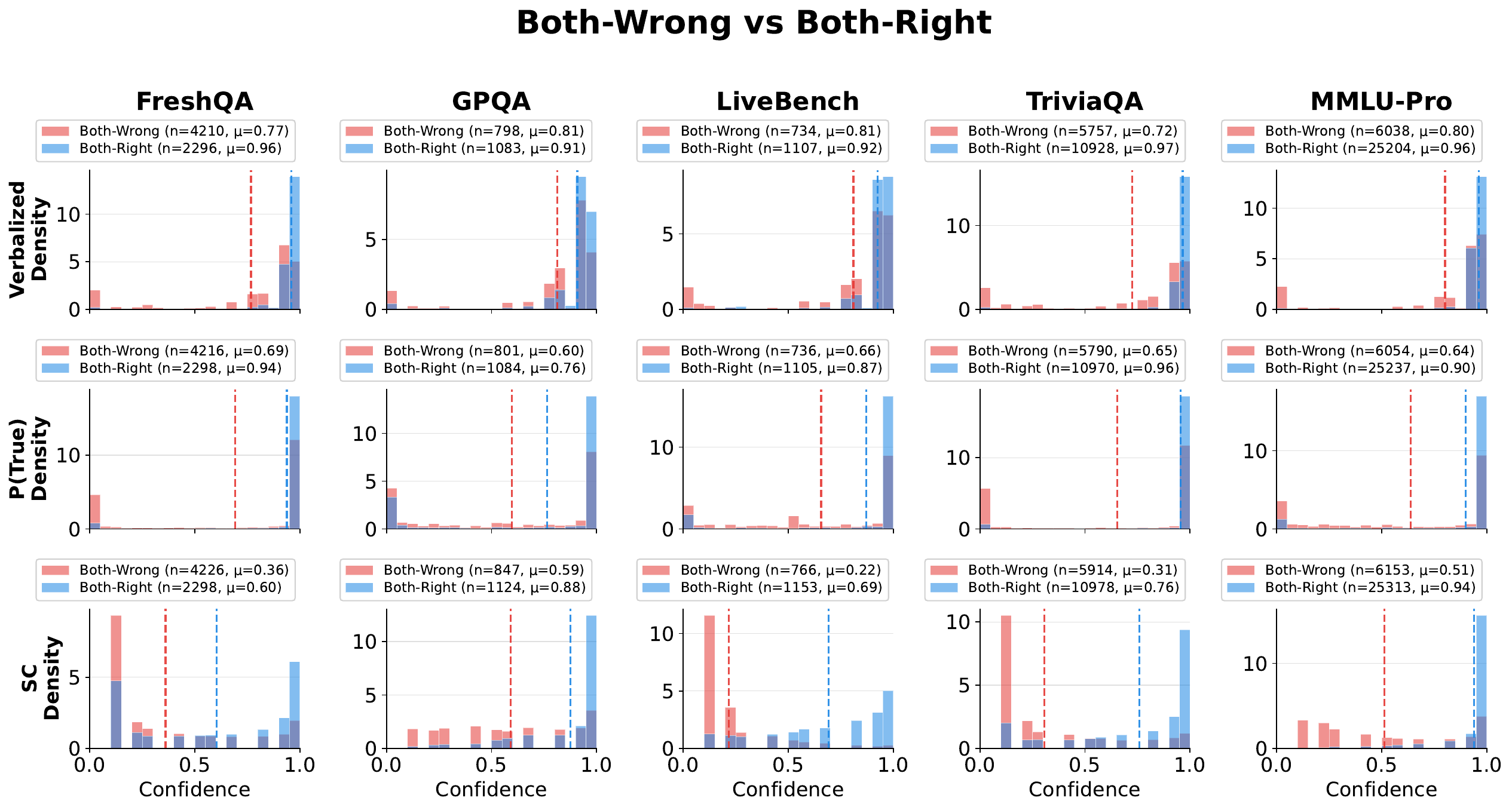}
    \caption{Histogram comparison of confidence distributions on both-wrong and both-right examples across three confidence elicitation methods and five datasets. The figure is intended as a descriptive view of how well the two shared-outcome regions are separated, and where overconfidence on both-wrong examples may appear.}
    \label{fig:both_wrong_vs_both_right_hist}
\end{figure*}

\subsection{Instance-Aligned Retention}
\label{app:ia-retention}

Since IA compares models only on shared-outcome examples, one concern is that
the retained subset might become too small to support stable conclusions. We
therefore report the retained sample size for each model-pair--dataset cell.
Retention is defined as $|D_{\mathrm{IA}}|/|D|$, where
$D_{\mathrm{IA}}=D_{\mathrm{both\mbox{-}right}}\cup D_{\mathrm{both\mbox{-}wrong}}$.
As shown in Table~\ref{tab:ia-retention}, IA retains substantial coverage
across all 40 model-pair--dataset cells: the median retention is 78.0\%,
the IQR is [66.3\%, 82.2\%], and the minimum is 55.2\%. No cell falls below
50\%, indicating that the IA results are not driven by narrow retained subsets. Together with the aggregate correlation in Figure~\ref{fig:head}, this suggests that the observed reversals reflect a systematic accuracy-related confound rather than sampling noise.

\begin{table*}[t]
    \centering
    \small
    \setlength{\tabcolsep}{6.7pt}
    \renewcommand{\arraystretch}{1.2}
    \begin{tabular}{llccccc}
        \toprule
        \rowcolor[gray]{0.95}
        Pair & Type & FreshQA & GPQA & LiveBench & TriviaQA & MMLU-Pro \\
        \midrule
        Qwen3-8B T/N        & Mode   & 474 (79.0) & 135 (68.2) & 113 (56.5) & 1181 (78.7) & 2263 (79.7) \\
        Qwen3-14B T/N       & Mode   & 474 (79.0) & 130 (65.7) & 120 (60.0) & 1229 (81.9) & 2360 (83.1) \\
        Qwen3-32B T/N       & Mode   & 469 (78.2) & 142 (71.7) & 115 (57.5) & 1208 (80.5) & 2369 (83.4) \\
        GPT-OSS-120B H/L    & Effort & 524 (87.3) & 159 (80.3) & 144 (72.0) & 1356 (90.4) & 2574 (90.6) \\
        GPT-OSS-20B H/L     & Effort & 498 (83.0) & 137 (69.2) & 128 (64.0) & 1215 (81.0) & 2377 (83.7) \\
        Qwen2.5 72B/7B      & Size   & 446 (74.3) & 123 (62.1) & 133 (66.5) & 1003 (66.9) & 1797 (63.3) \\
        LLaMA3.1 70B/8B     & Size   & 331 (55.2) & 124 (62.6) & 154 (77.0) & 1119 (74.6) & 1706 (60.0) \\
        GPT-OSS 120B/20B H  & Size   & 480 (80.0) & 168 (84.8) & 179 (89.5) & 1168 (77.9) & 2654 (93.4) \\
        \bottomrule
    \end{tabular}
    \caption{Instance-Aligned retained sample sizes across model-pair--dataset cells. Each cell reports $|D_{\mathrm{IA}}|$ with the retention percentage in parentheses. Dataset sizes are 600 for FreshQA, 198 for GPQA Diamond, 200 for LiveBench Reasoning, 1,500 for TriviaQA, and 2,841 for MMLU-Pro. T/N denotes thinking versus non-thinking, and H/L denotes high versus low reasoning effort.}
    \label{tab:ia-retention}
\end{table*}
\subsection{DA Effective Sample Size}
\label{app:da-ess}

Although DA uses all observed instances, concentrated weights can reduce
its effective precision. Using the Kish effective sample size,
$\mathrm{ESS}=(\sum_i w_i)^2/\sum_i w_i^2$, we obtain a median
$\mathrm{ESS}/N$ of $0.897$ across 120 method--pair--dataset cells.
Overall, $115/120$ cells have $\mathrm{ESS}/N\geq0.5$; the remaining
five fall below this diagnostic threshold, with a minimum of $0.144$. Thus, DA retains high effective precision in most comparisons, with a few cases substantially affected by weight concentration.
\subsection{Bootstrap Analysis of Ranking Reversals}
\label{app:bootstrap-reversals}

We assess sampling uncertainty using 2,000 paired bootstrap samples of
the evaluated questions. For each combination of confidence method,
model pair, and dataset, questions are resampled jointly for the two
models. We then recompute the Raw comparison and all three aligned
views from each resample. IA rebuilds the subset on which the models
share the same outcome, DA re-estimates the outcome proportions and
importance weights, and CA keeps each question together with its
candidate pool. The 2.5th and 97.5th percentiles form the 95\%
confidence intervals for the Raw and aligned gaps.

For each reversal in the original estimates, we also record the
proportion of bootstrap samples in which the Raw and aligned gaps
remain opposite in sign. Among the 720 comparisons, 265 show a
reversal in the original estimates. Their median retention probability
is 91.9\%, and 122 of the 265 remain reversals in at least 95\% of the
bootstrap samples.

\begin{table}[H]
    \centering
    \small
    \setlength{\tabcolsep}{6.2pt}
    \renewcommand{\arraystretch}{1.12}
    \begin{tabular}{lcc}
        \toprule
        \rowcolor[gray]{0.95}
        Metric & Point estimate & Strict criterion \\
        \midrule
        ECE         & 86 (71.7) & 31 (25.8) \\
        Brier Score & 84 (70.0) & 45 (37.5) \\
        \bottomrule
    \end{tabular}
    \caption{Bootstrap results aggregated over 120 cells. Each cell
    corresponds to one combination of confidence method, model pair,
    and dataset. Each entry reports the number of cells, with the
    percentage in parentheses. The point estimate column counts cells
    containing a reversal in at least one of CA, IA, or DA. The strict
    criterion additionally requires, for at least one view, that the
    Raw and aligned 95\% confidence intervals both exclude zero and
    have opposite signs.}
    \label{tab:bootstrap-reversal-summary}
\end{table}
\subsection{Extension to Graded Correctness}
\label{app:graded-correctness}

The binary formulation extends naturally to tasks with partial credit.
Let $z_i\in[0,1]$ denote the graded correctness of prediction $i$, and
interpret the confidence score $c_i\in[0,1]$ as the model's predicted
degree of correctness. The functional forms of the Brier Score and ECE
remain unchanged:
\[
\begin{aligned}
\mathrm{BS}
&= \frac{1}{N}\sum_{i=1}^{N}(c_i-z_i)^2,\\
\mathrm{ECE}
&= \sum_{b=1}^{B}\frac{|S_b|}{N}
\left|\bar z_b-\bar c_b\right|,
\end{aligned}
\]
where $S_b=\{i:c_i\in I_b\}$, and $\bar z_b$ and $\bar c_b$ denote the
mean graded correctness and mean confidence in bin $b$, respectively.
The binary setting is recovered when $z_i\in\{0,1\}$.

IA, DA, and CA extend naturally by replacing binary outcomes with graded
outcomes in their respective constructions. For graded IA, let $h$ be a
prespecified mapping from graded scores to matching labels, and define
\[
D_{\mathrm{IA}}^{\mathrm{graded}}(h)
=
\left\{
i:
h\!\left(z_i^{(A)}\right)
=
h\!\left(z_i^{(B)}\right)
\right\}.
\]
Setting $h(z)=z$ gives exact score matching. If $h$ instead assigns scores
to prespecified bins, it provides coarser alignment that can increase
sample retention. The choice of $h$ controls the trade-off between
alignment granularity and sample retention.

For graded DA, let $g:[0,1]\rightarrow\{1,\ldots,K\}$ map each graded
outcome to an outcome stratum, and let $p_{m,k}$ denote the empirical
proportion of model $m$'s outcomes falling in stratum $k$. Given a common
target distribution $q=(q_1,\ldots,q_K)$ over these strata, the importance
weight assigned to example $i$ for model $m$ is
\[
w_i^{(m)}
=
\frac{q_k}{p_{m,k}},
\qquad
k=g\!\left(z_i^{(m)}\right),
\]
assuming $p_{m,k}>0$ whenever $q_k>0$. These weights align each model's
graded-outcome distribution with the common target $q$ before computing
the calibration metrics.

For graded CA, the shared candidate pool and aggregation procedure remain
unchanged. The existing candidate correctness label $y_i(s)$ is extended
from $\{0,1\}$ to $[0,1]$, and $e_i^{(m)}(s)$ represents model $m$'s
predicted degree of correctness for candidate $s$. The candidate-level
Brier Score and ECE are then computed as in binary CA using these graded
quantities.

\paragraph{Experimental setup.}
To evaluate the graded extension while keeping inference cost manageable,
we compare Qwen2.5-7B and Qwen2.5-72B using verbalized confidence. We use
500 examples from each dataset and normalize the elicited 0--100 confidence
scores to $[0,1]$.

\paragraph{Datasets and graded correctness.}
We evaluate on two datasets with different forms of partial credit:
IF-multi
\citep{pyatkin2025generalizingverifiableinstructionfollowing} and APPS
\citep{hendrycks2021measuring}. IF-multi consists of instruction-following
prompts with multiple automatically verifiable constraints. We use prompts
containing three to five constraints and define graded correctness as the
fraction of constraints satisfied by the response, as determined by the
accompanying rule-based verifiers. APPS is a code-generation benchmark
evaluated through test cases. We sample problems from its official test
split across all three difficulty levels and define graded correctness as
the fraction of official test cases passed by the generated program. In
both datasets, verbalized confidence is elicited for the same quantity used
to define graded correctness: the expected fraction of satisfied constraints
or passed test cases.

\paragraph{Alignment implementation.}
We evaluate graded IA and DA, whose alignments are constructed directly
from graded correctness outcomes and do not require the additional
cross-model candidate evaluation used by CA. For IA, we use exact score
matching on both datasets, corresponding to $h(z)=z$ in the definition
above. For DA, the empirical graded-outcome distribution of Qwen2.5-7B
serves as the target. Because IF-multi has a small discrete set of scores
determined by three to five constraints, we align its exact graded-score
levels. APPS has a denser score distribution because the number of test
cases varies across problems, so we partition $[0,1]$ into 10 equal-width
outcome strata. Every stratum with positive target mass has positive
empirical support under both models, satisfying the support condition
required by graded DA.

\paragraph{Results.}
Qwen2.5-72B achieves higher mean graded correctness than Qwen2.5-7B
on both IF-multi (55.1\% vs.\ 44.1\%) and APPS (53.2\% vs.\ 31.7\%).
Table~\ref{tab:graded-results} reports the corresponding calibration gaps,
defined as the score of Qwen2.5-72B minus that of Qwen2.5-7B. Under Raw
evaluation, both the Brier Score and ECE favor Qwen2.5-72B on both
datasets. On both datasets, IA and DA reverse the Raw ordering for both
metrics, favoring Qwen2.5-7B. These results show that the reversal
phenomenon is not specific to binary correctness and can also arise when
outcomes receive partial credit.

\begin{table}[H]
\setlength{\tabcolsep}{8pt}
\centering
\small
\begin{tabular}{llccc}
\toprule
\rowcolor[gray]{0.95}
Dataset & Metric & Raw $\Delta$ & IA $\Delta$ & DA $\Delta$ \\
\midrule
\multirow{2}{*}{IF-multi}
    & Brier Score & $-6.0$ & $+1.4$ & $+1.5$ \\
    & ECE         & $-5.8$ & $+3.2$ & $+3.1$ \\
\midrule
\multirow{2}{*}{APPS}
    & Brier Score & $-6.3$ & $+6.7$ & $+9.8$ \\
    & ECE         & $-8.0$ & $+5.9$ & $+8.5$ \\
\bottomrule
\end{tabular}
\caption{Calibration gaps under graded correctness. $\Delta$ denotes
Qwen2.5-72B minus Qwen2.5-7B, and all values are multiplied by 100.
Because lower scores indicate better calibration, negative values favor
Qwen2.5-72B and positive values favor Qwen2.5-7B.}
\label{tab:graded-results}
\end{table}
\section{Calibration Metrics}
\label{app:Calibration Metrics}
Tables~\ref{tab:freshqa}--\ref{tab:mmlu} present the full calibration metrics underlying the trends visualized in Figures~\ref{fig:main_results_01} and~\ref{fig:main_results_02}. For each of the five evaluation datasets (FreshQA, TriviaQA, GPQA Diamond, LiveBench Reasoning, and MMLU-Pro), we report accuracy and two calibration measures (\ie Brier Score and ECE), across three confidence elicitation methods (Verbalize, P(True), and Self-Consistency) and four evaluation views: \textbf{Raw}, \textbf{CA}, \textbf{DA} and \textbf{IA}. All accuracy values are computed from a single run with temperature $=1$. Each model pair is followed by a $\Delta$ row showing the difference (Model$_1$ $-$ Model$_2$), where \textcolor{darkgreen}{green} indicates the better-performing direction and \textcolor{darkred}{red} the worse.

\definecolor{darkgreen}{RGB}{0,150,0}
\definecolor{darkred}{RGB}{220,40,40}
\definecolor{lightblue}{RGB}{208,227,251}

\begin{table*}[t!]
\setlength{\tabcolsep}{15.2pt}
\centering

\scalebox{0.72}{
\begin{tabular}{l c cc cc cc cc}
\toprule
\textbf{Model} & \textbf{Acc} & \multicolumn{2}{c}{\textbf{Raw}} & \multicolumn{2}{c}{\textbf{CA}} & \multicolumn{2}{c}{\textbf{DA}} & \multicolumn{2}{c}{\textbf{IA}} \\
\cmidrule(lr){3-4}\cmidrule(lr){5-6}\cmidrule(lr){7-8}\cmidrule(lr){9-10}
 & & Brier & ECE & Brier & ECE & Brier & ECE & Brier & ECE \\
\midrule
\rowcolor[gray]{0.95}\multicolumn{10}{c}{\textit{Verbalize}} \\
Qwen2.5{-}7B & 21.4 & 53.3 & 53.3 & 49.5 & 45.9 & 52.2 & 56.8 & 49.6 & 54.6 \\
Qwen2.5{-}72B & 38.7 & 44.6 & 42.3 & 41.2 & 38.7 & 54.5 & 58.4 & 52.1 & 56.1 \\
$\Delta$  &  \cellcolor{lightblue!50}\textcolor{darkred}{-17.3} &  \cellcolor{lightblue!50}\textcolor{darkred}{+8.7}  &  \cellcolor{lightblue!50}\textcolor{darkred}{+11.0}  &  \cellcolor{lightblue!50}\textcolor{darkred}{+8.3}  &  \cellcolor{lightblue!50}\textcolor{darkred}{+7.2}  &  \cellcolor{lightblue!50}\textcolor{darkgreen}{-2.3}  &  \cellcolor{lightblue!50}\textcolor{darkgreen}{-1.6}  &  \cellcolor{lightblue!50}\textcolor{darkgreen}{-2.5}  &  \cellcolor{lightblue!50}\textcolor{darkgreen}{-1.5} \\
\cdashline{1-10}
Llama3.1{-}8B & 32.1 & 48.5 & 50.7 & 44.9 & 44.2 & 45.1 & 47.2 & 43.7 & 46.6 \\
Llama3.1{-}70B & 49.3 & 32.1 & 32.9 & 32.6 & 31.6 & 50.5 & 51.9 & 49.0 & 50.5 \\
$\Delta$  &  \cellcolor{lightblue!50}\textcolor{darkred}{-17.2} &  \cellcolor{lightblue!50}\textcolor{darkred}{+16.3}  &  \cellcolor{lightblue!50}\textcolor{darkred}{+17.9}  &  \cellcolor{lightblue!50}\textcolor{darkred}{+12.3}  &  \cellcolor{lightblue!50}\textcolor{darkred}{+12.5}  &  \cellcolor{lightblue!50}\textcolor{darkgreen}{-5.4}  &  \cellcolor{lightblue!50}\textcolor{darkgreen}{-4.7}  &  \cellcolor{lightblue!50}\textcolor{darkgreen}{-5.3}  &  \cellcolor{lightblue!50}\textcolor{darkgreen}{-3.9} \\
\cdashline{1-10}
Qwen3{-}8B{-}nothink & 26.8 & 57.0 & 60.7 & 53.8 & 56.7 & 57.1 & 61.0 & 57.6 & 61.7 \\
Qwen3{-}8B{-}think & 36.1 & 42.4 & 46.4 & 35.8 & 42.3 & 61.3 & 65.5 & 60.8 & 65.3 \\
$\Delta$  &  \cellcolor{lightblue!50}\textcolor{darkred}{-9.3} &  \cellcolor{lightblue!50}\textcolor{darkred}{+14.6}  &  \cellcolor{lightblue!50}\textcolor{darkred}{+14.3}  &  \cellcolor{lightblue!50}\textcolor{darkred}{+18.0}  &  \cellcolor{lightblue!50}\textcolor{darkred}{+14.4}  &  \cellcolor{lightblue!50}\textcolor{darkgreen}{-4.1}  &  \cellcolor{lightblue!50}\textcolor{darkgreen}{-4.5}  &  \cellcolor{lightblue!50}\textcolor{darkgreen}{-3.2}  &  \cellcolor{lightblue!50}\textcolor{darkgreen}{-3.6} \\
\cdashline{1-10}
Qwen3{-}32B{-}nothink & 31.3 & 36.5 & 37.6 & 43.3 & 44.9 & 36.5 & 37.6 & 33.6 & 34.9 \\
Qwen3{-}32B{-}think & 41.6 & 50.7 & 52.8 & 45.0 & 47.6 & 59.5 & 62.4 & 57.4 & 60.5 \\
$\Delta$  &  \cellcolor{lightblue!50}\textcolor{darkred}{-10.3} &  \cellcolor{lightblue!50}\textcolor{darkgreen}{-14.2}  &  \cellcolor{lightblue!50}\textcolor{darkgreen}{-15.2}  &  \cellcolor{lightblue!50}\textcolor{darkgreen}{-1.7}  &  \cellcolor{lightblue!50}\textcolor{darkgreen}{-2.7}  &  \cellcolor{lightblue!50}\textcolor{darkgreen}{-23.0}  &  \cellcolor{lightblue!50}\textcolor{darkgreen}{-24.8}  &  \cellcolor{lightblue!50}\textcolor{darkgreen}{-23.8}  &  \cellcolor{lightblue!50}\textcolor{darkgreen}{-25.6} \\
\cdashline{1-10}
GPT{-}OSS{-}20B{-}low & 35.8 & 39.5 & 42.0 & 31.2 & 32.4 & 39.5 & 42.0 & 39.3 & 43.1 \\
GPT{-}OSS{-}20B{-}high & 40.1 & 44.2 & 47.8 & 26.5 & 27.5 & 47.6 & 51.8 & 46.6 & 51.1 \\
$\Delta$  &  \cellcolor{lightblue!50}\textcolor{darkred}{-4.3} &  \cellcolor{lightblue!50}\textcolor{darkgreen}{-4.7}  &  \cellcolor{lightblue!50}\textcolor{darkgreen}{-5.7}  &  \cellcolor{lightblue!50}\textcolor{darkred}{+4.7}  &  \cellcolor{lightblue!50}\textcolor{darkred}{+4.9}  &  \cellcolor{lightblue!50}\textcolor{darkgreen}{-8.1}  &  \cellcolor{lightblue!50}\textcolor{darkgreen}{-9.7}  &  \cellcolor{lightblue!50}\textcolor{darkgreen}{-7.3}  &  \cellcolor{lightblue!50}\textcolor{darkgreen}{-8.0} \\
\cdashline{1-10}
GPT{-}OSS{-}120B{-}low & 46.3 & 33.1 & 35.0 & 32.7 & 33.4 & 33.1 & 35.0 & 32.4 & 35.6 \\
GPT{-}OSS{-}120B{-}high & 47.8 & 33.6 & 37.2 & 32.8 & 34.4 & 34.5 & 38.4 & 33.8 & 38.3 \\
$\Delta$  &  \cellcolor{lightblue!50}\textcolor{darkred}{-1.5} &  \cellcolor{lightblue!50}\textcolor{darkgreen}{-0.5}  &  \cellcolor{lightblue!50}\textcolor{darkgreen}{-2.2}  &  \cellcolor{lightblue!50}\textcolor{darkgreen}{-0.1}  &  \cellcolor{lightblue!50}\textcolor{darkgreen}{-1.0}  &  \cellcolor{lightblue!50}\textcolor{darkgreen}{-1.4}  &  \cellcolor{lightblue!50}\textcolor{darkgreen}{-3.4}  &  \cellcolor{lightblue!50}\textcolor{darkgreen}{-1.4}  &  \cellcolor{lightblue!50}\textcolor{darkgreen}{-2.6} \\
\cdashline{1-10}
GPT{-}OSS{-}20B{-}high & 40.1 & 44.2 & 47.8 & 26.5 & 27.5 & 47.6 & 51.8 & 46.6 & 51.1 \\
GPT{-}OSS{-}120B{-}high & 47.8 & 33.6 & 37.2 & 32.8 & 34.4 & 34.5 & 38.4 & 33.8 & 38.3 \\
$\Delta$  &  \cellcolor{lightblue!50}\textcolor{darkred}{-7.7} &  \cellcolor{lightblue!50}\textcolor{darkred}{+10.6}  &  \cellcolor{lightblue!50}\textcolor{darkred}{+10.5}  &  \cellcolor{lightblue!50}\textcolor{darkgreen}{-6.2}  &  \cellcolor{lightblue!50}\textcolor{darkgreen}{-6.9}  &  \cellcolor{lightblue!50}\textcolor{darkred}{+13.1}  &  \cellcolor{lightblue!50}\textcolor{darkred}{+13.3}  &  \cellcolor{lightblue!50}\textcolor{darkred}{+12.8}  &  \cellcolor{lightblue!50}\textcolor{darkred}{+12.9} \\
\midrule
\rowcolor[gray]{0.95}\multicolumn{10}{c}{\textit{P(True)}} \\
Qwen2.5{-}7B & 21.4 & 39.7 & 40.5 & 39.0 & 39.1 & 39.7 & 40.5 & 39.3 & 40.0 \\
Qwen2.5{-}72B & 38.7 & 31.4 & 31.6 & 26.0 & 25.6 & 38.7 & 41.4 & 36.8 & 39.8 \\
$\Delta$  &  \cellcolor{lightblue!50}\textcolor{darkred}{-17.3} &  \cellcolor{lightblue!50}\textcolor{darkred}{+8.3}  &  \cellcolor{lightblue!50}\textcolor{darkred}{+8.8}  &  \cellcolor{lightblue!50}\textcolor{darkred}{+13.0}  &  \cellcolor{lightblue!50}\textcolor{darkred}{+13.4}  &  \cellcolor{lightblue!50}\textcolor{darkred}{+1.0}  &  \cellcolor{lightblue!50}\textcolor{darkgreen}{-0.9}  &  \cellcolor{lightblue!50}\textcolor{darkred}{+2.6}  &  \cellcolor{lightblue!50}\textcolor{darkred}{+0.2} \\
\cdashline{1-10}
Llama3.1{-}8B & 32.1 & 40.6 & 39.8 & 36.0 & 38.0 & 41.3 & 44.2 & 41.0 & 44.4 \\
Llama3.1{-}70B & 49.3 & 35.6 & 36.0 & 38.2 & 36.9 & 61.1 & 62.4 & 59.4 & 61.3 \\
$\Delta$  &  \cellcolor{lightblue!50}\textcolor{darkred}{-17.2} &  \cellcolor{lightblue!50}\textcolor{darkred}{+5.0}  &  \cellcolor{lightblue!50}\textcolor{darkred}{+3.7}  &  \cellcolor{lightblue!50}\textcolor{darkgreen}{-2.2}  &  \cellcolor{lightblue!50}\textcolor{darkred}{+1.1}  &  \cellcolor{lightblue!50}\textcolor{darkgreen}{-19.8}  &  \cellcolor{lightblue!50}\textcolor{darkgreen}{-18.1}  &  \cellcolor{lightblue!50}\textcolor{darkgreen}{-18.4}  &  \cellcolor{lightblue!50}\textcolor{darkgreen}{-17.0} \\
\cdashline{1-10}
Qwen3{-}8B{-}nothink & 26.8 & 50.3 & 51.2 & 50.5 & 51.6 & 50.3 & 51.2 & 50.5 & 51.6 \\
Qwen3{-}8B{-}think & 36.1 & 66.5 & 66.7 & 65.5 & 65.8 & 66.5 & 66.7 & 65.5 & 65.8 \\
$\Delta$  &  \cellcolor{lightblue!50}\textcolor{darkred}{-9.3} &  \cellcolor{lightblue!50}\textcolor{darkgreen}{-16.2}  &  \cellcolor{lightblue!50}\textcolor{darkgreen}{-15.5}  &  \cellcolor{lightblue!50}\textcolor{darkgreen}{-14.9}  &  \cellcolor{lightblue!50}\textcolor{darkgreen}{-14.1}  &  \cellcolor{lightblue!50}\textcolor{darkgreen}{-16.2}  &  \cellcolor{lightblue!50}\textcolor{darkgreen}{-15.5}  &  \cellcolor{lightblue!50}\textcolor{darkgreen}{-14.9}  &  \cellcolor{lightblue!50}\textcolor{darkgreen}{-14.1} \\
\cdashline{1-10}
Qwen3{-}32B{-}nothink & 31.3 & 33.6 & 36.5 & 40.8 & 42.1 & 33.6 & 36.5 & 32.7 & 36.0 \\
Qwen3{-}32B{-}think & 41.6 & 47.1 & 47.0 & 38.0 & 37.6 & 53.4 & 53.5 & 52.5 & 52.7 \\
$\Delta$  &  \cellcolor{lightblue!50}\textcolor{darkred}{-10.3} &  \cellcolor{lightblue!50}\textcolor{darkgreen}{-13.4}  &  \cellcolor{lightblue!50}\textcolor{darkgreen}{-10.5}  &  \cellcolor{lightblue!50}\textcolor{darkred}{+2.8}  &  \cellcolor{lightblue!50}\textcolor{darkred}{+4.5}  &  \cellcolor{lightblue!50}\textcolor{darkgreen}{-19.7}  &  \cellcolor{lightblue!50}\textcolor{darkgreen}{-17.0}  &  \cellcolor{lightblue!50}\textcolor{darkgreen}{-19.8}  &  \cellcolor{lightblue!50}\textcolor{darkgreen}{-16.7} \\
\cdashline{1-10}
GPT{-}OSS{-}20B{-}low & 35.8 & 44.0 & 44.1 & 46.9 & 46.9 & 44.0 & 44.1 & 44.6 & 44.6 \\
GPT{-}OSS{-}20B{-}high & 40.1 & 51.7 & 51.7 & 46.5 & 46.6 & 55.4 & 55.5 & 54.8 & 54.8 \\
$\Delta$  &  \cellcolor{lightblue!50}\textcolor{darkred}{-4.3} &  \cellcolor{lightblue!50}\textcolor{darkgreen}{-7.6}  &  \cellcolor{lightblue!50}\textcolor{darkgreen}{-7.7}  &  \cellcolor{lightblue!50}\textcolor{darkred}{+0.4}  &  \cellcolor{lightblue!50}\textcolor{darkred}{+0.4}  &  \cellcolor{lightblue!50}\textcolor{darkgreen}{-11.4}  &  \cellcolor{lightblue!50}\textcolor{darkgreen}{-11.4}  &  \cellcolor{lightblue!50}\textcolor{darkgreen}{-10.2}  &  \cellcolor{lightblue!50}\textcolor{darkgreen}{-10.2} \\
\cdashline{1-10}
GPT{-}OSS{-}120B{-}low & 46.3 & 39.6 & 39.6 & 43.0 & 43.1 & 39.6 & 39.6 & 39.1 & 39.1 \\
GPT{-}OSS{-}120B{-}high & 47.8 & 45.1 & 45.1 & 40.5 & 40.5 & 46.2 & 46.2 & 45.2 & 45.2 \\
$\Delta$  &  \cellcolor{lightblue!50}\textcolor{darkred}{-1.5} &  \cellcolor{lightblue!50}\textcolor{darkgreen}{-5.5}  &  \cellcolor{lightblue!50}\textcolor{darkgreen}{-5.5}  &  \cellcolor{lightblue!50}\textcolor{darkred}{+2.6}  &  \cellcolor{lightblue!50}\textcolor{darkred}{+2.6}  &  \cellcolor{lightblue!50}\textcolor{darkgreen}{-6.6}  &  \cellcolor{lightblue!50}\textcolor{darkgreen}{-6.6}  &  \cellcolor{lightblue!50}\textcolor{darkgreen}{-6.1}  &  \cellcolor{lightblue!50}\textcolor{darkgreen}{-6.1} \\
\cdashline{1-10}
GPT{-}OSS{-}20B{-}high & 40.1 & 51.7 & 51.7 & 46.5 & 46.6 & 55.4 & 55.5 & 54.8 & 54.8 \\
GPT{-}OSS{-}120B{-}high & 47.8 & 45.1 & 45.1 & 40.5 & 40.5 & 46.2 & 46.2 & 45.2 & 45.2 \\
$\Delta$  &  \cellcolor{lightblue!50}\textcolor{darkred}{-7.7} &  \cellcolor{lightblue!50}\textcolor{darkred}{+6.6}  &  \cellcolor{lightblue!50}\textcolor{darkred}{+6.7}  &  \cellcolor{lightblue!50}\textcolor{darkred}{+6.0}  &  \cellcolor{lightblue!50}\textcolor{darkred}{+6.1}  &  \cellcolor{lightblue!50}\textcolor{darkred}{+9.2}  &  \cellcolor{lightblue!50}\textcolor{darkred}{+9.3}  &  \cellcolor{lightblue!50}\textcolor{darkred}{+9.5}  &  \cellcolor{lightblue!50}\textcolor{darkred}{+9.6} \\
\midrule
\rowcolor[gray]{0.95}\multicolumn{10}{c}{\textit{Self-Consistency}} \\
Qwen2.5{-}7B & 21.4 & 27.4 & 24.4 & 17.5 & 14.5 & 27.4 & 24.4 & 25.2 & 21.2 \\
Qwen2.5{-}72B & 38.7 & 33.5 & 29.9 & 24.5 & 14.7 & 38.0 & 38.2 & 36.1 & 36.5 \\
$\Delta$  &  \cellcolor{lightblue!50}\textcolor{darkred}{-17.3} &  \cellcolor{lightblue!50}\textcolor{darkgreen}{-6.1}  &  \cellcolor{lightblue!50}\textcolor{darkgreen}{-5.5}  &  \cellcolor{lightblue!50}\textcolor{darkgreen}{-7.0}  &  \cellcolor{lightblue!50}\textcolor{darkgreen}{-0.2}  &  \cellcolor{lightblue!50}\textcolor{darkgreen}{-10.6}  &  \cellcolor{lightblue!50}\textcolor{darkgreen}{-13.8}  &  \cellcolor{lightblue!50}\textcolor{darkgreen}{-10.9}  &  \cellcolor{lightblue!50}\textcolor{darkgreen}{-15.3} \\
\cdashline{1-10}
Llama3.1{-}8B & 32.1 & 21.9 & 14.1 & 19.5 & 13.3 & 21.9 & 14.1 & 23.7 & 16.0 \\
Llama3.1{-}70B & 49.3 & 28.5 & 23.5 & 29.8 & 24.9 & 24.9 & 18.4 & 24.5 & 18.8 \\
$\Delta$  &  \cellcolor{lightblue!50}\textcolor{darkred}{-17.2} &  \cellcolor{lightblue!50}\textcolor{darkgreen}{-6.6}  &  \cellcolor{lightblue!50}\textcolor{darkgreen}{-9.4}  &  \cellcolor{lightblue!50}\textcolor{darkgreen}{-10.3}  &  \cellcolor{lightblue!50}\textcolor{darkgreen}{-11.6}  &  \cellcolor{lightblue!50}\textcolor{darkgreen}{-3.0}  &  \cellcolor{lightblue!50}\textcolor{darkgreen}{-4.2}  &  \cellcolor{lightblue!50}\textcolor{darkgreen}{-0.7}  &  \cellcolor{lightblue!50}\textcolor{darkgreen}{-2.8} \\
\cdashline{1-10}
Qwen3{-}8B{-}nothink & 26.8 & 26.2 & 22.6 & 17.8 & 11.5 & 26.1 & 22.5 & 23.5 & 19.7 \\
Qwen3{-}8B{-}think & 36.1 & 23.0 & 15.0 & 23.1 & 16.2 & 20.8 & 14.3 & 19.8 & 14.1 \\
$\Delta$  &  \cellcolor{lightblue!50}\textcolor{darkred}{-9.3} &  \cellcolor{lightblue!50}\textcolor{darkred}{+3.3}  &  \cellcolor{lightblue!50}\textcolor{darkred}{+7.6}  &  \cellcolor{lightblue!50}\textcolor{darkgreen}{-5.3}  &  \cellcolor{lightblue!50}\textcolor{darkgreen}{-4.7}  &  \cellcolor{lightblue!50}\textcolor{darkred}{+5.3}  &  \cellcolor{lightblue!50}\textcolor{darkred}{+8.2}  &  \cellcolor{lightblue!50}\textcolor{darkred}{+3.8}  &  \cellcolor{lightblue!50}\textcolor{darkred}{+5.7} \\
\cdashline{1-10}
Qwen3{-}32B{-}nothink & 31.3 & 24.7 & 18.4 & 20.2 & 14.0 & 24.7 & 18.4 & 24.5 & 17.5 \\
Qwen3{-}32B{-}think & 41.6 & 23.9 & 16.8 & 26.3 & 21.5 & 21.3 & 14.2 & 20.9 & 14.1 \\
$\Delta$  &  \cellcolor{lightblue!50}\textcolor{darkred}{-10.3} &  \cellcolor{lightblue!50}\textcolor{darkred}{+0.7}  &  \cellcolor{lightblue!50}\textcolor{darkred}{+1.6}  &  \cellcolor{lightblue!50}\textcolor{darkgreen}{-6.1}  &  \cellcolor{lightblue!50}\textcolor{darkgreen}{-7.5}  &  \cellcolor{lightblue!50}\textcolor{darkred}{+3.3}  &  \cellcolor{lightblue!50}\textcolor{darkred}{+4.3}  &  \cellcolor{lightblue!50}\textcolor{darkred}{+3.6}  &  \cellcolor{lightblue!50}\textcolor{darkred}{+3.4} \\
\cdashline{1-10}
GPT{-}OSS{-}20B{-}low & 35.8 & 21.9 & 14.7 & 21.4 & 15.6 & 21.9 & 14.7 & 22.9 & 15.9 \\
GPT{-}OSS{-}20B{-}high & 40.1 & 25.6 & 18.3 & 26.2 & 19.9 & 23.6 & 16.7 & 23.7 & 17.8 \\
$\Delta$  &  \cellcolor{lightblue!50}\textcolor{darkred}{-4.3} &  \cellcolor{lightblue!50}\textcolor{darkgreen}{-3.7}  &  \cellcolor{lightblue!50}\textcolor{darkgreen}{-3.6}  &  \cellcolor{lightblue!50}\textcolor{darkgreen}{-4.8}  &  \cellcolor{lightblue!50}\textcolor{darkgreen}{-4.3}  &  \cellcolor{lightblue!50}\textcolor{darkgreen}{-1.8}  &  \cellcolor{lightblue!50}\textcolor{darkgreen}{-2.0}  &  \cellcolor{lightblue!50}\textcolor{darkgreen}{-0.8}  &  \cellcolor{lightblue!50}\textcolor{darkgreen}{-1.9} \\
\cdashline{1-10}
GPT{-}OSS{-}120B{-}low & 46.3 & 29.4 & 23.7 & 28.3 & 22.7 & 29.4 & 23.7 & 30.6 & 25.1 \\
GPT{-}OSS{-}120B{-}high & 47.8 & 31.2 & 26.5 & 32.3 & 27.1 & 31.3 & 27.0 & 30.5 & 26.1 \\
$\Delta$  &  \cellcolor{lightblue!50}\textcolor{darkred}{-1.5} &  \cellcolor{lightblue!50}\textcolor{darkgreen}{-1.9}  &  \cellcolor{lightblue!50}\textcolor{darkgreen}{-2.8}  &  \cellcolor{lightblue!50}\textcolor{darkgreen}{-4.0}  &  \cellcolor{lightblue!50}\textcolor{darkgreen}{-4.4}  &  \cellcolor{lightblue!50}\textcolor{darkgreen}{-1.9}  &  \cellcolor{lightblue!50}\textcolor{darkgreen}{-3.3}  &  \cellcolor{lightblue!50}\textcolor{darkred}{+0.1}  &  \cellcolor{lightblue!50}\textcolor{darkgreen}{-1.1} \\
\cdashline{1-10}
GPT{-}OSS{-}20B{-}high & 40.1 & 25.6 & 18.3 & 26.2 & 19.9 & 23.6 & 16.7 & 23.7 & 17.8 \\
GPT{-}OSS{-}120B{-}high & 47.8 & 31.2 & 26.5 & 32.3 & 27.1 & 31.3 & 27.0 & 30.5 & 26.1 \\
$\Delta$  &  \cellcolor{lightblue!50}\textcolor{darkred}{-7.7} &  \cellcolor{lightblue!50}\textcolor{darkgreen}{-5.7}  &  \cellcolor{lightblue!50}\textcolor{darkgreen}{-8.2}  &  \cellcolor{lightblue!50}\textcolor{darkgreen}{-6.0}  &  \cellcolor{lightblue!50}\textcolor{darkgreen}{-7.2}  &  \cellcolor{lightblue!50}\textcolor{darkgreen}{-7.7}  &  \cellcolor{lightblue!50}\textcolor{darkgreen}{-10.3}  &  \cellcolor{lightblue!50}\textcolor{darkgreen}{-6.8}  &  \cellcolor{lightblue!50}\textcolor{darkgreen}{-8.3} \\
\bottomrule
\end{tabular}
}
\caption{Calibration metrics on \textbf{FreshQA}. For calibration metrics (Brier, ECE), lower is better ($\downarrow$). $\Delta$ rows show the difference (Model$_1$ $-$ Model$_2$). \textcolor{darkgreen}{Green}: better; \textcolor{darkred}{Red}: worse.}
\label{tab:freshqa}
\end{table*}


\begin{table*}[t!]
\setlength\tabcolsep{15.2pt}
\centering

\scalebox{0.72}{
\begin{tabular}{l c cc cc cc cc}
\toprule
\textbf{Model} & \textbf{Acc} & \multicolumn{2}{c}{\textbf{Raw}} & \multicolumn{2}{c}{\textbf{CA}} & \multicolumn{2}{c}{\textbf{DA}} & \multicolumn{2}{c}{\textbf{IA}} \\
\cmidrule(lr){3-4}\cmidrule(lr){5-6}\cmidrule(lr){7-8}\cmidrule(lr){9-10}
 & & Brier & ECE & Brier & ECE & Brier & ECE & Brier & ECE \\
\midrule
\rowcolor[gray]{0.95}\multicolumn{10}{c}{\textit{Verbalize}} \\
Qwen2.5{-}7B & 41.4 & 31.2 & 32.8 & 27.9 & 28.4 & 31.2 & 32.8 & 25.7 & 27.0 \\
Qwen2.5{-}72B & 71.7 & 20.5 & 20.6 & 18.3 & 18.6 & 38.7 & 40.0 & 26.7 & 27.6 \\
$\Delta$  &  \cellcolor{lightblue!50}\textcolor{darkred}{-30.3} &  \cellcolor{lightblue!50}\textcolor{darkred}{+10.7}  &  \cellcolor{lightblue!50}\textcolor{darkred}{+12.2}  &  \cellcolor{lightblue!50}\textcolor{darkred}{+9.6}  &  \cellcolor{lightblue!50}\textcolor{darkred}{+9.8}  &  \cellcolor{lightblue!50}\textcolor{darkgreen}{-7.5}  &  \cellcolor{lightblue!50}\textcolor{darkgreen}{-7.2}  &  \cellcolor{lightblue!50}\textcolor{darkgreen}{-1.0}  &  \cellcolor{lightblue!50}\textcolor{darkgreen}{-0.6} \\
\cdashline{1-10}
Llama3.1{-}8B & 59.9 & 29.8 & 29.9 & 24.5 & 23.6 & 29.8 & 29.9 & 18.7 & 18.3 \\
Llama3.1{-}70B & 79.1 & 19.0 & 18.9 & 18.9 & 18.8 & 29.5 & 29.5 & 18.4 & 18.3 \\
$\Delta$  &  \cellcolor{lightblue!50}\textcolor{darkred}{-19.2} &  \cellcolor{lightblue!50}\textcolor{darkred}{+10.8}  &  \cellcolor{lightblue!50}\textcolor{darkred}{+11.0}  &  \cellcolor{lightblue!50}\textcolor{darkred}{+5.6}  &  \cellcolor{lightblue!50}\textcolor{darkred}{+4.8}  &  \cellcolor{lightblue!50}\textcolor{darkred}{+0.2}  &  \cellcolor{lightblue!50}\textcolor{darkred}{+0.4}  &  \cellcolor{lightblue!50}\textcolor{darkred}{+0.3}  &  \cellcolor{lightblue!50}\textcolor{darkgreen}{-0.1} \\
\cdashline{1-10}
Qwen3{-}8B{-}nothink & 45.3 & 42.9 & 45.2 & 38.6 & 40.3 & 42.9 & 45.2 & 37.6 & 39.6 \\
Qwen3{-}8B{-}think & 60.3 & 32.4 & 33.8 & 30.1 & 32.5 & 44.5 & 47.4 & 37.5 & 40.2 \\
$\Delta$  &  \cellcolor{lightblue!50}\textcolor{darkred}{-15.0} &  \cellcolor{lightblue!50}\textcolor{darkred}{+10.5}  &  \cellcolor{lightblue!50}\textcolor{darkred}{+11.4}  &  \cellcolor{lightblue!50}\textcolor{darkred}{+8.5}  &  \cellcolor{lightblue!50}\textcolor{darkred}{+7.8}  &  \cellcolor{lightblue!50}\textcolor{darkgreen}{-1.5}  &  \cellcolor{lightblue!50}\textcolor{darkgreen}{-2.3}  &  \cellcolor{lightblue!50}\textcolor{darkred}{+0.2}  &  \cellcolor{lightblue!50}\textcolor{darkgreen}{-0.6} \\
\cdashline{1-10}
Qwen3{-}32B{-}nothink & 54.1 & 24.7 & 25.1 & 27.2 & 27.7 & 24.7 & 25.1 & 21.3 & 21.7 \\
Qwen3{-}32B{-}think & 64.5 & 28.7 & 28.9 & 25.7 & 26.4 & 36.6 & 37.7 & 30.6 & 31.7 \\
$\Delta$  &  \cellcolor{lightblue!50}\textcolor{darkred}{-10.4} &  \cellcolor{lightblue!50}\textcolor{darkgreen}{-4.0}  &  \cellcolor{lightblue!50}\textcolor{darkgreen}{-3.8}  &  \cellcolor{lightblue!50}\textcolor{darkred}{+1.4}  &  \cellcolor{lightblue!50}\textcolor{darkred}{+1.3}  &  \cellcolor{lightblue!50}\textcolor{darkgreen}{-11.9}  &  \cellcolor{lightblue!50}\textcolor{darkgreen}{-12.6}  &  \cellcolor{lightblue!50}\textcolor{darkgreen}{-9.3}  &  \cellcolor{lightblue!50}\textcolor{darkgreen}{-10.0} \\
\cdashline{1-10}
GPT{-}OSS{-}20B{-}low & 52.0 & 21.9 & 22.2 & 22.0 & 22.3 & 21.9 & 22.2 & 18.9 & 19.1 \\
GPT{-}OSS{-}20B{-}high & 60.3 & 24.3 & 25.0 & 25.3 & 25.7 & 29.1 & 31.1 & 24.6 & 27.0 \\
$\Delta$  &  \cellcolor{lightblue!50}\textcolor{darkred}{-8.3} &  \cellcolor{lightblue!50}\textcolor{darkgreen}{-2.4}  &  \cellcolor{lightblue!50}\textcolor{darkgreen}{-2.7}  &  \cellcolor{lightblue!50}\textcolor{darkgreen}{-3.3}  &  \cellcolor{lightblue!50}\textcolor{darkgreen}{-3.4}  &  \cellcolor{lightblue!50}\textcolor{darkgreen}{-7.2}  &  \cellcolor{lightblue!50}\textcolor{darkgreen}{-8.9}  &  \cellcolor{lightblue!50}\textcolor{darkgreen}{-5.6}  &  \cellcolor{lightblue!50}\textcolor{darkgreen}{-8.0} \\
\cdashline{1-10}
GPT{-}OSS{-}120B{-}low & 74.9 & 12.0 & 10.8 & 12.4 & 11.0 & 12.0 & 10.8 & 10.4 & 8.6 \\
GPT{-}OSS{-}120B{-}high & 78.8 & 12.9 & 11.8 & 12.3 & 11.4 & 17.2 & 17.0 & 12.1 & 11.4 \\
$\Delta$  &  \cellcolor{lightblue!50}\textcolor{darkred}{-3.9} &  \cellcolor{lightblue!50}\textcolor{darkgreen}{-0.9}  &  \cellcolor{lightblue!50}\textcolor{darkgreen}{-0.9}  &  \cellcolor{lightblue!50}\textcolor{darkred}{+0.1}  &  \cellcolor{lightblue!50}\textcolor{darkgreen}{-0.4}  &  \cellcolor{lightblue!50}\textcolor{darkgreen}{-5.2}  &  \cellcolor{lightblue!50}\textcolor{darkgreen}{-6.2}  &  \cellcolor{lightblue!50}\textcolor{darkgreen}{-1.7}  &  \cellcolor{lightblue!50}\textcolor{darkgreen}{-2.8} \\
\cdashline{1-10}
GPT{-}OSS{-}20B{-}high & 60.3 & 24.3 & 25.0 & 25.3 & 25.7 & 29.1 & 31.1 & 24.6 & 27.0 \\
GPT{-}OSS{-}120B{-}high & 78.8 & 12.9 & 11.8 & 12.3 & 11.4 & 17.2 & 17.0 & 12.1 & 11.4 \\
$\Delta$  &  \cellcolor{lightblue!50}\textcolor{darkred}{-18.5} &  \cellcolor{lightblue!50}\textcolor{darkred}{+11.4}  &  \cellcolor{lightblue!50}\textcolor{darkred}{+13.2}  &  \cellcolor{lightblue!50}\textcolor{darkred}{+13.0}  &  \cellcolor{lightblue!50}\textcolor{darkred}{+14.3}  &  \cellcolor{lightblue!50}\textcolor{darkred}{+11.9}  &  \cellcolor{lightblue!50}\textcolor{darkred}{+14.1}  &  \cellcolor{lightblue!50}\textcolor{darkred}{+12.5}  &  \cellcolor{lightblue!50}\textcolor{darkred}{+15.7} \\
\midrule
\rowcolor[gray]{0.95}\multicolumn{10}{c}{\textit{P(True)}} \\
Qwen2.5{-}7B & 41.4 & 25.9 & 25.9 & 26.0 & 25.7 & 25.9 & 25.9 & 22.7 & 22.6 \\
Qwen2.5{-}72B & 71.7 & 17.9 & 17.4 & 15.4 & 14.7 & 31.8 & 32.5 & 21.6 & 21.9 \\
$\Delta$  &  \cellcolor{lightblue!50}\textcolor{darkred}{-30.3} &  \cellcolor{lightblue!50}\textcolor{darkred}{+8.0}  &  \cellcolor{lightblue!50}\textcolor{darkred}{+8.4}  &  \cellcolor{lightblue!50}\textcolor{darkred}{+10.6}  &  \cellcolor{lightblue!50}\textcolor{darkred}{+11.0}  &  \cellcolor{lightblue!50}\textcolor{darkgreen}{-5.9}  &  \cellcolor{lightblue!50}\textcolor{darkgreen}{-6.6}  &  \cellcolor{lightblue!50}\textcolor{darkred}{+1.1}  &  \cellcolor{lightblue!50}\textcolor{darkred}{+0.7} \\
\cdashline{1-10}
Llama3.1{-}8B & 59.9 & 26.0 & 25.0 & 21.9 & 19.1 & 26.0 & 25.0 & 17.3 & 14.4 \\
Llama3.1{-}70B & 79.1 & 17.1 & 15.6 & 17.2 & 15.9 & 29.5 & 29.8 & 18.2 & 17.9 \\
$\Delta$  &  \cellcolor{lightblue!50}\textcolor{darkred}{-19.2} &  \cellcolor{lightblue!50}\textcolor{darkred}{+9.0}  &  \cellcolor{lightblue!50}\textcolor{darkred}{+9.5}  &  \cellcolor{lightblue!50}\textcolor{darkred}{+4.8}  &  \cellcolor{lightblue!50}\textcolor{darkred}{+3.2}  &  \cellcolor{lightblue!50}\textcolor{darkgreen}{-3.4}  &  \cellcolor{lightblue!50}\textcolor{darkgreen}{-4.8}  &  \cellcolor{lightblue!50}\textcolor{darkgreen}{-0.9}  &  \cellcolor{lightblue!50}\textcolor{darkgreen}{-3.5} \\
\cdashline{1-10}
Qwen3{-}8B{-}nothink & 45.3 & 38.2 & 38.6 & 34.8 & 34.9 & 38.2 & 38.6 & 35.1 & 35.3 \\
Qwen3{-}8B{-}think & 60.3 & 34.6 & 34.7 & 30.1 & 30.2 & 46.6 & 46.9 & 39.3 & 39.8 \\
$\Delta$  &  \cellcolor{lightblue!50}\textcolor{darkred}{-15.0} &  \cellcolor{lightblue!50}\textcolor{darkred}{+3.6}  &  \cellcolor{lightblue!50}\textcolor{darkred}{+3.9}  &  \cellcolor{lightblue!50}\textcolor{darkred}{+4.8}  &  \cellcolor{lightblue!50}\textcolor{darkred}{+4.8}  &  \cellcolor{lightblue!50}\textcolor{darkgreen}{-8.4}  &  \cellcolor{lightblue!50}\textcolor{darkgreen}{-8.4}  &  \cellcolor{lightblue!50}\textcolor{darkgreen}{-4.3}  &  \cellcolor{lightblue!50}\textcolor{darkgreen}{-4.5} \\
\cdashline{1-10}
Qwen3{-}32B{-}nothink & 54.1 & 22.7 & 22.1 & 25.1 & 24.6 & 22.7 & 22.1 & 20.8 & 20.2 \\
Qwen3{-}32B{-}think & 64.5 & 29.6 & 29.5 & 24.2 & 24.0 & 35.8 & 35.8 & 30.1 & 30.1 \\
$\Delta$  &  \cellcolor{lightblue!50}\textcolor{darkred}{-10.4} &  \cellcolor{lightblue!50}\textcolor{darkgreen}{-6.9}  &  \cellcolor{lightblue!50}\textcolor{darkgreen}{-7.5}  &  \cellcolor{lightblue!50}\textcolor{darkred}{+0.9}  &  \cellcolor{lightblue!50}\textcolor{darkred}{+0.6}  &  \cellcolor{lightblue!50}\textcolor{darkgreen}{-13.1}  &  \cellcolor{lightblue!50}\textcolor{darkgreen}{-13.7}  &  \cellcolor{lightblue!50}\textcolor{darkgreen}{-9.3}  &  \cellcolor{lightblue!50}\textcolor{darkgreen}{-9.9} \\
\cdashline{1-10}
GPT{-}OSS{-}20B{-}low & 52.0 & 25.0 & 25.1 & 26.8 & 26.8 & 25.0 & 25.1 & 22.3 & 22.3 \\
GPT{-}OSS{-}20B{-}high & 60.3 & 31.2 & 31.2 & 27.3 & 27.4 & 36.4 & 36.4 & 31.4 & 31.4 \\
$\Delta$  &  \cellcolor{lightblue!50}\textcolor{darkred}{-8.3} &  \cellcolor{lightblue!50}\textcolor{darkgreen}{-6.1}  &  \cellcolor{lightblue!50}\textcolor{darkgreen}{-6.2}  &  \cellcolor{lightblue!50}\textcolor{darkgreen}{-0.5}  &  \cellcolor{lightblue!50}\textcolor{darkgreen}{-0.5}  &  \cellcolor{lightblue!50}\textcolor{darkgreen}{-11.4}  &  \cellcolor{lightblue!50}\textcolor{darkgreen}{-11.4}  &  \cellcolor{lightblue!50}\textcolor{darkgreen}{-9.1}  &  \cellcolor{lightblue!50}\textcolor{darkgreen}{-9.1} \\
\cdashline{1-10}
GPT{-}OSS{-}120B{-}low & 74.9 & 14.7 & 14.7 & 16.2 & 16.2 & 14.7 & 14.7 & 12.7 & 12.7 \\
GPT{-}OSS{-}120B{-}high & 78.8 & 19.1 & 19.1 & 16.9 & 16.9 & 24.8 & 24.8 & 18.1 & 18.1 \\
$\Delta$  &  \cellcolor{lightblue!50}\textcolor{darkred}{-3.9} &  \cellcolor{lightblue!50}\textcolor{darkgreen}{-4.3}  &  \cellcolor{lightblue!50}\textcolor{darkgreen}{-4.3}  &  \cellcolor{lightblue!50}\textcolor{darkgreen}{-0.6}  &  \cellcolor{lightblue!50}\textcolor{darkgreen}{-0.6}  &  \cellcolor{lightblue!50}\textcolor{darkgreen}{-10.1}  &  \cellcolor{lightblue!50}\textcolor{darkgreen}{-10.1}  &  \cellcolor{lightblue!50}\textcolor{darkgreen}{-5.4}  &  \cellcolor{lightblue!50}\textcolor{darkgreen}{-5.4} \\
\cdashline{1-10}
GPT{-}OSS{-}20B{-}high & 60.3 & 31.2 & 31.2 & 27.3 & 27.4 & 36.4 & 36.4 & 31.4 & 31.4 \\
GPT{-}OSS{-}120B{-}high & 78.8 & 19.1 & 19.1 & 16.9 & 16.9 & 24.8 & 24.8 & 18.1 & 18.1 \\
$\Delta$  &  \cellcolor{lightblue!50}\textcolor{darkred}{-18.5} &  \cellcolor{lightblue!50}\textcolor{darkred}{+12.1}  &  \cellcolor{lightblue!50}\textcolor{darkred}{+12.1}  &  \cellcolor{lightblue!50}\textcolor{darkred}{+10.5}  &  \cellcolor{lightblue!50}\textcolor{darkred}{+10.5}  &  \cellcolor{lightblue!50}\textcolor{darkred}{+11.6}  &  \cellcolor{lightblue!50}\textcolor{darkred}{+11.7}  &  \cellcolor{lightblue!50}\textcolor{darkred}{+13.2}  &  \cellcolor{lightblue!50}\textcolor{darkred}{+13.2} \\
\midrule
\rowcolor[gray]{0.95}\multicolumn{10}{c}{\textit{Self-Consistency}} \\
Qwen2.5{-}7B & 41.4 & 17.9 & 8.9 & 13.6 & 7.8 & 17.9 & 8.9 & 18.1 & 12.2 \\
Qwen2.5{-}72B & 71.7 & 17.1 & 11.7 & 21.5 & 22.4 & 25.1 & 22.7 & 19.2 & 14.2 \\
$\Delta$  &  \cellcolor{lightblue!50}\textcolor{darkred}{-30.3} &  \cellcolor{lightblue!50}\textcolor{darkred}{+0.9}  &  \cellcolor{lightblue!50}\textcolor{darkgreen}{-2.7}  &  \cellcolor{lightblue!50}\textcolor{darkgreen}{-7.9}  &  \cellcolor{lightblue!50}\textcolor{darkgreen}{-14.6}  &  \cellcolor{lightblue!50}\textcolor{darkgreen}{-7.2}  &  \cellcolor{lightblue!50}\textcolor{darkgreen}{-13.8}  &  \cellcolor{lightblue!50}\textcolor{darkgreen}{-1.1}  &  \cellcolor{lightblue!50}\textcolor{darkgreen}{-2.1} \\
\cdashline{1-10}
Llama3.1{-}8B & 59.9 & 23.8 & 19.4 & 23.3 & 21.6 & 23.8 & 19.4 & 26.8 & 28.0 \\
Llama3.1{-}70B & 79.1 & 18.8 & 17.0 & 25.2 & 30.0 & 18.7 & 11.8 & 17.0 & 13.9 \\
$\Delta$  &  \cellcolor{lightblue!50}\textcolor{darkred}{-19.2} &  \cellcolor{lightblue!50}\textcolor{darkred}{+5.0}  &  \cellcolor{lightblue!50}\textcolor{darkred}{+2.4}  &  \cellcolor{lightblue!50}\textcolor{darkgreen}{-1.9}  &  \cellcolor{lightblue!50}\textcolor{darkgreen}{-8.4}  &  \cellcolor{lightblue!50}\textcolor{darkred}{+5.1}  &  \cellcolor{lightblue!50}\textcolor{darkred}{+7.6}  &  \cellcolor{lightblue!50}\textcolor{darkred}{+9.7}  &  \cellcolor{lightblue!50}\textcolor{darkred}{+14.1} \\
\cdashline{1-10}
Qwen3{-}8B{-}nothink & 45.3 & 19.8 & 11.9 & 14.8 & 4.9 & 19.8 & 11.9 & 17.4 & 9.0 \\
Qwen3{-}8B{-}think & 60.3 & 14.5 & 7.4 & 19.2 & 17.9 & 14.3 & 4.7 & 13.0 & 4.5 \\
$\Delta$  &  \cellcolor{lightblue!50}\textcolor{darkred}{-15.0} &  \cellcolor{lightblue!50}\textcolor{darkred}{+5.3}  &  \cellcolor{lightblue!50}\textcolor{darkred}{+4.4}  &  \cellcolor{lightblue!50}\textcolor{darkgreen}{-4.4}  &  \cellcolor{lightblue!50}\textcolor{darkgreen}{-13.0}  &  \cellcolor{lightblue!50}\textcolor{darkred}{+5.5}  &  \cellcolor{lightblue!50}\textcolor{darkred}{+7.2}  &  \cellcolor{lightblue!50}\textcolor{darkred}{+4.4}  &  \cellcolor{lightblue!50}\textcolor{darkred}{+4.5} \\
\cdashline{1-10}
Qwen3{-}32B{-}nothink & 54.1 & 17.6 & 9.5 & 14.7 & 7.4 & 17.6 & 9.5 & 15.9 & 8.5 \\
Qwen3{-}32B{-}think & 64.5 & 14.7 & 11.3 & 18.2 & 18.2 & 13.8 & 7.3 & 13.3 & 8.2 \\
$\Delta$  &  \cellcolor{lightblue!50}\textcolor{darkred}{-10.4} &  \cellcolor{lightblue!50}\textcolor{darkred}{+2.9}  &  \cellcolor{lightblue!50}\textcolor{darkgreen}{-1.8}  &  \cellcolor{lightblue!50}\textcolor{darkgreen}{-3.5}  &  \cellcolor{lightblue!50}\textcolor{darkgreen}{-10.8}  &  \cellcolor{lightblue!50}\textcolor{darkred}{+3.8}  &  \cellcolor{lightblue!50}\textcolor{darkred}{+2.2}  &  \cellcolor{lightblue!50}\textcolor{darkred}{+2.6}  &  \cellcolor{lightblue!50}\textcolor{darkred}{+0.3} \\
\cdashline{1-10}
GPT{-}OSS{-}20B{-}low & 52.0 & 17.4 & 13.3 & 17.6 & 15.9 & 17.4 & 13.3 & 18.0 & 14.7 \\
GPT{-}OSS{-}20B{-}high & 60.3 & 17.9 & 12.8 & 19.8 & 18.3 & 16.9 & 9.0 & 16.5 & 9.9 \\
$\Delta$  &  \cellcolor{lightblue!50}\textcolor{darkred}{-8.3} &  \cellcolor{lightblue!50}\textcolor{darkgreen}{-0.5}  &  \cellcolor{lightblue!50}\textcolor{darkred}{+0.4}  &  \cellcolor{lightblue!50}\textcolor{darkgreen}{-2.2}  &  \cellcolor{lightblue!50}\textcolor{darkgreen}{-2.4}  &  \cellcolor{lightblue!50}\textcolor{darkred}{+0.5}  &  \cellcolor{lightblue!50}\textcolor{darkred}{+4.2}  &  \cellcolor{lightblue!50}\textcolor{darkred}{+1.4}  &  \cellcolor{lightblue!50}\textcolor{darkred}{+4.8} \\
\cdashline{1-10}
GPT{-}OSS{-}120B{-}low & 74.9 & 13.7 & 9.7 & 13.9 & 11.1 & 13.7 & 9.7 & 13.1 & 10.4 \\
GPT{-}OSS{-}120B{-}high & 78.8 & 14.6 & 10.8 & 15.7 & 13.1 & 15.1 & 9.9 & 14.1 & 10.2 \\
$\Delta$  &  \cellcolor{lightblue!50}\textcolor{darkred}{-3.9} &  \cellcolor{lightblue!50}\textcolor{darkgreen}{-0.9}  &  \cellcolor{lightblue!50}\textcolor{darkgreen}{-1.0}  &  \cellcolor{lightblue!50}\textcolor{darkgreen}{-1.8}  &  \cellcolor{lightblue!50}\textcolor{darkgreen}{-1.9}  &  \cellcolor{lightblue!50}\textcolor{darkgreen}{-1.4}  &  \cellcolor{lightblue!50}\textcolor{darkgreen}{-0.2}  &  \cellcolor{lightblue!50}\textcolor{darkgreen}{-0.9}  &  \cellcolor{lightblue!50}\textcolor{darkred}{+0.3} \\
\cdashline{1-10}
GPT{-}OSS{-}20B{-}high & 60.3 & 17.9 & 12.8 & 19.8 & 18.3 & 16.9 & 9.0 & 16.5 & 9.9 \\
GPT{-}OSS{-}120B{-}high & 78.8 & 14.6 & 10.8 & 15.7 & 13.1 & 15.1 & 9.9 & 14.1 & 10.2 \\
$\Delta$  &  \cellcolor{lightblue!50}\textcolor{darkred}{-18.5} &  \cellcolor{lightblue!50}\textcolor{darkred}{+3.3}  &  \cellcolor{lightblue!50}\textcolor{darkred}{+2.1}  &  \cellcolor{lightblue!50}\textcolor{darkred}{+4.1}  &  \cellcolor{lightblue!50}\textcolor{darkred}{+5.2}  &  \cellcolor{lightblue!50}\textcolor{darkred}{+1.8}  &  \cellcolor{lightblue!50}\textcolor{darkgreen}{-0.9}  &  \cellcolor{lightblue!50}\textcolor{darkred}{+2.5}  &  \cellcolor{lightblue!50}\textcolor{darkgreen}{-0.2} \\
\bottomrule
\end{tabular}
}
\caption{Calibration metrics on \textbf{TriviaQA}. For calibration metrics (Brier, ECE), lower is better ($\downarrow$). $\Delta$ rows show the difference (Model$_1$ $-$ Model$_2$). \textcolor{darkgreen}{Green}: better; \textcolor{darkred}{Red}: worse.}
\label{tab:triviaqa}
\end{table*}


\begin{table*}[t!]
\setlength\tabcolsep{15.2pt}
\centering

\scalebox{0.72}{
\begin{tabular}{l c cc cc cc cc}
\toprule
\textbf{Model} & \textbf{Acc} & \multicolumn{2}{c}{\textbf{Raw}} & \multicolumn{2}{c}{\textbf{CA}} & \multicolumn{2}{c}{\textbf{DA}} & \multicolumn{2}{c}{\textbf{IA}} \\
\cmidrule(lr){3-4}\cmidrule(lr){5-6}\cmidrule(lr){7-8}\cmidrule(lr){9-10}
 & & Brier & ECE & Brier & ECE & Brier & ECE & Brier & ECE \\
\midrule
\rowcolor[gray]{0.95}\multicolumn{10}{c}{\textit{Verbalize}} \\
Qwen2.5{-}7B & 33.3 & 52.5 & 54.8 & 48.6 & 49.3 & 54.7 & 57.1 & 50.1 & 52.1 \\
Qwen2.5{-}72B & 50.0 & 41.8 & 41.2 & 41.7 & 42.4 & 54.7 & 57.1 & 52.0 & 54.6 \\
$\Delta$  &  \cellcolor{lightblue!50}\textcolor{darkred}{-16.7} &  \cellcolor{lightblue!50}\textcolor{darkred}{+10.8}  &  \cellcolor{lightblue!50}\textcolor{darkred}{+13.6}  &  \cellcolor{lightblue!50}\textcolor{darkred}{+6.9}  &  \cellcolor{lightblue!50}\textcolor{darkred}{+6.9}  &  \cellcolor{lightblue!50}0.0  &  \cellcolor{lightblue!50}0.0  &  \cellcolor{lightblue!50}\textcolor{darkgreen}{-1.9}  &  \cellcolor{lightblue!50}\textcolor{darkgreen}{-2.5} \\
\cdashline{1-10}
Llama3.1{-}8B & 18.7 & 61.9 & 63.7 & 54.6 & 54.9 & 61.9 & 63.7 & 62.9 & 64.8 \\
Llama3.1{-}70B & 39.9 & 43.6 & 45.5 & 39.7 & 40.8 & 47.0 & 50.4 & 53.3 & 56.6 \\
$\Delta$  &  \cellcolor{lightblue!50}\textcolor{darkred}{-21.2} &  \cellcolor{lightblue!50}\textcolor{darkred}{+18.3}  &  \cellcolor{lightblue!50}\textcolor{darkred}{+18.2}  &  \cellcolor{lightblue!50}\textcolor{darkred}{+14.9}  &  \cellcolor{lightblue!50}\textcolor{darkred}{+14.1}  &  \cellcolor{lightblue!50}\textcolor{darkred}{+14.9}  &  \cellcolor{lightblue!50}\textcolor{darkred}{+13.3}  &  \cellcolor{lightblue!50}\textcolor{darkred}{+9.5}  &  \cellcolor{lightblue!50}\textcolor{darkred}{+8.2} \\
\cdashline{1-10}
Qwen3{-}8B{-}nothink & 46.0 & 41.9 & 42.5 & 32.5 & 29.5 & 45.5 & 45.5 & 38.0 & 36.7 \\
Qwen3{-}8B{-}think & 61.6 & 26.8 & 25.9 & 31.5 & 32.8 & 45.1 & 46.4 & 38.3 & 37.8 \\
$\Delta$  &  \cellcolor{lightblue!50}\textcolor{darkred}{-15.6} &  \cellcolor{lightblue!50}\textcolor{darkred}{+15.2}  &  \cellcolor{lightblue!50}\textcolor{darkred}{+16.6}  &  \cellcolor{lightblue!50}\textcolor{darkred}{+0.9}  &  \cellcolor{lightblue!50}\textcolor{darkgreen}{-3.3}  &  \cellcolor{lightblue!50}\textcolor{darkred}{+0.4}  &  \cellcolor{lightblue!50}\textcolor{darkgreen}{-0.9}  &  \cellcolor{lightblue!50}\textcolor{darkgreen}{-0.3}  &  \cellcolor{lightblue!50}\textcolor{darkgreen}{-1.1} \\
\cdashline{1-10}
Qwen3{-}32B{-}nothink & 51.5 & 38.5 & 39.2 & 34.0 & 33.8 & 38.5 & 39.2 & 29.5 & 28.9 \\
Qwen3{-}32B{-}think & 67.7 & 31.9 & 31.9 & 35.9 & 35.8 & 45.4 & 46.2 & 35.8 & 36.0 \\
$\Delta$  &  \cellcolor{lightblue!50}\textcolor{darkred}{-16.2} &  \cellcolor{lightblue!50}\textcolor{darkred}{+6.6}  &  \cellcolor{lightblue!50}\textcolor{darkred}{+7.3}  &  \cellcolor{lightblue!50}\textcolor{darkgreen}{-1.9}  &  \cellcolor{lightblue!50}\textcolor{darkgreen}{-2.0}  &  \cellcolor{lightblue!50}\textcolor{darkgreen}{-6.9}  &  \cellcolor{lightblue!50}\textcolor{darkgreen}{-7.1}  &  \cellcolor{lightblue!50}\textcolor{darkgreen}{-6.3}  &  \cellcolor{lightblue!50}\textcolor{darkgreen}{-7.1} \\
\cdashline{1-10}
GPT{-}OSS{-}20B{-}low & 56.6 & 24.2 & 19.7 & 25.8 & 20.0 & 24.2 & 19.7 & 12.8 & 7.6 \\
GPT{-}OSS{-}20B{-}high & 74.2 & 12.7 & 9.2 & 22.8 & 19.3 & 33.7 & 34.3 & 11.3 & 8.7 \\
$\Delta$  &  \cellcolor{lightblue!50}\textcolor{darkred}{-17.6} &  \cellcolor{lightblue!50}\textcolor{darkred}{+11.5}  &  \cellcolor{lightblue!50}\textcolor{darkred}{+10.5}  &  \cellcolor{lightblue!50}\textcolor{darkred}{+3.0}  &  \cellcolor{lightblue!50}\textcolor{darkred}{+0.7}  &  \cellcolor{lightblue!50}\textcolor{darkgreen}{-9.5}  &  \cellcolor{lightblue!50}\textcolor{darkgreen}{-14.6}  &  \cellcolor{lightblue!50}\textcolor{darkred}{+1.5}  &  \cellcolor{lightblue!50}\textcolor{darkgreen}{-1.1} \\
\cdashline{1-10}
GPT{-}OSS{-}120B{-}low & 64.6 & 17.5 & 11.6 & 20.9 & 14.5 & 17.5 & 11.6 & 12.9 & 8.6 \\
GPT{-}OSS{-}120B{-}high & 81.3 & 14.5 & 11.2 & 15.5 & 12.9 & 26.3 & 26.1 & 15.6 & 14.2 \\
$\Delta$  &  \cellcolor{lightblue!50}\textcolor{darkred}{-16.7} &  \cellcolor{lightblue!50}\textcolor{darkred}{+2.9}  &  \cellcolor{lightblue!50}\textcolor{darkred}{+0.5}  &  \cellcolor{lightblue!50}\textcolor{darkred}{+5.4}  &  \cellcolor{lightblue!50}\textcolor{darkred}{+1.6}  &  \cellcolor{lightblue!50}\textcolor{darkgreen}{-8.8}  &  \cellcolor{lightblue!50}\textcolor{darkgreen}{-14.5}  &  \cellcolor{lightblue!50}\textcolor{darkgreen}{-2.7}  &  \cellcolor{lightblue!50}\textcolor{darkgreen}{-5.5} \\
\cdashline{1-10}
GPT{-}OSS{-}20B{-}high & 74.2 & 12.7 & 9.2 & 22.8 & 19.3 & 33.7 & 34.3 & 11.3 & 8.7 \\
GPT{-}OSS{-}120B{-}high & 81.3 & 14.5 & 11.2 & 15.5 & 12.9 & 26.3 & 26.1 & 15.6 & 14.2 \\
$\Delta$  &  \cellcolor{lightblue!50}\textcolor{darkred}{-7.1} &  \cellcolor{lightblue!50}\textcolor{darkgreen}{-1.8}  &  \cellcolor{lightblue!50}\textcolor{darkgreen}{-2.0}  &  \cellcolor{lightblue!50}\textcolor{darkred}{+7.3}  &  \cellcolor{lightblue!50}\textcolor{darkred}{+6.4}  &  \cellcolor{lightblue!50}\textcolor{darkred}{+7.4}  &  \cellcolor{lightblue!50}\textcolor{darkred}{+8.2}  &  \cellcolor{lightblue!50}\textcolor{darkgreen}{-4.2}  &  \cellcolor{lightblue!50}\textcolor{darkgreen}{-5.5} \\
\midrule
\rowcolor[gray]{0.95}\multicolumn{10}{c}{\textit{P(True)}} \\
Qwen2.5{-}7B & 33.3 & 34.7 & 33.5 & 39.5 & 37.5 & 34.7 & 33.5 & 31.1 & 30.1 \\
Qwen2.5{-}72B & 50.0 & 40.3 & 39.1 & 34.8 & 32.0 & 47.4 & 48.2 & 46.4 & 48.8 \\
$\Delta$  &  \cellcolor{lightblue!50}\textcolor{darkred}{-16.7} &  \cellcolor{lightblue!50}\textcolor{darkgreen}{-5.6}  &  \cellcolor{lightblue!50}\textcolor{darkgreen}{-5.6}  &  \cellcolor{lightblue!50}\textcolor{darkred}{+4.8}  &  \cellcolor{lightblue!50}\textcolor{darkred}{+5.5}  &  \cellcolor{lightblue!50}\textcolor{darkgreen}{-12.7}  &  \cellcolor{lightblue!50}\textcolor{darkgreen}{-14.7}  &  \cellcolor{lightblue!50}\textcolor{darkgreen}{-15.3}  &  \cellcolor{lightblue!50}\textcolor{darkgreen}{-18.7} \\
\cdashline{1-10}
Llama3.1{-}8B & 18.7 & 38.2 & 40.4 & 36.1 & 34.3 & 38.2 & 40.4 & 39.4 & 43.0 \\
Llama3.1{-}70B & 39.9 & 39.8 & 37.3 & 37.5 & 35.4 & 49.2 & 52.0 & 51.3 & 55.9 \\
$\Delta$  &  \cellcolor{lightblue!50}\textcolor{darkred}{-21.2} &  \cellcolor{lightblue!50}\textcolor{darkgreen}{-1.6}  &  \cellcolor{lightblue!50}\textcolor{darkred}{+3.1}  &  \cellcolor{lightblue!50}\textcolor{darkgreen}{-1.4}  &  \cellcolor{lightblue!50}\textcolor{darkgreen}{-1.2}  &  \cellcolor{lightblue!50}\textcolor{darkgreen}{-11.1}  &  \cellcolor{lightblue!50}\textcolor{darkgreen}{-11.6}  &  \cellcolor{lightblue!50}\textcolor{darkgreen}{-12.0}  &  \cellcolor{lightblue!50}\textcolor{darkgreen}{-12.9} \\
\cdashline{1-10}
Qwen3{-}8B{-}nothink & 46.0 & 40.6 & 39.3 & 32.5 & 29.5 & 40.6 & 39.3 & 41.7 & 40.2 \\
Qwen3{-}8B{-}think & 61.6 & 26.8 & 25.9 & 31.5 & 32.8 & 47.5 & 47.9 & 43.4 & 43.6 \\
$\Delta$  &  \cellcolor{lightblue!50}\textcolor{darkred}{-15.6} &  \cellcolor{lightblue!50}\textcolor{darkred}{+13.8}  &  \cellcolor{lightblue!50}\textcolor{darkred}{+13.4}  &  \cellcolor{lightblue!50}\textcolor{darkred}{+0.9}  &  \cellcolor{lightblue!50}\textcolor{darkgreen}{-3.3}  &  \cellcolor{lightblue!50}\textcolor{darkgreen}{-6.9}  &  \cellcolor{lightblue!50}\textcolor{darkgreen}{-8.6}  &  \cellcolor{lightblue!50}\textcolor{darkgreen}{-1.7}  &  \cellcolor{lightblue!50}\textcolor{darkgreen}{-3.4} \\
\cdashline{1-10}
Qwen3{-}32B{-}nothink & 51.5 & 32.1 & 28.8 & 32.0 & 28.9 & 32.1 & 28.8 & 28.8 & 25.2 \\
Qwen3{-}32B{-}think & 67.7 & 38.1 & 38.1 & 38.2 & 38.1 & 47.8 & 47.7 & 41.4 & 41.4 \\
$\Delta$  &  \cellcolor{lightblue!50}\textcolor{darkred}{-16.2} &  \cellcolor{lightblue!50}\textcolor{darkgreen}{-6.0}  &  \cellcolor{lightblue!50}\textcolor{darkgreen}{-9.3}  &  \cellcolor{lightblue!50}\textcolor{darkgreen}{-6.2}  &  \cellcolor{lightblue!50}\textcolor{darkgreen}{-9.2}  &  \cellcolor{lightblue!50}\textcolor{darkgreen}{-15.7}  &  \cellcolor{lightblue!50}\textcolor{darkgreen}{-18.9}  &  \cellcolor{lightblue!50}\textcolor{darkgreen}{-12.6}  &  \cellcolor{lightblue!50}\textcolor{darkgreen}{-16.2} \\
\cdashline{1-10}
GPT{-}OSS{-}20B{-}low & 56.6 & 29.9 & 30.2 & 36.8 & 37.0 & 29.9 & 30.2 & 23.9 & 24.2 \\
GPT{-}OSS{-}20B{-}high & 74.2 & 17.4 & 17.6 & 22.5 & 22.7 & 38.7 & 38.8 & 17.4 & 17.4 \\
$\Delta$  &  \cellcolor{lightblue!50}\textcolor{darkred}{-17.6} &  \cellcolor{lightblue!50}\textcolor{darkred}{+12.5}  &  \cellcolor{lightblue!50}\textcolor{darkred}{+12.7}  &  \cellcolor{lightblue!50}\textcolor{darkred}{+14.4}  &  \cellcolor{lightblue!50}\textcolor{darkred}{+14.3}  &  \cellcolor{lightblue!50}\textcolor{darkgreen}{-8.8}  &  \cellcolor{lightblue!50}\textcolor{darkgreen}{-8.6}  &  \cellcolor{lightblue!50}\textcolor{darkred}{+6.5}  &  \cellcolor{lightblue!50}\textcolor{darkred}{+6.8} \\
\cdashline{1-10}
GPT{-}OSS{-}120B{-}low & 64.6 & 28.5 & 28.5 & 27.2 & 27.0 & 28.5 & 28.5 & 19.4 & 19.4 \\
GPT{-}OSS{-}120B{-}high & 81.3 & 25.2 & 25.2 & 25.0 & 25.0 & 39.5 & 39.5 & 25.2 & 25.2 \\
$\Delta$  &  \cellcolor{lightblue!50}\textcolor{darkred}{-16.7} &  \cellcolor{lightblue!50}\textcolor{darkred}{+3.2}  &  \cellcolor{lightblue!50}\textcolor{darkred}{+3.3}  &  \cellcolor{lightblue!50}\textcolor{darkred}{+2.2}  &  \cellcolor{lightblue!50}\textcolor{darkred}{+1.9}  &  \cellcolor{lightblue!50}\textcolor{darkgreen}{-11.0}  &  \cellcolor{lightblue!50}\textcolor{darkgreen}{-10.9}  &  \cellcolor{lightblue!50}\textcolor{darkgreen}{-5.7}  &  \cellcolor{lightblue!50}\textcolor{darkgreen}{-5.8} \\
\cdashline{1-10}
GPT{-}OSS{-}20B{-}high & 74.2 & 17.4 & 17.6 & 22.5 & 22.7 & 38.7 & 38.8 & 17.4 & 17.4 \\
GPT{-}OSS{-}120B{-}high & 81.3 & 25.2 & 25.2 & 25.0 & 25.0 & 39.5 & 39.5 & 25.2 & 25.2 \\
$\Delta$  &  \cellcolor{lightblue!50}\textcolor{darkred}{-7.1} &  \cellcolor{lightblue!50}\textcolor{darkgreen}{-7.8}  &  \cellcolor{lightblue!50}\textcolor{darkgreen}{-7.7}  &  \cellcolor{lightblue!50}\textcolor{darkgreen}{-2.5}  &  \cellcolor{lightblue!50}\textcolor{darkgreen}{-2.3}  &  \cellcolor{lightblue!50}\textcolor{darkgreen}{-0.7}  &  \cellcolor{lightblue!50}\textcolor{darkgreen}{-0.7}  &  \cellcolor{lightblue!50}\textcolor{darkgreen}{-7.7}  &  \cellcolor{lightblue!50}\textcolor{darkgreen}{-7.7} \\
\midrule
\rowcolor[gray]{0.95}\multicolumn{10}{c}{\textit{Self-Consistency}} \\
Qwen2.5{-}7B & 33.3 & 55.5 & 56.2 & 32.5 & 34.7 & 55.5 & 56.2 & 52.5 & 54.1 \\
Qwen2.5{-}72B & 50.0 & 32.4 & 34.1 & 32.5 & 34.7 & 32.4 & 34.1 & 32.5 & 34.7 \\
$\Delta$  &  \cellcolor{lightblue!50}\textcolor{darkred}{-16.7} &  \cellcolor{lightblue!50}\textcolor{darkred}{+23.2}  &  \cellcolor{lightblue!50}\textcolor{darkred}{+22.1}  &  \cellcolor{lightblue!50}0.0  &  \cellcolor{lightblue!50}0.0  &  \cellcolor{lightblue!50}\textcolor{darkred}{+23.2}  &  \cellcolor{lightblue!50}\textcolor{darkred}{+22.1}  &  \cellcolor{lightblue!50}\textcolor{darkred}{+19.9}  &  \cellcolor{lightblue!50}\textcolor{darkred}{+19.3} \\
\cdashline{1-10}
Llama3.1{-}8B & 18.7 & 18.9 & 17.5 & 17.2 & 20.2 & 18.9 & 17.5 & 17.2 & 20.2 \\
Llama3.1{-}70B & 39.9 & 28.3 & 32.5 & 28.7 & 34.0 & 28.3 & 32.5 & 28.7 & 34.0 \\
$\Delta$  &  \cellcolor{lightblue!50}\textcolor{darkred}{-21.2} &  \cellcolor{lightblue!50}\textcolor{darkgreen}{-9.4}  &  \cellcolor{lightblue!50}\textcolor{darkgreen}{-14.9}  &  \cellcolor{lightblue!50}\textcolor{darkgreen}{-11.5}  &  \cellcolor{lightblue!50}\textcolor{darkgreen}{-13.9}  &  \cellcolor{lightblue!50}\textcolor{darkgreen}{-9.4}  &  \cellcolor{lightblue!50}\textcolor{darkgreen}{-14.9}  &  \cellcolor{lightblue!50}\textcolor{darkgreen}{-11.5}  &  \cellcolor{lightblue!50}\textcolor{darkgreen}{-13.9} \\
\cdashline{1-10}
Qwen3{-}8B{-}nothink & 46.0 & 23.2 & 12.8 & 20.1 & 13.5 & 23.2 & 14.6 & 24.1 & 15.4 \\
Qwen3{-}8B{-}think & 61.6 & 23.2 & 18.0 & 22.0 & 12.7 & 29.4 & 27.5 & 26.4 & 24.4 \\
$\Delta$  &  \cellcolor{lightblue!50}\textcolor{darkred}{-15.6} &  \cellcolor{lightblue!50}0.0  &  \cellcolor{lightblue!50}\textcolor{darkgreen}{-5.2}  &  \cellcolor{lightblue!50}\textcolor{darkgreen}{-1.9}  &  \cellcolor{lightblue!50}\textcolor{darkred}{+0.8}  &  \cellcolor{lightblue!50}\textcolor{darkgreen}{-6.2}  &  \cellcolor{lightblue!50}\textcolor{darkgreen}{-12.9}  &  \cellcolor{lightblue!50}\textcolor{darkgreen}{-2.3}  &  \cellcolor{lightblue!50}\textcolor{darkgreen}{-9.0} \\
\cdashline{1-10}
Qwen3{-}32B{-}nothink & 51.5 & 20.3 & 13.7 & 15.2 & 10.5 & 20.3 & 13.7 & 17.8 & 10.6 \\
Qwen3{-}32B{-}think & 67.7 & 18.8 & 14.8 & 17.8 & 9.2 & 26.1 & 26.0 & 20.5 & 20.5 \\
$\Delta$  &  \cellcolor{lightblue!50}\textcolor{darkred}{-16.2} &  \cellcolor{lightblue!50}\textcolor{darkred}{+1.5}  &  \cellcolor{lightblue!50}\textcolor{darkgreen}{-1.1}  &  \cellcolor{lightblue!50}\textcolor{darkgreen}{-2.6}  &  \cellcolor{lightblue!50}\textcolor{darkred}{+1.3}  &  \cellcolor{lightblue!50}\textcolor{darkgreen}{-5.8}  &  \cellcolor{lightblue!50}\textcolor{darkgreen}{-12.3}  &  \cellcolor{lightblue!50}\textcolor{darkgreen}{-2.8}  &  \cellcolor{lightblue!50}\textcolor{darkgreen}{-9.9} \\
\cdashline{1-10}
GPT{-}OSS{-}20B{-}low & 56.6 & 18.4 & 14.8 & 12.2 & 10.4 & 18.4 & 14.8 & 16.9 & 13.5 \\
GPT{-}OSS{-}20B{-}high & 74.2 & 15.7 & 13.0 & 15.8 & 6.9 & 22.1 & 22.5 & 16.8 & 18.4 \\
$\Delta$  &  \cellcolor{lightblue!50}\textcolor{darkred}{-17.6} &  \cellcolor{lightblue!50}\textcolor{darkred}{+2.7}  &  \cellcolor{lightblue!50}\textcolor{darkred}{+1.8}  &  \cellcolor{lightblue!50}\textcolor{darkgreen}{-3.6}  &  \cellcolor{lightblue!50}\textcolor{darkred}{+3.5}  &  \cellcolor{lightblue!50}\textcolor{darkgreen}{-3.7}  &  \cellcolor{lightblue!50}\textcolor{darkgreen}{-7.7}  &  \cellcolor{lightblue!50}\textcolor{darkred}{+0.1}  &  \cellcolor{lightblue!50}\textcolor{darkgreen}{-4.9} \\
\cdashline{1-10}
GPT{-}OSS{-}120B{-}low & 64.6 & 19.5 & 15.2 & 13.0 & 10.0 & 19.5 & 15.2 & 14.3 & 10.9 \\
GPT{-}OSS{-}120B{-}high & 81.3 & 11.0 & 9.9 & 12.3 & 7.2 & 15.4 & 16.0 & 11.0 & 10.1 \\
$\Delta$  &  \cellcolor{lightblue!50}\textcolor{darkred}{-16.7} &  \cellcolor{lightblue!50}\textcolor{darkred}{+8.5}  &  \cellcolor{lightblue!50}\textcolor{darkred}{+5.2}  &  \cellcolor{lightblue!50}\textcolor{darkred}{+0.6}  &  \cellcolor{lightblue!50}\textcolor{darkred}{+2.8}  &  \cellcolor{lightblue!50}\textcolor{darkred}{+4.0}  &  \cellcolor{lightblue!50}\textcolor{darkgreen}{-0.9}  &  \cellcolor{lightblue!50}\textcolor{darkred}{+3.3}  &  \cellcolor{lightblue!50}\textcolor{darkred}{+0.8} \\
\cdashline{1-10}
GPT{-}OSS{-}20B{-}high & 74.2 & 15.7 & 13.0 & 15.8 & 6.9 & 22.1 & 22.5 & 16.8 & 18.4 \\
GPT{-}OSS{-}120B{-}high & 81.3 & 11.0 & 9.9 & 12.3 & 7.2 & 15.4 & 16.0 & 11.0 & 10.1 \\
$\Delta$  &  \cellcolor{lightblue!50}\textcolor{darkred}{-7.1} &  \cellcolor{lightblue!50}\textcolor{darkred}{+4.6}  &  \cellcolor{lightblue!50}\textcolor{darkred}{+3.1}  &  \cellcolor{lightblue!50}\textcolor{darkred}{+3.5}  &  \cellcolor{lightblue!50}\textcolor{darkgreen}{-0.3}  &  \cellcolor{lightblue!50}\textcolor{darkred}{+6.7}  &  \cellcolor{lightblue!50}\textcolor{darkred}{+6.5}  &  \cellcolor{lightblue!50}\textcolor{darkred}{+5.8}  &  \cellcolor{lightblue!50}\textcolor{darkred}{+8.2} \\
\bottomrule
\end{tabular}
}
\caption{Calibration metrics on \textbf{GPQA Diamond}. For calibration metrics (Brier, ECE), lower is better ($\downarrow$). $\Delta$ rows show the difference (Model$_1$ $-$ Model$_2$). \textcolor{darkgreen}{Green}: better; \textcolor{darkred}{Red}: worse.}
\label{tab:gpqa_diamond}
\end{table*}


\begin{table*}[t!]
\setlength\tabcolsep{15.2pt}
\centering

\scalebox{0.72}{
\begin{tabular}{l c cc cc cc cc}
\toprule
\textbf{Model} & \textbf{Acc} & \multicolumn{2}{c}{\textbf{Raw}} & \multicolumn{2}{c}{\textbf{CA}} & \multicolumn{2}{c}{\textbf{DA}} & \multicolumn{2}{c}{\textbf{IA}} \\
\cmidrule(lr){3-4}\cmidrule(lr){5-6}\cmidrule(lr){7-8}\cmidrule(lr){9-10}
 & & Brier & ECE & Brier & ECE & Brier & ECE & Brier & ECE \\
\midrule
\rowcolor[gray]{0.95}\multicolumn{10}{c}{\textit{Verbalize}} \\
Qwen2.5{-}7B & 17.0 & 70.6 & 74.1 & 64.5 & 67.0 & 73.0 & 76.8 & 71.2 & 75.7 \\
Qwen2.5{-}72B & 36.5 & 55.9 & 57.9 & 60.5 & 64.8 & 73.0 & 76.8 & 74.2 & 78.4 \\
$\Delta$  &  \cellcolor{lightblue!50}\textcolor{darkred}{-19.5} &  \cellcolor{lightblue!50}\textcolor{darkred}{+14.7}  &  \cellcolor{lightblue!50}\textcolor{darkred}{+16.2}  &  \cellcolor{lightblue!50}\textcolor{darkred}{+4.0}  &  \cellcolor{lightblue!50}\textcolor{darkred}{+2.3}  &  \cellcolor{lightblue!50}0.0  &  \cellcolor{lightblue!50}0.0  &  \cellcolor{lightblue!50}\textcolor{darkgreen}{-3.1}  &  \cellcolor{lightblue!50}\textcolor{darkgreen}{-2.7} \\
\cdashline{1-10}
Llama3.1{-}8B & 9.5 & 70.5 & 73.0 & 65.1 & 66.8 & 70.5 & 73.0 & 70.7 & 73.9 \\
Llama3.1{-}70B & 22.5 & 52.1 & 53.8 & 44.7 & 47.8 & 57.1 & 60.3 & 59.0 & 62.5 \\
$\Delta$  &  \cellcolor{lightblue!50}\textcolor{darkred}{-13.0} &  \cellcolor{lightblue!50}\textcolor{darkred}{+18.4}  &  \cellcolor{lightblue!50}\textcolor{darkred}{+19.2}  &  \cellcolor{lightblue!50}\textcolor{darkred}{+20.4}  &  \cellcolor{lightblue!50}\textcolor{darkred}{+19.0}  &  \cellcolor{lightblue!50}\textcolor{darkred}{+13.4}  &  \cellcolor{lightblue!50}\textcolor{darkred}{+12.7}  &  \cellcolor{lightblue!50}\textcolor{darkred}{+11.7}  &  \cellcolor{lightblue!50}\textcolor{darkred}{+11.5} \\
\cdashline{1-10}
Qwen3{-}8B{-}nothink & 38.5 & 55.6 & 55.4 & 21.3 & 19.1 & 53.5 & 55.1 & 31.5 & 30.8 \\
Qwen3{-}8B{-}think & 77.0 & 17.3 & 16.0 & 37.6 & 38.5 & 49.4 & 53.4 & 29.0 & 30.7 \\
$\Delta$  &  \cellcolor{lightblue!50}\textcolor{darkred}{-38.5} &  \cellcolor{lightblue!50}\textcolor{darkred}{+38.3}  &  \cellcolor{lightblue!50}\textcolor{darkred}{+39.4}  &  \cellcolor{lightblue!50}\textcolor{darkgreen}{-16.3}  &  \cellcolor{lightblue!50}\textcolor{darkgreen}{-19.4}  &  \cellcolor{lightblue!50}\textcolor{darkred}{+4.1}  &  \cellcolor{lightblue!50}\textcolor{darkred}{+1.8}  &  \cellcolor{lightblue!50}\textcolor{darkred}{+2.4}  &  \cellcolor{lightblue!50}+0.0 \\
\cdashline{1-10}
Qwen3{-}32B{-}nothink & 47.0 & 43.7 & 44.5 & 27.7 & 27.8 & 43.7 & 44.5 & 18.9 & 17.3 \\
Qwen3{-}32B{-}think & 85.5 & 12.0 & 11.7 & 27.4 & 27.8 & 40.7 & 43.1 & 17.3 & 18.1 \\
$\Delta$  &  \cellcolor{lightblue!50}\textcolor{darkred}{-38.5} &  \cellcolor{lightblue!50}\textcolor{darkred}{+31.7}  &  \cellcolor{lightblue!50}\textcolor{darkred}{+32.8}  &  \cellcolor{lightblue!50}\textcolor{darkred}{+0.3}  &  \cellcolor{lightblue!50}0.0  &  \cellcolor{lightblue!50}\textcolor{darkred}{+3.0}  &  \cellcolor{lightblue!50}\textcolor{darkred}{+1.4}  &  \cellcolor{lightblue!50}\textcolor{darkred}{+1.6}  &  \cellcolor{lightblue!50}\textcolor{darkgreen}{-0.7} \\
\cdashline{1-10}
GPT{-}OSS{-}20B{-}low & 59.5 & 16.5 & 8.9 & 23.8 & 20.8 & 16.5 & 8.9 & 8.5 & 16.3 \\
GPT{-}OSS{-}20B{-}high & 89.5 & 2.2 & 6.3 & 16.9 & 14.9 & 13.7 & 14.8 & 2.2 & 7.4 \\
$\Delta$  &  \cellcolor{lightblue!50}\textcolor{darkred}{-30.0} &  \cellcolor{lightblue!50}\textcolor{darkred}{+14.3}  &  \cellcolor{lightblue!50}\textcolor{darkred}{+2.6}  &  \cellcolor{lightblue!50}\textcolor{darkred}{+7.0}  &  \cellcolor{lightblue!50}\textcolor{darkred}{+5.9}  &  \cellcolor{lightblue!50}\textcolor{darkred}{+2.8}  &  \cellcolor{lightblue!50}\textcolor{darkgreen}{-5.9}  &  \cellcolor{lightblue!50}\textcolor{darkred}{+6.3}  &  \cellcolor{lightblue!50}\textcolor{darkred}{+8.8} \\
\cdashline{1-10}
GPT{-}OSS{-}120B{-}low & 69.0 & 9.1 & 6.5 & 23.8 & 20.8 & 9.1 & 6.5 & 6.0 & 13.2 \\
GPT{-}OSS{-}120B{-}high & 93.0 & 3.0 & 6.4 & 8.5 & 5.2 & 9.4 & 11.0 & 2.5 & 5.2 \\
$\Delta$  &  \cellcolor{lightblue!50}\textcolor{darkred}{-24.0} &  \cellcolor{lightblue!50}\textcolor{darkred}{+6.1}  &  \cellcolor{lightblue!50}\textcolor{darkred}{+0.1}  &  \cellcolor{lightblue!50}\textcolor{darkred}{+15.4}  &  \cellcolor{lightblue!50}\textcolor{darkred}{+15.5}  &  \cellcolor{lightblue!50}\textcolor{darkgreen}{-0.3}  &  \cellcolor{lightblue!50}\textcolor{darkgreen}{-4.6}  &  \cellcolor{lightblue!50}\textcolor{darkred}{+3.4}  &  \cellcolor{lightblue!50}\textcolor{darkred}{+8.0} \\
\cdashline{1-10}
GPT{-}OSS{-}20B{-}high & 89.5 & 2.2 & 6.3 & 16.9 & 14.9 & 13.7 & 14.8 & 2.2 & 7.4 \\
GPT{-}OSS{-}120B{-}high & 93.0 & 3.0 & 6.4 & 8.5 & 5.2 & 9.4 & 11.0 & 2.5 & 5.2 \\
$\Delta$  &  \cellcolor{lightblue!50}\textcolor{darkred}{-3.5} &  \cellcolor{lightblue!50}\textcolor{darkgreen}{-0.8}  &  \cellcolor{lightblue!50}\textcolor{darkgreen}{-0.1}  &  \cellcolor{lightblue!50}\textcolor{darkred}{+8.4}  &  \cellcolor{lightblue!50}\textcolor{darkred}{+9.7}  &  \cellcolor{lightblue!50}\textcolor{darkred}{+4.3}  &  \cellcolor{lightblue!50}\textcolor{darkred}{+3.7}  &  \cellcolor{lightblue!50}\textcolor{darkgreen}{-0.3}  &  \cellcolor{lightblue!50}\textcolor{darkred}{+2.2} \\
\midrule
\rowcolor[gray]{0.95}\multicolumn{10}{c}{\textit{P(True)}} \\
Qwen2.5{-}7B & 17.0 & 59.1 & 61.4 & 57.6 & 59.2 & 59.1 & 61.4 & 57.2 & 59.6 \\
Qwen2.5{-}72B & 36.5 & 54.1 & 54.3 & 52.7 & 54.6 & 69.8 & 71.6 & 72.3 & 73.5 \\
$\Delta$  &  \cellcolor{lightblue!50}\textcolor{darkred}{-19.5} &  \cellcolor{lightblue!50}\textcolor{darkred}{+5.0}  &  \cellcolor{lightblue!50}\textcolor{darkred}{+7.1}  &  \cellcolor{lightblue!50}\textcolor{darkred}{+5.0}  &  \cellcolor{lightblue!50}\textcolor{darkred}{+4.6}  &  \cellcolor{lightblue!50}\textcolor{darkgreen}{-10.7}  &  \cellcolor{lightblue!50}\textcolor{darkgreen}{-10.2}  &  \cellcolor{lightblue!50}\textcolor{darkgreen}{-15.1}  &  \cellcolor{lightblue!50}\textcolor{darkgreen}{-13.9} \\
\cdashline{1-10}
Llama3.1{-}8B & 9.5 & 34.9 & 45.9 & 38.8 & 46.9 & 34.9 & 45.9 & 34.4 & 47.8 \\
Llama3.1{-}70B & 22.5 & 52.1 & 53.8 & 44.7 & 47.8 & 62.2 & 69.9 & 64.7 & 73.3 \\
$\Delta$  &  \cellcolor{lightblue!50}\textcolor{darkred}{-13.0} &  \cellcolor{lightblue!50}\textcolor{darkgreen}{-17.2}  &  \cellcolor{lightblue!50}\textcolor{darkgreen}{-7.9}  &  \cellcolor{lightblue!50}\textcolor{darkgreen}{-5.9}  &  \cellcolor{lightblue!50}\textcolor{darkgreen}{-0.9}  &  \cellcolor{lightblue!50}\textcolor{darkgreen}{-27.3}  &  \cellcolor{lightblue!50}\textcolor{darkgreen}{-24.1}  &  \cellcolor{lightblue!50}\textcolor{darkgreen}{-30.3}  &  \cellcolor{lightblue!50}\textcolor{darkgreen}{-25.6} \\
\cdashline{1-10}
Qwen3{-}8B{-}nothink & 38.5 & 45.1 & 45.5 & 34.0 & 33.3 & 45.1 & 45.5 & 34.0 & 33.3 \\
Qwen3{-}8B{-}think & 77.0 & 29.0 & 30.0 & 29.0 & 30.0 & 46.9 & 49.5 & 29.0 & 30.0 \\
$\Delta$  &  \cellcolor{lightblue!50}\textcolor{darkred}{-38.5} &  \cellcolor{lightblue!50}\textcolor{darkred}{+16.1}  &  \cellcolor{lightblue!50}\textcolor{darkred}{+15.5}  &  \cellcolor{lightblue!50}\textcolor{darkred}{+5.0}  &  \cellcolor{lightblue!50}\textcolor{darkred}{+3.3}  &  \cellcolor{lightblue!50}\textcolor{darkgreen}{-1.9}  &  \cellcolor{lightblue!50}\textcolor{darkgreen}{-4.1}  &  \cellcolor{lightblue!50}\textcolor{darkred}{+5.0}  &  \cellcolor{lightblue!50}\textcolor{darkred}{+3.3} \\
\cdashline{1-10}
Qwen3{-}32B{-}nothink & 47.0 & 30.3 & 28.0 & 19.5 & 17.4 & 30.3 & 28.0 & 22.5 & 24.6 \\
Qwen3{-}32B{-}think & 85.5 & 12.8 & 13.0 & 21.3 & 21.0 & 40.0 & 40.6 & 18.7 & 19.0 \\
$\Delta$  &  \cellcolor{lightblue!50}\textcolor{darkred}{-38.5} &  \cellcolor{lightblue!50}\textcolor{darkred}{+17.4}  &  \cellcolor{lightblue!50}\textcolor{darkred}{+15.0}  &  \cellcolor{lightblue!50}\textcolor{darkgreen}{-1.8}  &  \cellcolor{lightblue!50}\textcolor{darkgreen}{-3.7}  &  \cellcolor{lightblue!50}\textcolor{darkgreen}{-9.7}  &  \cellcolor{lightblue!50}\textcolor{darkgreen}{-12.6}  &  \cellcolor{lightblue!50}\textcolor{darkred}{+3.8}  &  \cellcolor{lightblue!50}\textcolor{darkred}{+5.6} \\
\cdashline{1-10}
GPT{-}OSS{-}20B{-}low & 59.5 & 21.9 & 22.2 & 15.3 & 15.7 & 21.9 & 22.2 & 12.7 & 12.8 \\
GPT{-}OSS{-}20B{-}high & 89.5 & 6.2 & 6.2 & 12.7 & 13.0 & 23.3 & 23.3 & 3.9 & 3.9 \\
$\Delta$  &  \cellcolor{lightblue!50}\textcolor{darkred}{-30.0} &  \cellcolor{lightblue!50}\textcolor{darkred}{+15.7}  &  \cellcolor{lightblue!50}\textcolor{darkred}{+16.0}  &  \cellcolor{lightblue!50}\textcolor{darkred}{+2.6}  &  \cellcolor{lightblue!50}\textcolor{darkred}{+2.8}  &  \cellcolor{lightblue!50}\textcolor{darkgreen}{-1.4}  &  \cellcolor{lightblue!50}\textcolor{darkgreen}{-1.1}  &  \cellcolor{lightblue!50}\textcolor{darkred}{+8.8}  &  \cellcolor{lightblue!50}\textcolor{darkred}{+8.9} \\
\cdashline{1-10}
GPT{-}OSS{-}120B{-}low & 69.0 & 15.0 & 15.0 & 9.4 & 9.3 & 15.0 & 15.0 & 12.5 & 12.5 \\
GPT{-}OSS{-}120B{-}high & 93.0 & 4.5 & 4.5 & 5.0 & 5.0 & 16.2 & 16.2 & 3.5 & 3.5 \\
$\Delta$  &  \cellcolor{lightblue!50}\textcolor{darkred}{-24.0} &  \cellcolor{lightblue!50}\textcolor{darkred}{+10.5}  &  \cellcolor{lightblue!50}\textcolor{darkred}{+10.5}  &  \cellcolor{lightblue!50}\textcolor{darkred}{+4.4}  &  \cellcolor{lightblue!50}\textcolor{darkred}{+4.3}  &  \cellcolor{lightblue!50}\textcolor{darkgreen}{-1.2}  &  \cellcolor{lightblue!50}\textcolor{darkgreen}{-1.3}  &  \cellcolor{lightblue!50}\textcolor{darkred}{+9.0}  &  \cellcolor{lightblue!50}\textcolor{darkred}{+9.0} \\
\cdashline{1-10}
GPT{-}OSS{-}20B{-}high & 89.5 & 6.2 & 6.2 & 12.7 & 13.0 & 23.3 & 23.3 & 3.9 & 3.9 \\
GPT{-}OSS{-}120B{-}high & 93.0 & 4.5 & 4.5 & 5.0 & 5.0 & 16.2 & 16.2 & 3.5 & 3.5 \\
$\Delta$  &  \cellcolor{lightblue!50}\textcolor{darkred}{-3.5} &  \cellcolor{lightblue!50}\textcolor{darkred}{+1.8}  &  \cellcolor{lightblue!50}\textcolor{darkred}{+1.7}  &  \cellcolor{lightblue!50}\textcolor{darkred}{+7.7}  &  \cellcolor{lightblue!50}\textcolor{darkred}{+8.0}  &  \cellcolor{lightblue!50}\textcolor{darkred}{+7.1}  &  \cellcolor{lightblue!50}\textcolor{darkred}{+7.0}  &  \cellcolor{lightblue!50}\textcolor{darkred}{+0.4}  &  \cellcolor{lightblue!50}\textcolor{darkred}{+0.4} \\
\midrule
\rowcolor[gray]{0.95}\multicolumn{10}{c}{\textit{Self-Consistency}} \\
Qwen2.5{-}7B & 17.0 & 20.0 & 18.8 & 12.3 & 16.0 & 20.0 & 18.8 & 19.5 & 20.2 \\
Qwen2.5{-}72B & 36.5 & 14.5 & 11.8 & 12.3 & 16.0 & 14.5 & 11.8 & 12.3 & 16.0 \\
$\Delta$  &  \cellcolor{lightblue!50}\textcolor{darkred}{-19.5} &  \cellcolor{lightblue!50}\textcolor{darkred}{+5.5}  &  \cellcolor{lightblue!50}\textcolor{darkred}{+7.0}  &  \cellcolor{lightblue!50}0.0  &  \cellcolor{lightblue!50}0.0  &  \cellcolor{lightblue!50}\textcolor{darkred}{+5.5}  &  \cellcolor{lightblue!50}\textcolor{darkred}{+7.0}  &  \cellcolor{lightblue!50}\textcolor{darkred}{+7.2}  &  \cellcolor{lightblue!50}\textcolor{darkred}{+4.2} \\
\cdashline{1-10}
Llama3.1{-}8B & 9.5 & 7.8 & 5.3 & 6.5 & 7.0 & 7.8 & 5.3 & 6.5 & 7.0 \\
Llama3.1{-}70B & 22.5 & 8.5 & 8.8 & 7.2 & 11.9 & 8.5 & 8.8 & 7.2 & 11.9 \\
$\Delta$  &  \cellcolor{lightblue!50}\textcolor{darkred}{-13.0} &  \cellcolor{lightblue!50}\textcolor{darkgreen}{-0.7}  &  \cellcolor{lightblue!50}\textcolor{darkgreen}{-3.5}  &  \cellcolor{lightblue!50}\textcolor{darkgreen}{-0.8}  &  \cellcolor{lightblue!50}\textcolor{darkgreen}{-4.9}  &  \cellcolor{lightblue!50}\textcolor{darkgreen}{-0.7}  &  \cellcolor{lightblue!50}\textcolor{darkgreen}{-3.5}  &  \cellcolor{lightblue!50}\textcolor{darkgreen}{-0.8}  &  \cellcolor{lightblue!50}\textcolor{darkgreen}{-4.9} \\
\cdashline{1-10}
Qwen3{-}8B{-}nothink & 38.5 & 17.0 & 7.6 & 11.8 & 10.1 & 17.0 & 8.6 & 21.9 & 21.9 \\
Qwen3{-}8B{-}think & 77.0 & 12.5 & 15.3 & 22.4 & 33.7 & 9.8 & 9.0 & 9.2 & 8.8 \\
$\Delta$  &  \cellcolor{lightblue!50}\textcolor{darkred}{-38.5} &  \cellcolor{lightblue!50}\textcolor{darkred}{+4.5}  &  \cellcolor{lightblue!50}\textcolor{darkgreen}{-7.8}  &  \cellcolor{lightblue!50}\textcolor{darkgreen}{-10.6}  &  \cellcolor{lightblue!50}\textcolor{darkgreen}{-23.6}  &  \cellcolor{lightblue!50}\textcolor{darkred}{+7.2}  &  \cellcolor{lightblue!50}\textcolor{darkgreen}{-0.4}  &  \cellcolor{lightblue!50}\textcolor{darkred}{+12.7}  &  \cellcolor{lightblue!50}\textcolor{darkred}{+13.1} \\
\cdashline{1-10}
Qwen3{-}32B{-}nothink & 47.0 & 16.4 & 8.8 & 9.1 & 10.1 & 16.4 & 8.8 & 20.9 & 27.0 \\
Qwen3{-}32B{-}think & 85.5 & 11.8 & 17.4 & 22.0 & 33.5 & 10.0 & 8.0 & 9.4 & 9.3 \\
$\Delta$  &  \cellcolor{lightblue!50}\textcolor{darkred}{-38.5} &  \cellcolor{lightblue!50}\textcolor{darkred}{+4.6}  &  \cellcolor{lightblue!50}\textcolor{darkgreen}{-8.7}  &  \cellcolor{lightblue!50}\textcolor{darkgreen}{-12.8}  &  \cellcolor{lightblue!50}\textcolor{darkgreen}{-23.4}  &  \cellcolor{lightblue!50}\textcolor{darkred}{+6.4}  &  \cellcolor{lightblue!50}\textcolor{darkred}{+0.8}  &  \cellcolor{lightblue!50}\textcolor{darkred}{+11.5}  &  \cellcolor{lightblue!50}\textcolor{darkred}{+17.7} \\
\cdashline{1-10}
GPT{-}OSS{-}20B{-}low & 59.5 & 16.1 & 18.1 & 13.2 & 19.0 & 16.3 & 18.7 & 26.6 & 43.6 \\
GPT{-}OSS{-}20B{-}high & 89.5 & 16.9 & 27.2 & 26.6 & 43.7 & 12.3 & 9.3 & 10.7 & 19.2 \\
$\Delta$  &  \cellcolor{lightblue!50}\textcolor{darkred}{-30.0} &  \cellcolor{lightblue!50}\textcolor{darkgreen}{-0.7}  &  \cellcolor{lightblue!50}\textcolor{darkgreen}{-9.0}  &  \cellcolor{lightblue!50}\textcolor{darkgreen}{-13.3}  &  \cellcolor{lightblue!50}\textcolor{darkgreen}{-24.6}  &  \cellcolor{lightblue!50}\textcolor{darkred}{+4.0}  &  \cellcolor{lightblue!50}\textcolor{darkred}{+9.4}  &  \cellcolor{lightblue!50}\textcolor{darkred}{+15.8}  &  \cellcolor{lightblue!50}\textcolor{darkred}{+24.3} \\
\cdashline{1-10}
GPT{-}OSS{-}120B{-}low & 69.0 & 19.2 & 17.4 & 16.9 & 20.1 & 19.6 & 18.9 & 24.2 & 35.7 \\
GPT{-}OSS{-}120B{-}high & 93.0 & 10.6 & 18.6 & 18.4 & 31.7 & 8.4 & 8.7 & 10.4 & 17.3 \\
$\Delta$  &  \cellcolor{lightblue!50}\textcolor{darkred}{-24.0} &  \cellcolor{lightblue!50}\textcolor{darkred}{+8.5}  &  \cellcolor{lightblue!50}\textcolor{darkgreen}{-1.3}  &  \cellcolor{lightblue!50}\textcolor{darkgreen}{-1.5}  &  \cellcolor{lightblue!50}\textcolor{darkgreen}{-11.6}  &  \cellcolor{lightblue!50}\textcolor{darkred}{+11.1}  &  \cellcolor{lightblue!50}\textcolor{darkred}{+10.2}  &  \cellcolor{lightblue!50}\textcolor{darkred}{+13.8}  &  \cellcolor{lightblue!50}\textcolor{darkred}{+18.4} \\
\cdashline{1-10}
GPT{-}OSS{-}20B{-}high & 89.5 & 16.9 & 27.2 & 26.6 & 43.7 & 12.3 & 9.3 & 10.7 & 19.2 \\
GPT{-}OSS{-}120B{-}high & 93.0 & 10.6 & 18.6 & 18.4 & 31.7 & 8.4 & 8.7 & 10.4 & 17.3 \\
$\Delta$  &  \cellcolor{lightblue!50}\textcolor{darkred}{-3.5} &  \cellcolor{lightblue!50}\textcolor{darkred}{+6.2}  &  \cellcolor{lightblue!50}\textcolor{darkred}{+8.5}  &  \cellcolor{lightblue!50}\textcolor{darkred}{+8.2}  &  \cellcolor{lightblue!50}\textcolor{darkred}{+12.0}  &  \cellcolor{lightblue!50}\textcolor{darkred}{+3.9}  &  \cellcolor{lightblue!50}\textcolor{darkred}{+0.7}  &  \cellcolor{lightblue!50}\textcolor{darkred}{+0.3}  &  \cellcolor{lightblue!50}\textcolor{darkred}{+1.9} \\
\bottomrule
\end{tabular}
}
\caption{Calibration metrics on \textbf{LiveBench}. For calibration metrics (Brier, ECE), lower is better ($\downarrow$). $\Delta$ rows show the difference (Model$_1$ $-$ Model$_2$). \textcolor{darkgreen}{Green}: better; \textcolor{darkred}{Red}: worse.}
\label{tab:livebench}
\end{table*}


\begin{table*}[t!]
\setlength\tabcolsep{15.2pt}
\centering

\scalebox{0.72}{
\begin{tabular}{l c cc cc cc cc}
\toprule
\textbf{Model} & \textbf{Acc} & \multicolumn{2}{c}{\textbf{Raw}} & \multicolumn{2}{c}{\textbf{CA}} & \multicolumn{2}{c}{\textbf{DA}} & \multicolumn{2}{c}{\textbf{IA}} \\
\cmidrule(lr){3-4}\cmidrule(lr){5-6}\cmidrule(lr){7-8}\cmidrule(lr){9-10}
 & & Brier & ECE & Brier & ECE & Brier & ECE & Brier & ECE \\
\midrule
\rowcolor[gray]{0.95}\multicolumn{10}{c}{\textit{Verbalize}} \\
Qwen2.5{-}7B & 49.0 & 43.6 & 44.8 & 33.9 & 34.2 & 43.6 & 44.8 & 28.4 & 28.3 \\
Qwen2.5{-}72B & 71.6 & 23.9 & 23.0 & 30.9 & 31.1 & 42.1 & 43.6 & 27.7 & 28.0 \\
$\Delta$  &  \cellcolor{lightblue!50}\textcolor{darkred}{-22.6} &  \cellcolor{lightblue!50}\textcolor{darkred}{+19.7}  &  \cellcolor{lightblue!50}\textcolor{darkred}{+21.8}  &  \cellcolor{lightblue!50}\textcolor{darkred}{+3.1}  &  \cellcolor{lightblue!50}\textcolor{darkred}{+3.0}  &  \cellcolor{lightblue!50}\textcolor{darkred}{+1.5}  &  \cellcolor{lightblue!50}\textcolor{darkred}{+1.2}  &  \cellcolor{lightblue!50}\textcolor{darkred}{+0.8}  &  \cellcolor{lightblue!50}\textcolor{darkred}{+0.3} \\
\cdashline{1-10}
Llama3.1{-}8B & 31.6 & 56.7 & 57.3 & 47.3 & 47.5 & 57.0 & 57.7 & 48.8 & 49.4 \\
Llama3.1{-}70B & 57.6 & 28.3 & 28.5 & 28.3 & 28.5 & 39.5 & 40.2 & 33.3 & 34.1 \\
$\Delta$  &  \cellcolor{lightblue!50}\textcolor{darkred}{-26.0} &  \cellcolor{lightblue!50}\textcolor{darkred}{+28.4}  &  \cellcolor{lightblue!50}\textcolor{darkred}{+28.8}  &  \cellcolor{lightblue!50}\textcolor{darkred}{+19.0}  &  \cellcolor{lightblue!50}\textcolor{darkred}{+19.0}  &  \cellcolor{lightblue!50}\textcolor{darkred}{+17.5}  &  \cellcolor{lightblue!50}\textcolor{darkred}{+17.4}  &  \cellcolor{lightblue!50}\textcolor{darkred}{+15.6}  &  \cellcolor{lightblue!50}\textcolor{darkred}{+15.3} \\
\cdashline{1-10}
Qwen3{-}8B{-}nothink & 70.9 & 25.2 & 25.1 & 19.2 & 19.1 & 25.2 & 25.1 & 12.6 & 11.4 \\
Qwen3{-}8B{-}think & 85.6 & 13.1 & 12.6 & 16.5 & 16.3 & 25.9 & 26.1 & 13.0 & 12.7 \\
$\Delta$  &  \cellcolor{lightblue!50}\textcolor{darkred}{-14.7} &  \cellcolor{lightblue!50}\textcolor{darkred}{+12.1}  &  \cellcolor{lightblue!50}\textcolor{darkred}{+12.5}  &  \cellcolor{lightblue!50}\textcolor{darkred}{+2.7}  &  \cellcolor{lightblue!50}\textcolor{darkred}{+2.8}  &  \cellcolor{lightblue!50}\textcolor{darkgreen}{-0.6}  &  \cellcolor{lightblue!50}\textcolor{darkgreen}{-0.9}  &  \cellcolor{lightblue!50}\textcolor{darkgreen}{-0.5}  &  \cellcolor{lightblue!50}\textcolor{darkgreen}{-1.3} \\
\cdashline{1-10}
Qwen3{-}32B{-}nothink & 76.6 & 18.7 & 18.1 & 14.4 & 13.7 & 18.7 & 18.1 & 9.6 & 8.0 \\
Qwen3{-}32B{-}think & 88.6 & 11.2 & 10.5 & 15.2 & 14.4 & 21.2 & 21.0 & 10.6 & 10.0 \\
$\Delta$  &  \cellcolor{lightblue!50}\textcolor{darkred}{-12.0} &  \cellcolor{lightblue!50}\textcolor{darkred}{+7.5}  &  \cellcolor{lightblue!50}\textcolor{darkred}{+7.6}  &  \cellcolor{lightblue!50}\textcolor{darkgreen}{-0.8}  &  \cellcolor{lightblue!50}\textcolor{darkgreen}{-0.7}  &  \cellcolor{lightblue!50}\textcolor{darkgreen}{-2.5}  &  \cellcolor{lightblue!50}\textcolor{darkgreen}{-2.9}  &  \cellcolor{lightblue!50}\textcolor{darkgreen}{-1.0}  &  \cellcolor{lightblue!50}\textcolor{darkgreen}{-2.0} \\
\cdashline{1-10}
GPT{-}OSS{-}20B{-}low & 74.8 & 15.9 & 12.1 & 12.8 & 9.1 & 16.0 & 12.2 & 9.1 & 4.9 \\
GPT{-}OSS{-}20B{-}high & 86.7 & 10.5 & 7.9 & 12.4 & 9.7 & 20.3 & 19.0 & 10.1 & 7.7 \\
$\Delta$  &  \cellcolor{lightblue!50}\textcolor{darkred}{-11.9} &  \cellcolor{lightblue!50}\textcolor{darkred}{+5.3}  &  \cellcolor{lightblue!50}\textcolor{darkred}{+4.2}  &  \cellcolor{lightblue!50}\textcolor{darkred}{+0.4}  &  \cellcolor{lightblue!50}\textcolor{darkgreen}{-0.6}  &  \cellcolor{lightblue!50}\textcolor{darkgreen}{-4.3}  &  \cellcolor{lightblue!50}\textcolor{darkgreen}{-6.9}  &  \cellcolor{lightblue!50}\textcolor{darkgreen}{-1.0}  &  \cellcolor{lightblue!50}\textcolor{darkgreen}{-2.9} \\
\cdashline{1-10}
GPT{-}OSS{-}120B{-}low & 85.2 & 10.6 & 6.6 & 10.0 & 6.7 & 10.6 & 6.6 & 7.3 & 3.2 \\
GPT{-}OSS{-}120B{-}high & 88.2 & 8.9 & 6.3 & 9.1 & 6.1 & 11.2 & 9.0 & 7.4 & 5.0 \\
$\Delta$  &  \cellcolor{lightblue!50}\textcolor{darkred}{-3.0} &  \cellcolor{lightblue!50}\textcolor{darkred}{+1.7}  &  \cellcolor{lightblue!50}\textcolor{darkred}{+0.3}  &  \cellcolor{lightblue!50}\textcolor{darkred}{+0.9}  &  \cellcolor{lightblue!50}\textcolor{darkred}{+0.7}  &  \cellcolor{lightblue!50}\textcolor{darkgreen}{-0.6}  &  \cellcolor{lightblue!50}\textcolor{darkgreen}{-2.4}  &  \cellcolor{lightblue!50}\textcolor{darkgreen}{-0.1}  &  \cellcolor{lightblue!50}\textcolor{darkgreen}{-1.8} \\
\cdashline{1-10}
GPT{-}OSS{-}20B{-}high & 86.7 & 10.5 & 7.9 & 12.4 & 9.7 & 20.3 & 19.0 & 10.1 & 7.7 \\
GPT{-}OSS{-}120B{-}high & 88.2 & 8.9 & 6.3 & 9.1 & 6.1 & 11.2 & 9.0 & 7.4 & 5.0 \\
$\Delta$  &  \cellcolor{lightblue!50}\textcolor{darkred}{-1.5} &  \cellcolor{lightblue!50}\textcolor{darkred}{+1.7}  &  \cellcolor{lightblue!50}\textcolor{darkred}{+1.6}  &  \cellcolor{lightblue!50}\textcolor{darkred}{+3.3}  &  \cellcolor{lightblue!50}\textcolor{darkred}{+3.7}  &  \cellcolor{lightblue!50}\textcolor{darkred}{+9.1}  &  \cellcolor{lightblue!50}\textcolor{darkred}{+10.1}  &  \cellcolor{lightblue!50}\textcolor{darkred}{+2.7}  &  \cellcolor{lightblue!50}\textcolor{darkred}{+2.8} \\
\midrule
\rowcolor[gray]{0.95}\multicolumn{10}{c}{\textit{P(True)}} \\
Qwen2.5{-}7B & 49.0 & 33.7 & 32.7 & 29.4 & 28.3 & 33.7 & 32.7 & 26.1 & 25.3 \\
Qwen2.5{-}72B & 71.6 & 20.6 & 19.2 & 21.1 & 19.5 & 31.4 & 30.6 & 20.9 & 20.0 \\
$\Delta$  &  \cellcolor{lightblue!50}\textcolor{darkred}{-22.6} &  \cellcolor{lightblue!50}\textcolor{darkred}{+13.0}  &  \cellcolor{lightblue!50}\textcolor{darkred}{+13.5}  &  \cellcolor{lightblue!50}\textcolor{darkred}{+8.3}  &  \cellcolor{lightblue!50}\textcolor{darkred}{+8.8}  &  \cellcolor{lightblue!50}\textcolor{darkred}{+2.2}  &  \cellcolor{lightblue!50}\textcolor{darkred}{+2.1}  &  \cellcolor{lightblue!50}\textcolor{darkred}{+5.2}  &  \cellcolor{lightblue!50}\textcolor{darkred}{+5.3} \\
\cdashline{1-10}
Llama3.1{-}8B & 31.6 & 28.6 & 25.6 & 28.9 & 19.9 & 28.6 & 25.6 & 25.8 & 18.9 \\
Llama3.1{-}70B & 57.6 & 27.5 & 26.0 & 28.3 & 28.1 & 40.6 & 43.2 & 34.6 & 37.0 \\
$\Delta$  &  \cellcolor{lightblue!50}\textcolor{darkred}{-26.0} &  \cellcolor{lightblue!50}\textcolor{darkred}{+1.1}  &  \cellcolor{lightblue!50}\textcolor{darkgreen}{-0.4}  &  \cellcolor{lightblue!50}\textcolor{darkred}{+0.6}  &  \cellcolor{lightblue!50}\textcolor{darkgreen}{-8.2}  &  \cellcolor{lightblue!50}\textcolor{darkgreen}{-11.9}  &  \cellcolor{lightblue!50}\textcolor{darkgreen}{-17.6}  &  \cellcolor{lightblue!50}\textcolor{darkgreen}{-8.9}  &  \cellcolor{lightblue!50}\textcolor{darkgreen}{-18.2} \\
\cdashline{1-10}
Qwen3{-}8B{-}nothink & 70.9 & 23.4 & 21.3 & 20.3 & 19.2 & 23.4 & 21.3 & 16.2 & 15.2 \\
Qwen3{-}8B{-}think & 85.6 & 19.7 & 19.6 & 20.3 & 19.8 & 30.3 & 29.2 & 19.5 & 19.4 \\
$\Delta$  &  \cellcolor{lightblue!50}\textcolor{darkred}{-14.7} &  \cellcolor{lightblue!50}\textcolor{darkred}{+3.7}  &  \cellcolor{lightblue!50}\textcolor{darkred}{+1.7}  &  \cellcolor{lightblue!50}+0.0  &  \cellcolor{lightblue!50}\textcolor{darkgreen}{-0.6}  &  \cellcolor{lightblue!50}\textcolor{darkgreen}{-6.9}  &  \cellcolor{lightblue!50}\textcolor{darkgreen}{-7.8}  &  \cellcolor{lightblue!50}\textcolor{darkgreen}{-3.3}  &  \cellcolor{lightblue!50}\textcolor{darkgreen}{-4.2} \\
\cdashline{1-10}
Qwen3{-}32B{-}nothink & 76.6 & 14.9 & 11.3 & 13.0 & 10.7 & 14.9 & 11.3 & 12.2 & 11.6 \\
Qwen3{-}32B{-}think & 88.6 & 13.0 & 13.1 & 15.3 & 15.2 & 22.9 & 22.8 & 12.7 & 12.7 \\
$\Delta$  &  \cellcolor{lightblue!50}\textcolor{darkred}{-12.0} &  \cellcolor{lightblue!50}\textcolor{darkred}{+1.9}  &  \cellcolor{lightblue!50}\textcolor{darkgreen}{-1.7}  &  \cellcolor{lightblue!50}\textcolor{darkgreen}{-2.3}  &  \cellcolor{lightblue!50}\textcolor{darkgreen}{-4.5}  &  \cellcolor{lightblue!50}\textcolor{darkgreen}{-8.0}  &  \cellcolor{lightblue!50}\textcolor{darkgreen}{-11.5}  &  \cellcolor{lightblue!50}\textcolor{darkgreen}{-0.5}  &  \cellcolor{lightblue!50}\textcolor{darkgreen}{-1.1} \\
\cdashline{1-10}
GPT{-}OSS{-}20B{-}low & 74.8 & 20.1 & 19.9 & 30.1 & 30.0 & 20.1 & 19.9 & 13.6 & 13.6 \\
GPT{-}OSS{-}20B{-}high & 86.7 & 16.1 & 16.2 & 14.7 & 14.7 & 25.3 & 25.3 & 15.6 & 15.7 \\
$\Delta$  &  \cellcolor{lightblue!50}\textcolor{darkred}{-11.9} &  \cellcolor{lightblue!50}\textcolor{darkred}{+4.0}  &  \cellcolor{lightblue!50}\textcolor{darkred}{+3.7}  &  \cellcolor{lightblue!50}\textcolor{darkred}{+15.4}  &  \cellcolor{lightblue!50}\textcolor{darkred}{+15.2}  &  \cellcolor{lightblue!50}\textcolor{darkgreen}{-5.2}  &  \cellcolor{lightblue!50}\textcolor{darkgreen}{-5.4}  &  \cellcolor{lightblue!50}\textcolor{darkgreen}{-2.1}  &  \cellcolor{lightblue!50}\textcolor{darkgreen}{-2.1} \\
\cdashline{1-10}
GPT{-}OSS{-}120B{-}low & 85.2 & 14.5 & 14.6 & 15.4 & 15.3 & 14.5 & 14.6 & 10.0 & 10.0 \\
GPT{-}OSS{-}120B{-}high & 88.2 & 21.7 & 21.8 & 17.2 & 17.3 & 23.4 & 23.5 & 20.2 & 20.2 \\
$\Delta$  &  \cellcolor{lightblue!50}\textcolor{darkred}{-3.0} &  \cellcolor{lightblue!50}\textcolor{darkgreen}{-7.2}  &  \cellcolor{lightblue!50}\textcolor{darkgreen}{-7.2}  &  \cellcolor{lightblue!50}\textcolor{darkgreen}{-1.8}  &  \cellcolor{lightblue!50}\textcolor{darkgreen}{-2.0}  &  \cellcolor{lightblue!50}\textcolor{darkgreen}{-8.9}  &  \cellcolor{lightblue!50}\textcolor{darkgreen}{-8.9}  &  \cellcolor{lightblue!50}\textcolor{darkgreen}{-10.2}  &  \cellcolor{lightblue!50}\textcolor{darkgreen}{-10.2} \\
\cdashline{1-10}
GPT{-}OSS{-}20B{-}high & 86.7 & 16.1 & 16.2 & 14.7 & 14.7 & 25.3 & 25.3 & 15.6 & 15.7 \\
GPT{-}OSS{-}120B{-}high & 88.2 & 21.7 & 21.8 & 17.2 & 17.3 & 23.4 & 23.5 & 20.2 & 20.2 \\
$\Delta$  &  \cellcolor{lightblue!50}\textcolor{darkred}{-1.5} &  \cellcolor{lightblue!50}\textcolor{darkgreen}{-5.6}  &  \cellcolor{lightblue!50}\textcolor{darkgreen}{-5.6}  &  \cellcolor{lightblue!50}\textcolor{darkgreen}{-2.5}  &  \cellcolor{lightblue!50}\textcolor{darkgreen}{-2.5}  &  \cellcolor{lightblue!50}\textcolor{darkred}{+1.9}  &  \cellcolor{lightblue!50}\textcolor{darkred}{+1.8}  &  \cellcolor{lightblue!50}\textcolor{darkgreen}{-4.5}  &  \cellcolor{lightblue!50}\textcolor{darkgreen}{-4.5} \\
\midrule
\rowcolor[gray]{0.95}\multicolumn{10}{c}{\textit{Self-Consistency}} \\
Qwen2.5{-}7B & 49.0 & 19.5 & 13.6 & 11.0 & 8.9 & 19.5 & 13.6 & 16.2 & 6.1 \\
Qwen2.5{-}72B & 71.6 & 12.8 & 6.3 & 14.6 & 9.7 & 19.3 & 19.4 & 13.6 & 13.4 \\
$\Delta$  &  \cellcolor{lightblue!50}\textcolor{darkred}{-22.6} &  \cellcolor{lightblue!50}\textcolor{darkred}{+6.7}  &  \cellcolor{lightblue!50}\textcolor{darkred}{+7.3}  &  \cellcolor{lightblue!50}\textcolor{darkgreen}{-3.6}  &  \cellcolor{lightblue!50}\textcolor{darkgreen}{-0.8}  &  \cellcolor{lightblue!50}\textcolor{darkred}{+0.2}  &  \cellcolor{lightblue!50}\textcolor{darkgreen}{-5.9}  &  \cellcolor{lightblue!50}\textcolor{darkred}{+2.6}  &  \cellcolor{lightblue!50}\textcolor{darkgreen}{-7.2} \\
\cdashline{1-10}
Llama3.1{-}8B & 31.6 & 16.1 & 6.8 & 9.9 & 7.4 & 16.1 & 6.7 & 14.9 & 5.7 \\
Llama3.1{-}70B & 57.6 & 13.1 & 5.0 & 16.5 & 13.1 & 14.9 & 15.8 & 12.5 & 14.5 \\
$\Delta$  &  \cellcolor{lightblue!50}\textcolor{darkred}{-26.0} &  \cellcolor{lightblue!50}\textcolor{darkred}{+3.0}  &  \cellcolor{lightblue!50}\textcolor{darkred}{+1.8}  &  \cellcolor{lightblue!50}\textcolor{darkgreen}{-6.6}  &  \cellcolor{lightblue!50}\textcolor{darkgreen}{-5.7}  &  \cellcolor{lightblue!50}\textcolor{darkred}{+1.2}  &  \cellcolor{lightblue!50}\textcolor{darkgreen}{-9.1}  &  \cellcolor{lightblue!50}\textcolor{darkred}{+2.4}  &  \cellcolor{lightblue!50}\textcolor{darkgreen}{-8.8} \\
\cdashline{1-10}
Qwen3{-}8B{-}nothink & 70.9 & 11.5 & 4.8 & 7.8 & 4.5 & 11.5 & 4.8 & 9.1 & 5.3 \\
Qwen3{-}8B{-}think & 85.6 & 9.2 & 6.9 & 9.7 & 5.7 & 17.0 & 16.9 & 8.9 & 8.9 \\
$\Delta$  &  \cellcolor{lightblue!50}\textcolor{darkred}{-14.7} &  \cellcolor{lightblue!50}\textcolor{darkred}{+2.3}  &  \cellcolor{lightblue!50}\textcolor{darkgreen}{-2.0}  &  \cellcolor{lightblue!50}\textcolor{darkgreen}{-1.9}  &  \cellcolor{lightblue!50}\textcolor{darkgreen}{-1.2}  &  \cellcolor{lightblue!50}\textcolor{darkgreen}{-5.5}  &  \cellcolor{lightblue!50}\textcolor{darkgreen}{-12.1}  &  \cellcolor{lightblue!50}\textcolor{darkred}{+0.2}  &  \cellcolor{lightblue!50}\textcolor{darkgreen}{-3.6} \\
\cdashline{1-10}
Qwen3{-}32B{-}nothink & 76.6 & 9.0 & 4.1 & 6.3 & 3.5 & 9.0 & 4.1 & 7.6 & 4.9 \\
Qwen3{-}32B{-}think & 88.6 & 7.1 & 5.0 & 7.7 & 4.6 & 13.3 & 12.9 & 6.9 & 6.3 \\
$\Delta$  &  \cellcolor{lightblue!50}\textcolor{darkred}{-12.0} &  \cellcolor{lightblue!50}\textcolor{darkred}{+1.9}  &  \cellcolor{lightblue!50}\textcolor{darkgreen}{-0.9}  &  \cellcolor{lightblue!50}\textcolor{darkgreen}{-1.5}  &  \cellcolor{lightblue!50}\textcolor{darkgreen}{-1.1}  &  \cellcolor{lightblue!50}\textcolor{darkgreen}{-4.3}  &  \cellcolor{lightblue!50}\textcolor{darkgreen}{-8.8}  &  \cellcolor{lightblue!50}\textcolor{darkred}{+0.8}  &  \cellcolor{lightblue!50}\textcolor{darkgreen}{-1.4} \\
\cdashline{1-10}
GPT{-}OSS{-}20B{-}low & 74.8 & 10.1 & 4.6 & 7.2 & 4.4 & 10.1 & 4.6 & 8.8 & 5.5 \\
GPT{-}OSS{-}20B{-}high & 86.7 & 8.3 & 6.5 & 8.6 & 4.7 & 14.7 & 14.6 & 8.3 & 8.0 \\
$\Delta$  &  \cellcolor{lightblue!50}\textcolor{darkred}{-11.9} &  \cellcolor{lightblue!50}\textcolor{darkred}{+1.8}  &  \cellcolor{lightblue!50}\textcolor{darkgreen}{-1.9}  &  \cellcolor{lightblue!50}\textcolor{darkgreen}{-1.4}  &  \cellcolor{lightblue!50}\textcolor{darkgreen}{-0.2}  &  \cellcolor{lightblue!50}\textcolor{darkgreen}{-4.6}  &  \cellcolor{lightblue!50}\textcolor{darkgreen}{-10.0}  &  \cellcolor{lightblue!50}\textcolor{darkred}{+0.5}  &  \cellcolor{lightblue!50}\textcolor{darkgreen}{-2.5} \\
\cdashline{1-10}
GPT{-}OSS{-}120B{-}low & 85.2 & 9.0 & 6.3 & 7.2 & 5.1 & 9.0 & 6.3 & 7.3 & 5.2 \\
GPT{-}OSS{-}120B{-}high & 88.2 & 7.5 & 6.5 & 7.6 & 5.0 & 9.2 & 8.6 & 6.6 & 6.5 \\
$\Delta$  &  \cellcolor{lightblue!50}\textcolor{darkred}{-3.0} &  \cellcolor{lightblue!50}\textcolor{darkred}{+1.6}  &  \cellcolor{lightblue!50}\textcolor{darkgreen}{-0.2}  &  \cellcolor{lightblue!50}\textcolor{darkgreen}{-0.4}  &  \cellcolor{lightblue!50}+0.0  &  \cellcolor{lightblue!50}\textcolor{darkgreen}{-0.2}  &  \cellcolor{lightblue!50}\textcolor{darkgreen}{-2.2}  &  \cellcolor{lightblue!50}\textcolor{darkred}{+0.8}  &  \cellcolor{lightblue!50}\textcolor{darkgreen}{-1.2} \\
\cdashline{1-10}
GPT{-}OSS{-}20B{-}high & 86.7 & 8.3 & 6.5 & 8.6 & 4.7 & 14.7 & 14.6 & 8.3 & 8.0 \\
GPT{-}OSS{-}120B{-}high & 88.2 & 7.5 & 6.5 & 7.6 & 5.0 & 9.2 & 8.6 & 6.6 & 6.5 \\
$\Delta$  &  \cellcolor{lightblue!50}\textcolor{darkred}{-1.5} &  \cellcolor{lightblue!50}\textcolor{darkred}{+0.8}  &  \cellcolor{lightblue!50}0.0  &  \cellcolor{lightblue!50}\textcolor{darkred}{+1.0}  &  \cellcolor{lightblue!50}\textcolor{darkgreen}{-0.4}  &  \cellcolor{lightblue!50}\textcolor{darkred}{+5.4}  &  \cellcolor{lightblue!50}\textcolor{darkred}{+6.1}  &  \cellcolor{lightblue!50}\textcolor{darkred}{+1.7}  &  \cellcolor{lightblue!50}\textcolor{darkred}{+1.6} \\
\bottomrule
\end{tabular}
}
\caption{Calibration metrics on \textbf{MMLU-Pro}. For calibration metrics (Brier, ECE), lower is better ($\downarrow$). $\Delta$ rows show the difference (Model$_1$ $-$ Model$_2$). \textcolor{darkgreen}{Green}: better; \textcolor{darkred}{Red}: worse.}
\label{tab:mmlu}
\end{table*}

\end{document}